\documentclass[final]{styles/arxiv/colt2021} 

\newif\ifspacehack
\usepackage{natbib}
\hypersetup{
    colorlinks = blue,
    breaklinks,
    linkcolor = blue,
    citecolor = blue,
    urlcolor  = blue,
}
\usepackage{url} 
\usepackage{graphicx}
\usepackage{mathtools}
\usepackage{footnote}
\usepackage{float}
\usepackage{xspace}
\usepackage{multirow}
\usepackage{xcolor}
\usepackage{wrapfig}
\usepackage{framed}
\usepackage{bbm}
\usepackage{footnote}
\usepackage{nicefrac}
\usepackage{makecell}
\usepackage[algo2e]{algorithm2e} 
\usepackage{algorithm}
\usepackage{amssymb}
\usepackage{bm}
\makesavenoteenv{tabular}
\makesavenoteenv{table}

\newcommand{\combine}{p}        
\newcommand{\sleep}{\hat{p}} 

\renewcommand{\tilde}{\widetilde}
\renewcommand{\hat}{\widehat}

\newcommand{\corral}{\textsc{Corral}\xspace}
\newcommand{\exptwo}{\textsc{Exp2}\xspace}
\newcommand{\expthree}{\textsc{Exp3}\xspace}
\newcommand{\expfour}{\textsc{Exp4}\xspace}

\newcommand{\scrible}{\textsc{SCRiBLe}\xspace}
\newcommand{\holder}{H\"{o}lder\xspace}

\newcommand \term[1]{\textsc{term}~(\textsc{#1})}

\def \R {\mathbb{R}}

\newcommand{\calB}{{\mathcal{B}}}
\newcommand{\calX}{{\mathcal{X}}}

\newcommand{\calI}{{\mathcal{I}}}

\newcommand{\Reg}{\textsc{Reg}}
\newcommand{\RegS}{\textsc{Reg}_{S}}

\newcommand{\posbias}{\textsc{Pos-Bias}\xspace}
\newcommand{\negbias}{\textsc{Neg-Bias}\xspace}

\newcommand{\target}{\textsc{Target}\xspace}

\newcommand{\injbias}{b}
\newcommand{\Correct}{\lambda}
\newcommand{\baseLR}{\eta}
\newcommand{\metaLR}{\epsilon}

\DeclareMathOperator*{\argmin}{argmin}

\newcommand{\E}{\field{E}}

\newcommand{\inner}[1]{ \left\langle {#1} \right\rangle }
\newcommand{\innersmall}[1]{ \langle {#1} \rangle }

\newcommand{\norm}[1]{\left\|{#1}\right\|}

\def \indi {\mathbbm{1}}

\newcommand{\constone}{C_1}

\newcommand{\ber}{\mathbf{Ber}}

\newcommand{\wh}{\widehat}
\newcommand{\wt}{\widetilde}

\newcommand{\ellhat}{\wh{\ell}}

\newcommand{\ellbar}{\bar{\ell}}

\newcommand{\costhat}{\wh{c}}
\newcommand{\cost}{c}

\newcommand{\posterm}{\textsc{Pos-Bias}\xspace}
\newcommand{\negterm}{\textsc{Neg-Bias}\xspace}
\newcommand{\bias}{\textsc{Deviation}\xspace}

\newcommand{\metareg}{\textsc{Meta}\mbox{-}\textsc{Regret}}
\newcommand{\basereg}{\textsc{Base}\mbox{-}\textsc{Regret}}

\newcommand{\ind}{\mathbbm{1}}

\newcommand{\order}{\ensuremath{\mathcal{O}}}
\newcommand{\otil}{\ensuremath{\tilde{\mathcal{O}}}}

\usepackage{lipsum,booktabs}
\usepackage{amsmath,mathrsfs,amssymb,amsfonts,bm,enumitem}
\usepackage{rotating}
\usepackage{pdflscape}
\usepackage{hyperref,url}
\hypersetup{
    colorlinks,
    breaklinks,
    linkcolor = blue,
    citecolor = blue,
    urlcolor  = blue,
}
\allowdisplaybreaks
\usepackage{appendix}
\usepackage{multirow,makecell}

\usepackage{algorithmic,algorithm}

\renewcommand{\tilde}{\widetilde}
\renewcommand{\hat}{\widehat}

\def \A {\mathcal{A}}
\def \B {\mathbb{B}}
\def \B {\mathcal{B}}

\def \E {\mathbb{E}}

\def \I {\mathcal{I}}

\def \O {\mathcal{O}}

\def \R {\mathbb{R}}
\def \S {\mathcal{S}}
\def \T {\top}

\def \V {\mathcal{V}}

\def \X {\mathcal{X}}

\def \Z {\mathcal{Z}}

\def \Ot {\tilde{\O}}

\usepackage{mathtools}
\let\norm\undefined 
\DeclarePairedDelimiter\norm{\lVert}{\rVert}
\DeclarePairedDelimiter\abs{\lvert}{\rvert}

\DeclareMathOperator*{\poly}{\text{poly}}

\newtheorem{myRequirement}{Requirement}

\usepackage{graphicx,color} 

\definecolor{wine_red}{RGB}{228,48,64}
\definecolor{DSgray}{cmyk}{0,1,0,0}

\usepackage{prettyref}
\newcommand{\pref}[1]{\prettyref{#1}}

\newcommand{\savehyperref}[2]{\texorpdfstring{\hyperref[#1]{#2}}{#2}}
\newrefformat{eq}{\savehyperref{#1}{Eq.~\textup{(\ref*{#1})}}}
\newrefformat{eqn}{\savehyperref{#1}{Eq.~(\ref*{#1})}}
\newrefformat{lem}{\savehyperref{#1}{Lemma~\ref*{#1}}}
\newrefformat{lemma}{\savehyperref{#1}{Lemma~\ref*{#1}}}
\newrefformat{def}{\savehyperref{#1}{Definition~\ref*{#1}}}
\newrefformat{line}{\savehyperref{#1}{Line~\ref*{#1}}}
\newrefformat{thm}{\savehyperref{#1}{Theorem~\ref*{#1}}}
\newrefformat{corr}{\savehyperref{#1}{Corollary~\ref*{#1}}}
\newrefformat{cor}{\savehyperref{#1}{Corollary~\ref*{#1}}}
\newrefformat{sec}{\savehyperref{#1}{Section~\ref*{#1}}}
\newrefformat{app}{\savehyperref{#1}{Appendix~\ref*{#1}}}
\newrefformat{appendix}{\savehyperref{#1}{Appendix~\ref*{#1}}}
\newrefformat{assum}{\savehyperref{#1}{Assumption~\ref*{#1}}}
\newrefformat{ex}{\savehyperref{#1}{Example~\ref*{#1}}}
\newrefformat{fig}{\savehyperref{#1}{Figure~\ref*{#1}}}
\newrefformat{alg}{\savehyperref{#1}{Algorithm~\ref*{#1}}}
\newrefformat{rem}{\savehyperref{#1}{Remark~\ref*{#1}}}
\newrefformat{conj}{\savehyperref{#1}{Conjecture~\ref*{#1}}}
\newrefformat{prop}{\savehyperref{#1}{Proposition~\ref*{#1}}}
\newrefformat{proto}{\savehyperref{#1}{Protocol~\ref*{#1}}}
\newrefformat{prob}{\savehyperref{#1}{Problem~\ref*{#1}}}
\newrefformat{claim}{\savehyperref{#1}{Claim~\ref*{#1}}}
\newrefformat{que}{\savehyperref{#1}{Question~\ref*{#1}}}
\newrefformat{op}{\savehyperref{#1}{Open Problem~\ref*{#1}}}
\newrefformat{fn}{\savehyperref{#1}{Footnote~\ref*{#1}}}
\newrefformat{require}{\savehyperref{#1}{Requirement~\ref*{#1}}}

\def \epsilon {\varepsilon}

\newcommand{\meta}{\textsc{Meta}\mbox{-}\textsc{Regret}\xspace}
\newcommand{\base}{\textsc{Base}\mbox{-}\textsc{Regret}\xspace}

\renewcommand{\ln}{\log}

\newcounter{protocol}

\def \ucirc {\mathring{u}}
\title[Corralling a Larger Band of Bandits]{Corralling a Larger Band of Bandits: \\ A Case Study on Switching Regret for Linear Bandits}
\usepackage{times}

\coltauthor{%
	\Name{Haipeng Luo\nametag{\thanks{Authors are listed in alphabetical order.}}} \Email{haipengl@usc.edu}\\
	\addr University of Southern California
	\AND
	\Name{Mengxiao Zhang\nametag{\footnotemark[1]}} \Email{mengxiao.zhang@usc.edu}\\
	\addr University of Southern California
	\AND
	\Name{Peng Zhao\nametag{\footnotemark[1]}} \Email{zhaop@lamda.nju.edu.cn}\\
	\addr National Key Laboratory for Novel Software Technology, Nanjing University
	\AND
	\Name{Zhi-Hua Zhou\nametag{\footnotemark[1]}} \Email{zhouzh@lamda.nju.edu.cn}\\
	\addr National Key Laboratory for Novel Software Technology, Nanjing University
}
\begin{document}

\maketitle

\begin{abstract}%
We consider the problem of combining and learning over a set of adversarial bandit algorithms with the goal of adaptively tracking the best one on the fly.
The \corral algorithm of~\citet{COLT'17:Corral} and its variants~\citep{NIPS'20:AdaptingMisspec} achieve this goal with a regret overhead of order $\Ot(\sqrt{MT})$ where $M$ is the number of base algorithms and $T$ is the time horizon. 
The polynomial dependence on $M$, however, prevents one from applying these algorithms to many applications where $M$ is $\poly(T)$ or even larger.
Motivated by this issue, we propose a new recipe to corral a larger band of bandit algorithms whose regret overhead has only \emph{logarithmic} dependence on $M$ as long as some conditions are satisfied.
As the main example, we apply our recipe to the problem of adversarial linear bandits over a $d$-dimensional $\ell_p$ unit-ball for $p \in (1,2]$.
By corralling a large set of $T$ base algorithms, each starting at a different time step, our final algorithm achieves the first optimal switching regret $\Ot(\sqrt{d S T})$ when competing against a sequence of comparators with $S$ switches (for some known $S$).
We further extend our results to linear bandits over a smooth and strongly convex domain as well as unconstrained linear bandits.
\end{abstract}

\section{Introduction}
\label{sec: intro}
We consider the problem of combining a set of bandit algorithms to learn the best one on the fly, which has many applications in dealing with uncertainty from the environment.
Indeed, by combining a set of base algorithms, each dedicated for a certain type of environments, the final meta algorithm can then automatically adapt to and perform well in every problem instance encountered, as long as the price of such meta-level learning is small enough.
While such ideas have a long history in online learning, doing so with partial information (that is, bandit feedback) is particularly challenging, and only recently have we seen success in various settings~\citep{COLT'17:Corral, NIPS'20:pacchiano2020model,NIPS'20:AdaptingMisspec,lee2020closer,ICML'21:misspecification,COLT'21:black-box,JMLR'21:BCO, ALT'22:corruption}.

We focus on an adversarial setting where the data are generated in an arbitrary and potentially malicious manner.
The closest work is~\citep{COLT'17:Corral}, where a generic algorithm called \corral is developed to learn over a set of $M$ base algorithms with extra regret overhead $\Ot(\sqrt{MT})$ after $T$ rounds.
In order to maintain $\Ot(\sqrt{T})$ overall regret, which is often the optimal bound and the goal when designing bandit algorithms, \corral can thus at most tolerate $M = \poly(\log T)$ base algorithms.
However, there are many applications where $M$ needs to be much larger to cover all possible scenarios of interest (we will soon provide an example where $M$ needs to be of order $T$).
Therefore, a natural question arises: \emph{can we corral an even larger band of bandit algorithms, ideally with only logarithmic dependence on $M$ in the regret?}

As an attempt to answer this question, we focus on the adversarial linear bandit problem and develop a new recipe to combine base algorithms, which reduces the problem to designing good unbiased loss estimators for the base algorithms and good optimistic loss estimators for the meta algorithm.
As long as these estimators ensure certain properties, the resulting algorithm enjoys logarithmic dependence on $M$ in the regret.
We discuss this recipe in detail along with a warm-up example on the classic multi-armed bandit problem in \pref{sec: corral-larger}.

Then, as a main example, in \pref{sec: lp_ball} we apply this recipe to develop the first optimal switching regret bound for adversarial linear bandits over a $d$-dimensional $\ell_p$ unit ball with $p \in (1,2]$.
Switching regret measures the learner's performance against a sequence of changing comparators with $S$ switches, 
and a standard technique to achieve so in the full-information setting is by combining $T$ base algorithms, each of which starts at a different time step and is guaranteed to perform well against a fixed comparator starting from this step (that is, a standard static regret guarantee); see for example \citep{hazan2007adaptive, ICML'15:Daniely-adaptive, luo2015achieving}.
Applying the same idea to bandit problems was not possible before because as mentioned, previous methods such as \corral cannot afford $T$ base algorithms.\footnote{One can compromise and corral $o(T)$ base algorithms instead, which leads to suboptimal switching regret; see such an attempt in~\citep[Appendix~G]{luo2018efficient}.}
However, this is exactly where our approach shines.
Indeed, by using our recipe to combine $T$ instances of the algorithm of~\citep{ALT'18:hybrid-barrier} together with carefully designed loss estimators, 
we manage to achieve logarithmic dependence on the number of base algorithms, resulting in the optimal (up to logarithmic factors) switching regret $\Ot(\sqrt{dST})$ for this problem for any fixed $S$.
As another example, in~\pref{appendix: strongly cvx set} we also generalize our results from $\ell_p$ balls to smooth and strongly convex sets.

Finally, in \pref{sec: unconstrained} we further generalize our results to the unconstrained linear bandit problem and obtain the first comparator-adaptive switching regret of order $\otil\big(\max_{k\in [S]}\|\ucirc_k\|_2 \cdot \sqrt{dST}\big)$ where $\ucirc_k$ is the $k$-th (arbitrary) comparator.
The algorithm requires two components, one of which is exactly our new algorithm developed for $\ell_p$ balls, the other being a new parameter-free algorithm for unconstrained Online Convex Optimization with the first comparator-adaptive switching regret. 
We note that this latter algorithm/result might be of independent interest.

\paragraph{High-level ideas.}
For such as a meta learning framework,
it is standard to decompose the overall regret as $\meta$, which measures the regret of the meta algorithm to the best base algorithm, and $\base$, which measures the best base algorithm to the best elementary action.
The main difficulty for bandit problems is that, it is hard to control $\base$ in such a framework due to possible starvation of feedback for the base algorithm.
The \corral algorithm of~\citet{COLT'17:Corral} addresses this via a new meta algorithm based on Online Mirror Descent (OMD) with the log-barrier regularizer and an increasing learning rate schedule, which together provides a negative term in $\meta$ large enough to (approximately) cancel $\base$.
However, the log-barrier regularizer unavoidably introduces $\poly(M)$ dependence in $\meta$.

Our ideas to address this issue are two-fold.
First, to make sure $\meta$ enjoys logarithmic dependence on $M$, 
we borrow the idea of the \expfour algorithm~\citep{SICOMP'02:Auer-EXP3}, which combines $M$ static experts (instead of learning algorithms) without paying polynomial dependence on $M$.
This is achieved by OMD with the negative entropy regularizer, plus a better loss estimator with lower variance for each expert.
In our case, this requires coming up with similar low-variance loss estimator for each base algorithm as well as updating each base algorithm no matter whether it is selected by the meta algorithm or not (in contrast, \corral only updates the selected base algorithm).
Without the log-barrier regularizer, however, we now cannot use the same increasing learning rate schedule as \corral to generate a large enough negative term to cancel $\base$.
To address this, our second main idea is to inject negative bias to the loss estimators (making term optimistic underestimators), with the goal of  generating a reasonably small positive bias in the regret and at the same time a large enough negative bias to cancel $\base$.
This idea is similar to that of~\citep{NIPS'20:AdaptingMisspec}, but they did not push it far enough and only improved \corral on logarithmic factors.
\paragraph{Related work.} 
Since the work of~\citet{COLT'17:Corral}, there have been several follow-ups in the same direction, either for adversarial environments~\citep{NIPS'20:AdaptingMisspec} or stochastic environments~\citep{NIPS'20:pacchiano2020model,ICML'21:dynamic-balance,AISTATS'21:corral-stochastic, ICML'21:misspecification}. 
The problem is also highly related to model selection in online learning with bandit feedback~\citep{NIPS'19ModelSelect,COLT'20:open-model-selection,NIPS'21Pareto}.

The optimal regret for adversarial linear bandits over a general $d$-dimensional set is $\otil(d\sqrt{T})$ \citep{NIPS'07:bandit-lower-bound, COLT'12:OptimalOLO}, but it becomes $\otil(\sqrt{dT})$ for the special case of $\ell_p$ balls with $p\in[1,2]$~\citep{ALT'18:hybrid-barrier}.
To the best of our knowledge, switching regret has not been studied for adversarial linear bandits, except for its special case of multi-armed bandits~\citep{SICOMP'02:Auer-EXP3,JMLR'10:AudibertB10}.
We discuss several natural attempts in \pref{appendix:potential-methods} to extend existing methods to linear bandits, but the best we can get is $\Ot(d\sqrt{ST})$ via combining the \exptwo algorithm~\citep{COLT'12:OptimalOLO} and the idea of uniform mixing~\citep{journals/ml/HerbsterW98,SICOMP'02:Auer-EXP3}.
On the other hand, our proposed approach achieves the optimal $\Ot(\sqrt{dST})$ regret.
In fact, our algorithm is also more computationally efficient as \exptwo requires log-concave sampling.

We assume a known and fixed $S$ in most places. Achieving the same result for \emph{all} $S$ simultaneously is known to be impossible for adaptive adversaries~\citep{NIPS'21Pareto}, and remains open for oblivious adversaries (our setting) even for the classic multi-armed bandit problem, so this is beyond the scope of this work.
We mention that, however, without knowing $S$ we can still achieve $\Ot(S\sqrt{dT})$ regret via a slightly different parameter tuning of our algorithm, or $\otil(\sqrt{dST}+T^{\nicefrac{3}{4}})$ regret via wrapping our algorithm with the generic Bandits-over-Bandits strategy of~\citet{cheung2019learning}.
As a final remark, note that for the easier stochastic environments, adapting to unknown $S$ without price has been shown possible; see~\citep{COLT'21:black-box} and references therein.

Regarding our extension to the unconstrained setting, while unconstrained online learning has been extensively studied in the full-information setting with gradient feedback since the work of~\citep{mcmahan2012no} (see e.g.~\citep{NIPS'13:unconstrainedEG,COLT'14:unconstrained-Hilbert,NIPS'15:Dylan-adaptive,COLT'17:Ashok,COLT'18:black-box-reduction}), as far as we know~\citep{NIPS'20:Dirk} is the only existing work considering the same with bandit feedback.
They consider static regret and propose a black-box reduction approach, taking inspiration from a similar reduction from the full-information setting~\citep{COLT'18:black-box-reduction}. 
We consider the more general switching regret, and our algorithm is also built on a similar reduction.

\section{Problem Setup and Notations}\label{sec: pre}

\paragraph{Problem setup.} 
While our idea is applicable to more general setting, for ease of discussions we focus on the adversarial linear bandit problem throughout the paper. 
Specifically, at the beginning of a $T$-round game, an adversary (knowing the learner's algorithm) secretly chooses a sequence of linear loss functions parametrized by $\ell_1,\ldots,\ell_T \in \R^d$. 
Then, at each round $t \in [T]$, the learner makes a decision by picking a point (also called action) $x_t$ from a known feasible domain $\X \subseteq \R^d$, and subsequently suffers and observes the loss $\ell_t^\T x_t$. 
Note that $\ell_t^\T x_t$ is the only feedback on $\ell_t$ revealed to the learner.
We measure the learner's performance via the \emph{switching regret}, defined as
\begin{equation}
    \label{eq:switching-regret-bound}
     \Reg(u_1,\ldots,u_T) \triangleq \sum_{t=1}^T \ell_t^\T x_t - \sum_{t=1}^T \ell_t^\T u_{t} = \sum_{k=1}^S \sum_{t \in \I_k} \ell_t^\T (x_t - \ucirc_{k}),
\end{equation}
where $u_1,\ldots,u_T \in \calX$ is a sequence of arbitrary comparators with $S-1$ switches for some known $S$ (that is, $\sum_{t=2}^T \mathbbm{1}\{u_{t-1} \neq u_t\} = S-1$) and $\I_1,\ldots,\I_S$ denotes a partition of $[T]$ such that for each $k$, $u_t$ remains the same (denoted by $\ucirc_k$) for all $t \in \I_k$.
Except for comparator-adaptive bounds discussed in \pref{sec: unconstrained}, our results have no explicit dependence on $\ucirc_1, \ldots, \ucirc_S$ other than the number $S$, so we often use $\RegS$ as a shorthand for $\Reg(u_1,\ldots,u_T)$.
The classic static regret is simply $\Reg_1$ (that is, competing with a fixed comparator throughout), which we also simply write as $\Reg$.

\paragraph{Notations.} For any integer $n$, we denote by $[n]$ the set $\{1,2,\ldots,n\}$, and $\Delta_n$ the simplex $\{ p \in \R_{\geq 0}^n \mid \sum_{i=1}^n p_i=1\}$.
We use $e_i$ to denote the standard basis vector (of appropriate dimension) with the $i$-th coordinate being 1 and others being 0. 
Given a vector $x \in \R^d$, its $\ell_p$ norm is defined by $\norm{x}_p = (\sum_{n=1}^d |x_n|^p)^{1/p}$. $\E_t[\cdot]$ denotes the conditional expectation given the history before round $t$.  The $\Ot(\cdot)$ notation omits the logarithmic dependence on the time horizon $T$ and the dimension $d$. For a differential convex function $\psi: \R^d \mapsto \R$, the induced Bregman divergence is defined by $D_{\psi}(x,y) = \psi(x) - \psi(y) - \inner{\nabla \psi(y), x-y}$.
\section{Corralling a Larger Band of Bandits: A Recipe}\label{sec: corral-larger}

\begin{algorithm}[t]
\floatname{algorithm}{Protocol}
   \caption{Combining $M$ base algorithms in adversarial linear bandits}
   \label{proto:corral}
    \For{$t=1,\ldots,T$}{
        Each base algorithm $\calB_i$ submits an action $\wt{a}_t^{(i)}\in \calX$ to the meta algorithm, for all $i\in[M]$.
        
        Meta algorithm selects $x_t$ such that $\mathbb{E}_t[x_t]=\sum_{i\in [M]}p_{t,i}\wt{a}_{t}^{(i)}$ for some distribution $p_t\in\Delta_M$.
        
        Play $x_t$ and receive feedback $\ell_t^\top x_t$.
        
        Construct base loss estimator $\ellhat_t\in\R^d$ and meta loss estimator $\costhat_t\in\mathbb{R}^M$.
        
        Base algorithms $\{\calB_i\}_{i=1}^M$ update themselves based on the base loss estimator $\ellhat_t$.

        Meta algorithm updates the weight $p_{t+1}$ based on $p_t$ and the meta loss estimator $\costhat_t$.
    }   
\end{algorithm}

In this section, we describe our general recipe to corral a large set of bandit algorithms.
We start by showing a general and natural protocol of such a meta-base framework in \pref{proto:corral}.
Specifically, suppose we maintain $M$ base algorithms $\{\calB_i\}_{i=1}^M$.
At the beginning of each round, each base algorithm $\calB_i$ submits its own action $\wt{a}_t^{(i)}\in \calX$ to the meta algorithm, which then decides the final action $x_t$ with expectation $\sum_{i\in [M]}p_{t,i}\wt{a}_{t}^{(i)}$ for some distribution $p_t\in\Delta_M$ specifying the importance/quality of each base algorithm.
After playing $x_t$ and receiving the feedback $\ell_t^\top x_t$,
we construct \emph{base loss estimator} $\ellhat_t\in\R^d$ and \emph{meta loss estimator} $\costhat_t\in\mathbb{R}^M$.
As their name suggests, base loss estimator estimates $\ell_t$ and is used to update each base algorithm, while meta loss estimator estimates $A_t^\top\ell_t$,
where the $i$-th column of $A_t \in \R^{d\times M}$ is $\wt{a}_t^{(i)}$,
and is used to update the meta algorithm to obtain the next distribution $p_{t+1}\in\Delta_M$.

In the following, we formalize the high-level idea discussed in \pref{sec: intro}.
For simplicity, we focus on the static regret $\Reg$ in this discussion (that is, $S=1$) and let $u$ be the fixed comparator.
The first step is to decompose the regret into two parts as mentioned in \pref{sec: intro}:
as long as $\ellhat_t$ and $\costhat_t$ are unbiased estimators (that is, $\mathbb{E}_t[\ellhat_t]=\ell_t$ and $\mathbb{E}_t[\costhat_t]=A_t^\top\ell_t$), one can show:
\begin{align}\label{eqn:decompose-corral}
\forall j\in[M],\;\;  \mathbb{E}[\Reg] = \underbrace{\mathbb{E}\Bigg[\sum_{t=1}^T\inner{p_t-e_j, \costhat_t}\Bigg]}_{\meta} + \underbrace{\mathbb{E}\Bigg[\sum_{t=1}^T\inner{\wt{a}_t^{(j)}-u, \ellhat_t}\Bigg]}_{\base}.
\end{align}
Controlling $\base$ is the key challenge.
Indeed, even if the base algorithm enjoys a good regret guarantee when running on its own, it might not ensure the same guarantee any more in this meta-base framework because it cannot fully control the final action and observe the feedback it needs.
At a technical level, this is reflected in a larger variance of $\ellhat_t$ due to the randomness from the meta algorithm, which then ruins the base algorithm's original regret guarantee.

As mentioned, the way \corral~\citep{COLT'17:Corral} addresses this issue is by using  OMD with the log-barrier regularizer and increasing learning rates as the meta algorithm, which ensures that $\meta$ is at most $\otil(\sqrt{MT})$ plus some \emph{negative} term large enough to cancel the prohibitively large part of \base.
The $\poly(M)$ dependence in their approach is unavoidable because they treat the problem that the meta algorithm is facing as a classic multi-armed bandit problem and ignores the fact that information can be shared among different base algorithms.
The recent follow-up \citep{NIPS'20:AdaptingMisspec} shares the same issue.

Instead, we propose the following idea.
We use OMD with entropy regularizer (a.k.a. multiplicative weights update) as the meta algorithm to update $p_{t+1}$, usually in the form $p_{t+1, i}\propto p_{t,i}e^{-\epsilon\costhat_{t,i}}$ where $\epsilon > 0$ is some learning rate.
This first ensures that the so-called regularization penalty term in $\meta$ is of order $\frac{\log M}{\epsilon}$ instead of $\frac{M}{\epsilon}$ as in \corral.
To control the other so-called stability term in $\meta$, the estimator $\costhat_t$ has to be constructed in a way with low variance, but we defer the discussion and first look at how to control \base in this case.
Since we are no longer using the log-barrier regularizer of \corral, a different way to generate a large negative term in $\meta$ to cancel $\base$ is needed.
To this end, we propose to inject a (negative) bias $b_t\in \mathbb{R}_{+}^M$ to the meta loss estimator $\costhat_t$, making it an optimistic underestimator.
More specifically, introduce another notation $\cost_t$ for some unbiased estimator of $A_t^\top\ell_t$. Then the adjusted meta loss estimator is defined as 
$\costhat_t=\cost_t-b_t$. 
Since $\costhat_t$ is biased now, the decomposition~\eqref{eqn:decompose-corral} needs to be updated accordingly as
\begin{align*}
\mathbb{E}[\Reg] = \underbrace{\mathbb{E}\Bigg[\sum_{t=1}^T\inner{p_t-e_j, \costhat_t}\Bigg]}_{\meta} + \underbrace{\mathbb{E}\Bigg[\sum_{t=1}^T\inner{\wt{a}_t^{(j)}-u, \ellhat_t}\Bigg]}_{\base}
+ \underbrace{\mathbb{E}\Bigg[\sum_{t=1}^T\inner{p_t, b_t}\Bigg]}_{\posbias}
- \underbrace{\mathbb{E}\Bigg[\sum_{t=1}^T\inner{e_j, b_t}\Bigg]}_{\negbias}.
\end{align*}
Based on this decomposition, our goal boils down to designing good base and meta loss estimators such that the following three terms are all well controlled:
\begin{align}
    &\base - \negbias \leq \target, \label{eqn:cancel-stab} \\
    &\posbias \leq \target,\label{eqn:pos-term-bias}\\
    &\meta \leq \target. \label{eqn:master_reg} 
\end{align}
Here, $\target$ represents the final targeted regret bound with logarithmic dependence on $M$ and usually $\sqrt{T}$-dependence on $T$ (such as $\otil(\sqrt{dST\log M})$ for our main application of switching regret discussed in \pref{sec: lp_ball}).

\paragraph{A recipe.}
We are now ready to summarize our recipe in the following three steps.
\begin{itemize}[leftmargin=*]
  \setlength\itemsep{0em}
\item \textbf{Step 1.}  Start from designing $\ellhat_t$, which often follows similar ideas of the original base algorithm.

\item \textbf{Step 2.} Then, by analyzing \base with such a base loss estimator, figure out what $b_t$ needs to be in order to ensure \pref{eqn:cancel-stab} and \pref{eqn:pos-term-bias} simultaneously.

\item \textbf{Step 3.} Finally, design $\cost_t$ to ensure \pref{eqn:master_reg}. As mentioned in \pref{sec: intro}, this is a problem similar to combining static experts as in the \expfour algorithm~\citep{SICOMP'02:Auer-EXP3}, and the key is to ensure that $\cost_t$ allows information sharing between base algorithms and enjoys low variance.
A natural choice is $\cost_{t,i}=\langle\wt{a}_t^{(i)}, \ellhat_t\rangle$, which is exactly what \expfour does and works in the toy example we show below, but sometimes one needs to replace $\ellhat_t$ with yet another better unbiased estimator of $\ell_t$, which turns out to be indeed the case for our main example in \pref{sec: lp_ball}.
\end{itemize}

\paragraph{A toy example.}
Now, we provide a warm-up example to show how to successfully apply our three-step recipe to the multi-armed bandit problem.
We note that this example does \emph{not} really lead to meaningful applications, as in the end we are simply combining different copies of the exact same algorithm.
Nevertheless, this serves as a simple and illustrating exercise to execute our recipe, paving the way for the more complicated scenario to be discussed in the next section.

Specifically, in multi-armed bandit, we have $\calX = \Delta_d$ and $\ell_t\in[0,1]^d$ for all $t\in[T]$, and we set the target to be $\target = \otil(\sqrt{dT\log M})$ (optimal up to logarithmic factors).
The meta algorithm is as specified before (multiplicative weights update).	
For the base algorithm, we choose a slight variant of the classic \expthree algorithm~\citep{SICOMP'02:Auer-EXP3}, so that $\wt{a}_{t+1}^{(i)}=\argmin_{a\in \Delta_d\cap [\eta, 1]^d}\big\{\langle a, \ellhat_t\rangle+\frac{1}{\eta}D_{\psi}(a,a_t^{(i)})\big\}$, where 
$\eta > 0$ is a clipping threshold (and also a learning rate) and $\psi(a)=\sum_{n=1}^da_n\log a_n$ is the negative entropy.
Given $q_t = \sum_{i=1}^{M}p_{t,i}\wt{a}_{t}^{(i)} \in \Delta_d$, the meta algorithm naturally samples an arm $n_t \in [d]$ according to $q_t$, meaning $x_t = e_{n_t}$.

\paragraph{Step 1.} With the feedback $\ell_t^\top x_t = \ell_{t,n_t}$, 
following \expthree we let the base loss estimator be the standard importance-weighted estimator: 
$\ellhat_{t}=\frac{\ell_{t,n_t}}{q_{t,n_t}}x_t $, which is clearly unbiased with $\mathbb{E}_t[\ellhat_t]=\ell_t$.

\paragraph{Step 2.}
By standard analysis (e.g.~\citep[Theorem 3.1]{bubeck2012regret}), 
$\base$ is at most $\eta dT+ \frac{\log d}{\eta} + \eta\mathbb{E}\big[\sum_{t=1}^T\sum_{n=1}^d\wt{a}_{t,n}^{(j)}\ellhat_{t,n}^2\big]$.
Since $\mathbb{E}_t[\ellhat_{t,n}^2] = \frac{\ell_{t,n}^2}{q_{t,n}}$,
the last term is further bounded by $\eta \mathbb{E}\big[\sum_{t=1}^T\sum_{n=1}^d\nicefrac{\wt{a}_{t,n}^{(j)}}{q_{t,n}}\big]$.
This is exactly the problematic stability term that can be prohibitively large.
We thus directly define the bias term $b_{t,j}$ as $\eta \sum_{n=1}^d\nicefrac{\wt{a}_{t,n}^{(j)}}{q_{t,n}}$,
so that $\base - \negbias$ is simply bounded by $\eta dT + \frac{\log d}{\eta}$.
Picking the optimal $\eta$ ensures \pref{eqn:cancel-stab}.
On the other hand, $\posbias$ happens to be small as well: $\posbias= \mathbb{E}\big[\sum_{t=1}^T\langle p_{t}, b_{t} \rangle\big] 
=\eta\mathbb{E}\big[\sum_{t=1}^T\sum_{i=1}^M p_{t,i} \sum_{n=1}^d\nicefrac{\wt{a}_{t,n}^{(i)}}{q_{t,n}} \big] = \eta\mathbb{E}\big[\sum_{t=1}^T \sum_{n=1}^d\nicefrac{q_{t,n}}{q_{t,n}} \big] = \eta dT$,
ensuring \pref{eqn:pos-term-bias}.

\paragraph{Step 3.}
Finally, we use the natural meta loss estimator $\cost_{t,i}=\langle\wt{a}_t^{(i)}, \ellhat_t\rangle$.
Since $q_{t,n}\geq \eta$ due to the clipping threshold and thus $0 \leq b_{t,i} \leq 1$ and $\costhat_{t,i} \geq -1$ (that is, not too negative), 
standard analysis shows $\meta \leq \frac{\log M}{\metaLR} + \metaLR\mathbb{E}\big[\sum_{t=1}^T\sum_{i=1}^Mp_{t,i}\costhat_{t,i}^2\big]$,
with the last term further bounded by $2\metaLR\mathbb{E}\big[\sum_{t=1}^T\sum_{i=1}^M (p_{t,i}\cost_{t,i}^2 + p_{t,i}b_{t,i}^2)\big] \leq 4\epsilon dT$.
Picking the optimal $\epsilon$ in the final bound $\meta \leq \frac{\log M}{\metaLR} + 4\epsilon dT$ then ensures \pref{eqn:master_reg}. 
This concludes our example and shows that our recipe indeed enjoys logarithmic dependence on $M$ in this case, which \corral fails to achieve.

\section{Optimal Switching Regret for Linear Bandits over $\ell_p$ Balls}
\label{sec: lp_ball}

As the main application in this work, we now discuss how to apply our recipe to achieve the optimal switching regret for adversarial linear bandits over $\ell_p$ balls.
In this problem, the feasible domain is an $\ell_p$ unit ball for some $p\in (1,2]$, namely, $\calX=\{x \in \mathbb{R}^d \mid \|x\|_p\leq 1\}$,
and each $\ell_t$ is assumed to be from the dual $\ell_q$ unit ball with $q = p/(p-1)$, such that $\abs{\ell_t^\T x} \leq 1$ for all $x \in \X$ and $t \in [T]$.
\citet{ALT'18:hybrid-barrier} show that the optimal regret in this case is $\Theta(\sqrt{dT})$, which is better than the general linear bandit problem by a factor of $\sqrt{d}$.
This implies that the optimal switching regret for this problem is $\Omega(\sqrt{dST})$ --- indeed, simply consider the case where $\I_1, \ldots, \I_S$ is an even partition of $[T]$ and the adversary forces the learner to suffer $\Omega(\sqrt{d |\I_k|}) = \Omega(\sqrt{d T/S})$ regret on each interval $\I_k$ by generating a new worst case instance regarding the static regret.
Therefore, our target regret bound here is set to $\target = \otil(\sqrt{dST})$.
We remind the reader that this problem has not been studied before and that in \pref{appendix:potential-methods}, we discuss other potential approaches and why none of them is able to achieve this goal.

The pseudocode of our final algorithm is shown in \pref{alg:ell-p-ball}. 
At a high-level, it is simply following the standard idea in the literature on obtaining switching regret, that is, maintain a set of $M=T$ base algorithms with static regret guarantees, the $t$-th of which $\calB_t$ starts learning from time step $t$ (before time $t$, one pretends that $\calB_t$ picks the same action as the meta algorithm).
If the meta algorithm itself enjoys a switching regret guarantee,\footnote{We point out that in the full-information setting, even a certain static regret guarantee from the meta algorithm is enough, but a switching regret guarantee is needed in the bandit setting for technical reasons.} then by competing with $\calB_{j_k}$ on interval $\I_k$ where $j_k$ is the first time step of $\I_k$ so that $\calB_{j_k}$ enjoys a (static) regret guarantee on $\I_k$,
the overall algorithm enjoys a switching regret for the original problem.
While this is a standard and simple idea, applying it to the bandit setting was not possible before our work due to the large number of base algorithms ($T$) needed to be combined.
Our approach, however, is able to overcome this with logarithmic dependence on $M$, making it the first successful execution of this long-standing idea in bandit problems.

\paragraph{Base algorithm overview.}
Our base algorithm is naturally the one proposed in~\citep{ALT'18:hybrid-barrier} that achieves $\otil(\sqrt{dT})$ static regret.\footnote{To be more accurate, the version we present here is a slightly simpler variant with the same guarantee.} 
Specifically, let $\calX'=\{x \mid \|x\|_p\leq 1-\gamma\}$ for some clipping parameter $\gamma$ be a slightly smaller ball.
At each round $t$, each base algorithm $\B_i$ (for $i \leq t$) has a vector $a_t^{(i)} \in \calX'$ at hand.
Then, it generates a Bernoulli random variable $\xi_t^{(i)}$ with mean $\|a_t^{(i)}\|_p$. If $\xi_t^{(i)}=0$, then its final decision $\wt{a}_t^{(i)}$ is uniformly sampled from $\{ \pm e_n \}_{n=1}^d$; otherwise, $\wt{a}_t^{(i)}=a_t^{(i)}/\|a_t^{(i)}\|_p$. Next, $\B_i$ submits $(\wt{a}_t^{(i)}, a_t^{(i)}, \xi_t^{(i)})$ to the meta algorithm. After receiving the base loss estimator $\ellhat_t$ (to be specified later), $\B_i$ updates $a_{t+1}^{(i)}$ using OMD with the regularizer $R(a)=-\ln(1-\|a\|_p^p)$, that is,
$a_{t+1}^{(i)} = \argmin_{a\in \calX'}\big\{ \langle a,\ellhat_t \rangle+\frac{1}{\baseLR}D_R(a,a_t^{(i)})\big\}$ for some learning rate $\eta > 0$.
We defer the pseudocode \pref{alg:ell-p-ball-base} to \pref{appendix:lp_ball_base}.

\begin{algorithm}[!t]
   \caption{Algorithm for adversarial linear bandits over $\ell_p$ balls with optimal switching regret}
   \label{alg:ell-p-ball}
    \textbf{Input:} 
    clipping parameter $\gamma$, base learning rate $\eta$, meta learning rate $\metaLR$, mixing rate $\mu$, exploration parameter $\beta$, bias coefficient $\lambda$, initial uniform distribution ${p}_1\in \Delta_T$.
    
   \For{$t=1, \ldots, T$}{
      \nl Start a new base algorithm $\B_t$, which is an instance of~\pref{alg:ell-p-ball-base} with learning rate $\eta$, clipping parameter $\gamma$, and initial round $t$.\label{line:initialize base}
      
      \nl Receive local decision $(\wt{a}_t^{(i)},a_t^{(i)},\xi_t^{(i)})$ from base algorithm $\B_i$ for each $i \leq t$. \label{line:receive local}
      
      \nl Compute the renormalized distribution $\wh{p}_t \in \Delta_t$ such that $\wh{p}_{t,i}\propto {p}_{t,i}$ for $i\in[t]$. \label{line:renormalize}
      
      \nl Sample a Bernoulli random variable $\rho_t$ with mean $\beta$. 
      If $\rho_t=1$, uniformly sample $x_t$ from $\{\pm e_n\}_{n=1}^d$;
      otherwise, sample $i_t \in [t]$ according to $\wh{p}_{t}$, and set $x_t = \tilde{a}_t^{(i_t)}$ and $\xi_t = \xi_t^{(i_t)}$. \label{line:make prediction}

      \nl Make the final decision $x_t$ and receive feedback $\ell_t^\top x_t$.

      \nl Construct the base loss estimator $\wh{\ell}_t \in \R^d$ as follows and send it to all base algorithms $\{\calB_i\}_{i=1}^t$: \label{line:estimator base}
      \begin{align}\label{eq:base-estimator}
          \wh{\ell}_t=\frac{\indi\{\rho_t=0\}\indi\{\xi_t=0\}}{1-\beta}\cdot \frac{d(\ell_t^\T x_t)}{1-\sum_{i=1}^t \wh{p}_{t,i}\|a_{t}^{(i)}\|_p} \cdot x_t.
      \end{align}
           

      \nl Construct another loss estimator $\bar{\ell}_t \in \R^d$ as
      \begin{align}\label{eq:meta-estimator}
          \bar{\ell}_t=\wt{M}_t^{-1}x_tx_t^\top\ell_t,
      \end{align}
      where $\wt{M}_t=\frac{\beta}{d}\sum_{n=1}^de_ne_n^\top+(1-\beta)\sum_{i=1}^t \wh{p}_{t,i}\wt{a}_t^{(i)}\wt{a}_t^{(i)^\top}$. \label{line:estimator meta pre}

      \nl Construct the meta loss estimator $\costhat_t\in \mathbb{R}^T$ as: 
      \begin{align}\label{eqn:meta_loss_estimator}
          \costhat_{t,i} = \begin{cases}
                   \innersmall{\wt{a}_t^{(i)}, \bar{\ell}_t}-b_{t,i}, &\text{$i \leq t$,}\\
                      \sum_{j=1}^t \wh{p}_{t,j}\costhat_{t,j}, &\text{$i > t$,}
          \end{cases}
          \quad\text{ where }\;
          b_{t,i}=\frac{1}{\Correct T(1-\beta)}\frac{1-\|a_t^{(i)}\|_p}{1-\sum_{j=1}^t \wh{p}_{t,j}\|a_t^{(j)}\|_p}.
      \end{align}\label{line:estimator meta}
      
      \nl Meta algorithm updates the weight ${p}_{t+1}\in\Delta_T$ according to \label{line:meta update}
        \begin{align}\label{eqn:meta-strategy-update}
        {p}_{t+1,i} = (1-\mu)\frac{{p}_{t,i} \exp(-\metaLR \costhat_{t,i})}{\sum_{j=1}^T {p}_{t,j} \exp(-\metaLR \costhat_{t,j})} + \frac{\mu}{T}, \quad\forall i \in [T].
        \end{align}

      }
\end{algorithm}

\paragraph{Meta algorithm overview.}
The meta algorithm maintains the distribution $p_t \in \Delta_T$ again via multiplicative weights update, but since a switching regret guarantee is required as mentioned, a slight variant studied in~\citep{SICOMP'02:Auer-EXP3} is needed which mixes the multiplicative weights update with a uniform distribution: $p_{t+1,i} = (1-\mu)\frac{{p}_{t,i} \exp(-\metaLR \costhat_{t,i})}{\sum_{j=1}^T {p}_{t,j} \exp(-\metaLR \costhat_{t,j})} + \frac{\mu}{T}$ for some mixing rate $\mu$, learning rate $\metaLR$, and meta loss estimator $\costhat_t$ (to be specified later).
As mentioned, at time $t$, all base algorithm $\calB_i$ with $i > t$ should be thought of as making the same decision as the meta algorithm,
so in a sense we are looking for an action $\wt{x}_t$ such that $\wt{x}_t = \sum_{i=1}^t p_{t,i} \wt{a}_t^{(i)} + \sum_{i=t+1}^T p_{t,i} \wt{x}_t$, or equivalently $\wt{x}_t = \sum_{i=1}^t \wh{p}_{t,i} \wt{a}_t^{(i)}$ with a distribution $\wh{p}_{t} \in \Delta_t$ satisfying $\wh{p}_{t,i} \propto p_{t,i}$.
Combining this with some extra exploration for technical reasons, the final decision $x_t$ of our algorithm  is decided as follows:
sample a Bernoulli random variable $\rho_t$ with mean $\beta$ (a small parameter);
if $\rho_t = 1$, then $x_t$ is uniformly sampled from $\{ \pm e_n\}_{n=1}^d$,
otherwise $x_t$ is sampled from $\wt{a}_t^{(1)}, \ldots, \wt{a}_t^{(t)}$ according to the distribution $\wh{p}_{t}$.
See \pref{line:renormalize}, \pref{line:make prediction}, and \pref{line:meta update}.
We are now ready to follow the three steps of our recipe to design the loss estimators.

\paragraph{Step 1.}
The design of the base loss estimator $\ellhat_t$ mostly follows~\citep{ALT'18:hybrid-barrier}, except for the extra consideration due to the sampling scheme of the meta algorithm (\pref{line:make prediction}).
The final form is shown in~\pref{eq:base-estimator},
and direction calculation verify its unbiasedness $\mathbb{E}_t[\ellhat_t]=\ell_t$ (see~\pref{lem:unbiasedness}).

\paragraph{Step 2.}
With $\ellhat_t$ fixed, for an interval $\I_k$, we analyze the static regret of $\calB_{j_k}$ on this interval (recall that $j_k$ is the first time step of $\I_k$), mostly following the analysis of~\citep{ALT'18:hybrid-barrier}. 
This corresponds to $\base$ (since we have moved from static regret to switching regret).
More concretely, in \pref{lem:base-reg} we show for some universal constant $\constone > 0$:
\begin{align*}
    \mathbb{E}\Bigg[\sum_{t\in\calI_k}\inner{\wt{a}_t^{(j_k)}-\ucirc_k, \ellhat_t}\Bigg]\leq \frac{\log(1/\gamma)}{\eta}+\eta\constone\sum_{t\in\calI_k}\frac{1-\|a_t^{(j_k)}\|_p}{1-\sum_{j=1}^t \wh{p}_{t,j} \|a_t^{(j)}\|_p}.
\end{align*}
Again, the second term above is the prohibitively large term, and we thus define $b_{t,i}$ in the same form; see \pref{eqn:meta_loss_estimator}.
As long as the parameters are chosen such that $\eta\constone \leq \frac{1}{\lambda T(1-\beta)}$, $\base-\negbias$ is simply bounded by $\frac{\log(1/\gamma)}{\eta}$, and \pref{eqn:cancel-stab} can be ensured.
Direct calculation shows that with such a bias term $b_{t,i}$, $\posbias$ is also small enough to ensure \pref{eqn:pos-term-bias}; see \pref{appendix:bound-part1}.

\paragraph{Step 3.}
Finally, it remains to design unbiased loss estimator $\cost_{t,i}$ and finalize the meta loss estimator $\costhat_{t,i}$.
As mentioned, a natural choice would be $\cost_{t,i}=\langle \wt{a}_t^{(i)}, \ellhat_t \rangle$.
However, despite its unbiasedness, it turns out to suffer a large variance in this case and cannot lead to a favorable guarantee for $\meta$.
To address this issue, we introduce yet another unbiased loss estimator $\ellbar_t$ for $\ell_t$, defined in~\pref{eq:meta-estimator}, which follows standard idea from the general linear bandit literature (see for example the \exptwo algorithm of~\citep{COLT'12:OptimalOLO}).
With that, $\cost_{t,i}$ is defined as $\langle \wt{a}_t^{(i)}, \ellbar_t \rangle$ instead, which now has a small enough variance.
We find it intriguing that using different unbiased loss estimators  ($\ellhat_t$ for base algorithms and $\ellbar_t$ for the meta algorithm) for the same quantity $\ell_t$ appears to be necessary for this problem.
As the final piece of the puzzle, we set $\costhat_{t,i} = \cost_{t,i} - b_{t,i}$ for $i \leq t$ as our recipe describes, and for $i > t$, recall that these base algorithms are thought of as making the same prediction of the meta algorithm, thus we set $\costhat_{t,i} = \sum_{j=1}^t \wh{p}_{t,j}\costhat_{t,j}$; see \pref{eqn:meta_loss_estimator}.
This ensures an important property $\langle p_t, \costhat_t\rangle = \sum_{i\leq t}\wh{p}_{t,i} \costhat_{t,i}$, which we use to finally prove that $\meta$ is small enough to ensure \pref{eqn:master_reg} (see \pref{lemma:meta-final}).

This concludes the description of our entire algorithm.
We formally prove in \pref{appendix:lp_ball} that our algorithm enjoys the following switching regret guarantee.
\begin{theorem}
\label{thm:ell-p-switching-regret}
Define $C=\sqrt{p-1}\cdot 2^{-\frac{2}{p-1}}$. With parameters  $\gamma = 4C\sqrt{\frac{dS}{T}}$, $\eta=C\sqrt{\frac{S}{dT}}$, $\metaLR=\min\big\{\sqrt{\frac{S}{dT}},\frac{1}{16d},\frac{C^2}{2}\big\}$, $\mu = \frac{1}{T}$, $\beta=8d\metaLR$, and $\Correct = \frac{C}{\sqrt{dST}}$, \pref{alg:ell-p-ball} guarantees
\begin{align*}
    \E[\RegS] = \E\left[ \sum_{t=1}^T \ell_t^\T x_t - \sum_{t=1}^T \ell_t^\T u_{t}\right] =\otil\left(\sqrt{dST}\right),
\end{align*}
where $u_1,\ldots,u_T\in \calX$ are arbitrary comparators such that $\sum_{t=2}^T\indi\{u_{t-1}\ne u_t\}\leq S-1$.
\end{theorem}

We point out again that this is the first optimal switching regret guarantee for linear bandits over $\ell_p$ balls with $p\in(1,2]$, demonstrating the importance of our new corralling method.

\paragraph{Extensions to smooth and strongly convex domain.}
Our ideas and results can be generalized to adversarial linear bandits over any smooth and strongly convex set, a setting studied in~\citep{arxiv'21:uniform-convex}.
Specifically, for a smooth and strongly convex set containing the $\ell_p$ unit ball and contained by the dual $\ell_q$ unit ball (for some $p \in (1,2]$), our algorithm achieves $\otil\big(d^{\nicefrac{1}{p}}\sqrt{ST}\big)$ switching regret.
We defer all details to \pref{appendix: strongly cvx set}.

\section{Extension to Unconstrained Linear Bandits}\label{sec: unconstrained}
In this section, we further extend our results on linear bandits to the \emph{unconstrained} setting, that is, $\X = \R^d$, which means
both the learner's decisions $\{x_t\}_{t=1}^T$ and the comparators $\{u_t\}_{t=1}^T$ can be chosen arbitrarily in $\R^d$. 
The loss vectors are assumed to have bounded $\ell_2$ norm: $\norm{\ell_t}_2 \leq 1$ for all $t \in [T]$. 
As mentioned, \citep{NIPS'20:Dirk} is the only existing work considering the same setting.
They study static regret and achieve a comparator-adaptive bound $\E[\Reg] = \otil(\|u\|_2\sqrt{dT})$ simultaneously for all $u$ (the fixed comparator).\footnote{The actual bound stated in~\citep{NIPS'20:Dirk} is actually $\otil(\|u\|_2 d\sqrt{T})$, but it is straightforward to see that it can be improved to $\otil(\|u\|_2\sqrt{dT})$ by picking the optimal linear bandit algorithm over $\ell_2$ balls in their reduction.
}
Building on our results in \pref{sec: lp_ball}, we generalize their static regret bound to switching regret and achieve a similar comparator-adaptive bound $\E[\Reg(u_1,\ldots,u_T)] = \otil\big(\max_{k\in [S]}\|\ucirc_k\|_2 \cdot \sqrt{dST}\big)$
simultaneously for all $u_1,\ldots,u_T$ with $S-1$ switches.

Instead of using our recipe and starting from scratch to solve this problem,
we directly make use of the reduction of~\citep{NIPS'20:Dirk} which reduces the unconstrained problem to the constrained counterpart (already solved by our \pref{alg:ell-p-ball}) plus another one-dimensional unconstrained problem; see \pref{sec:unconstrained-black-box}.
To solve the latter problem, in~\pref{sec:subroutine-OCO} we design a new unconstrained algorithm for general Online Convex Optimization (OCO) that enjoys a comparator-adaptive switching regret guarantee and might be of independent interest. 
Finally, we summarize the overall algorithm and provide the formal guarantees in~\pref{sec:unconstrained-overall}.

\subsection{Black-box reduction for switching regret of unconstrained linear bandits}
\label{sec:unconstrained-black-box}

The reduction of~\citep{NIPS'20:Dirk} takes heavy inspiration from~\citep{COLT'18:black-box-reduction}.
Specifically, suppose that we have two subroutines denoted by $\A_{\Z}$ and $\A_{\V}$: $\A_{\Z}$ is a constrained linear bandit algorithm over the $\ell_2$ ball $\Z =\{z \in \R^d \mid \|z\|_2\leq 1\}$ and 
$\A_{\V}$ is an unconstrained and one-dimensional online linear optimization algorithm with full-information feedback (in fact, in the one-dimensional linear case, there is no difference between full-information and bandit feedback).
Then, one can solve an unconstrained linear bandit problem as follows:
at each round $t \in [T]$, the learner makes the decision $x_t = v_t \cdot z_t$, where $z_t \in \Z$ is the direction returned by the constrained bandit algorithm $\A_{\Z}$, and $v_t \in \R$ is the scalar returned by the one-dimensional algorithm $\A_{\V}$. After observing the loss $\ell_t^\T x_t$, the learner then feeds $\ell_t^\T z_t = \frac{\ell_t^\T x_t}{v_t}$ to both $\A_{\Z}$ and $\A_{\V}$ so they can update themselves.
See \pref{alg:unconstrained-linear-bandits} (\pref{appendix:reduction-code}) for the pseudocode.

\citet{NIPS'20:Dirk} show that the static regret of such a reduction can be expressed using the regret of the two subroutines.
This can be directly generalized to switching regret, formally described below (see \pref{appendix:proof-unconstrained-reduction} for the proof).
\begin{lemma}
\label{lemma:unconstrained-decompose}
For an interval $\I \subseteq [T]$, let $\Reg^{\mathcal{V}}_{\I}(v) = \sum_{t \in \I} (v_t - v) \inner{z_t, \ell_t}$ be the regret of the unconstrained one-dimensional algorithm $\A_{\mathcal{V}}$ against a comparator $v \in \R$ on this interval, and similarly $\Reg^{\mathcal{Z}}_{\I}(z) = \sum_{t \in \I} \langle z_t - z, \ell_t\rangle$ be the regret of the constrained linear bandits algorithm $\A_{\mathcal{Z}}$ against a comparator $z \in \Z=\{z \in \R^d \mid \|z\|_2\leq 1\}$ on this interval. Then \pref{alg:unconstrained-linear-bandits} (with decision $x_t = z_t \cdot v_t$) satisfies
\begin{equation}
    \label{eq:unconstrained-decompose}
    \Reg(u_1,\ldots,u_T) = \sum_{k=1}^S \Reg_{\I_k}^{\mathcal{V}}(\norm{\ucirc_k}_2) + \sum_{k=1}^S \norm{\ucirc_k}_2 \cdot \Reg_{\I_k}^{\mathcal{Z}}\left( \frac{\ucirc_k}{\norm{\ucirc_k}_2} \right),
\end{equation}
where we recall that $\I_1,\ldots,\I_S$ denotes a partition of $[T]$ such that for each $k$, $u_t$ remains the same (denoted by $\ucirc_k$) for all $t \in \I_k$.
\end{lemma}

One can see that the first term in~\pref{eq:unconstrained-decompose} is clearly the switching regret of $\A_{\mathcal{V}}$,
while the second term, after upper bounded by $\max_{k\in[S]} \norm{\ucirc_k}_2 \sum_{k=1}^S  \Reg_{\I_k}^{\mathcal{Z}}\left( \frac{\ucirc_k}{\norm{\ucirc_k}_2} \right)$, is the the switching regret of $\A_{\mathcal{Z}}$ scaled by the maximum comparator norm.
Therefore, to control the second term, we simply apply our \pref{alg:ell-p-ball} as the subroutine $\A_{\Z}$, making it at most $\otil\big(\max_{k\in [S]}\|\ucirc_k\|_2 \cdot \sqrt{dST}\big)$.
On the other hand, to the best of our knowledge, there are no existing unconstrained algorithms with switching regret guarantees. 
To this end, we design one such algorithm in the next section.
In fact, for full generality, we do so for the more general unconstrained OCO problem of arbitrary dimension \emph{without the knowledge of $S$}, which might be of independent interest.

\subsection{Subroutine: switching regret of unconstrained online convex optimization}
\label{sec:subroutine-OCO}

As a slight detour, in this section we consider a general unconstrained OCO problem: at round $t \in [T]$, the learner makes a decision $v_t \in \R^d$ and simultaneously the adversary chooses a loss function $f_t: \R^d \mapsto \R$, then the algorithm suffers   loss $f_t(v_t)$ and observes the gradient $\nabla f_t(v_t)$ as feedback. Notably, the feasible domain is $\R^d$ (that is, no constraints). The goal of the learner is to minimize the switching regret 
\begin{equation}
    \label{eq:switching-regret-OCO}
    \Reg(u_1,\ldots,u_T) = \sum_{t=1}^T f_t(v_t) - \sum_{t=1}^T f_t(u_t) = \sum_{k=1}^S \sum_{t \in \I_k} \Big(f_t(v_t) - f_t(\ucirc_k)\Big),
\end{equation}
where the notations $\I_1, \ldots, \I_S$ and $\ucirc_1,\ldots,\ucirc_S \in \R^d$ are  defined similarly as in \pref{sec: pre}. 
Without loss of generality, it is assumed that $\max_{x}\|\nabla f_t(x)\|_2 \leq 1$ for all $t$.
Note that this setup is a strict generalization of what we need for the one-dimensional subroutine $\A_{\V}$ discussed in \pref{sec:unconstrained-black-box}.

Our idea is once again via a meta-base framework, which is in fact easier than our earlier discussions because now we have gradient feedback.
There are two quantities that we aim to adapt to: the number of switches $S$ and 
the comparator norm $\|\ucirc_k\|_2$  (although the latter can be unbounded, it suffices to consider a maximum norm of $2^T$ as~\citep[Appendix D.5]{COLT'21:impossible-tuning} shows).
Therefore, we create an exponential grid for these two quantities, and maintain one base algorithm for each possible configuration.
These base algorithms only need to satisfy some mild conditions specified in \pref{require:data-independent} of \pref{appendix:OCO-algorithm},
and many existing algorithms such as~\citep{ICML'15:Daniely-adaptive,AISTATS'17:coin-betting-adaptive,ICML'20:Ashok} indeed meet the requirements.

The design of the meta algorithm requires some care to ensure the desirable adaptive guarantees, and we achieve so by building upon the recent progress in the classic expert problem~\citep{COLT'21:impossible-tuning}.
In short, our meta algorithm is OMD with a multi-scale entropy regularizer and certain important correction terms. We defer the details to~\pref{appendix:OCO-algorithm}, including the pseudocode of the full algorithm in~\pref{alg:unconstrained-OCO}. Below we present the main comparator-adaptive switching regret guarantee of this algorithm.

\begin{theorem}
\label{thm:unconstrained-OCO}
\pref{alg:unconstrained-OCO} with a base algorithm satisfying~\pref{require:data-independent} guarantees that for any $S$, any partition $\I_1, \ldots, \I_S$ of $[T]$, and any comparator sequence $\ucirc_1,\ldots,\ucirc_S \in \R^d$, we have
\begin{equation*}
    \sum_{k=1}^S \left(\sum_{t \in\I_k} f_t(v_t) -  \sum_{t \in\I_k} f_t(\ucirc_k) \right) \leq \otil\left( \sum_{k=1}^S \norm{\ucirc_k}_2\sqrt{\abs{\I_k}}  \right) \leq \otil\left( \max_{k \in [S]} \norm{\ucirc_k}_2\cdot \sqrt{ST}\right).
\end{equation*}
\end{theorem}

We emphasize again that in contrast to our other results on bandit problems, the guarantee above is achieved for all $S$ simultaneously (in other words, the algorithm does not need the knowledge of $S$).
It also adapts to the norm of the comparator $\norm{\ucirc_k}_2$ on each interval $\I_k$, instead of only the maximum norm $\max_{k \in [S]} \norm{\ucirc_k}_2$.
As another remark,
if the base algorithms further guarantee a data-dependent regret (this is satisfied by for example the algorithm of~\citet{ICML'20:Ashok}), our switching regret guarantee can be further improved to $\Ot\Big( \sum_{k=1}^S \norm{\ucirc_k}_2\sqrt{\sum_{t \in \I_k} \norm{\nabla f_t(v_t)}_2^2}\Big) \leq \otil\Big( \max_{k \in [S]} \norm{\ucirc_k}_2\cdot \sqrt{S\sum_{t=1}^T \norm{\nabla f_t(v_t)}_2^2}\Big)$,
 replacing the dependence on $T$ by the cumulative gradient norm square. 
This results holds even if the algorithm is required to make decisions from a bounded domain, thus strictly improving the $\Ot\Big(D_{\max} \sqrt{S \sum_{t=1}^T \norm{\nabla f_t(v_t)}_2^2}\Big)$ result of~\citep{ICML'20:Ashok,NIPS'20:sword} where $D_{\max}$ is the diameter of the domain.
See \pref{appendix:data-dependent-OCO} for details.

\subsection{Summary: comparator-adaptive switching regret for unconstrained linear bandits}
\label{sec:unconstrained-overall}

Combining all previous discussions, we now present the final result on unconstrained linear bandits.
\begin{theorem}
\label{thm:unconstrained-linear-bandits}
Using~\pref{alg:ell-p-ball} (with $p =2$) as the subroutine $\A_{\Z}$ and~\pref{alg:unconstrained-OCO} as the subroutine $\A_\V$ in the black-box reduction~\pref{alg:unconstrained-linear-bandits}, the overall algorithm enjoys the following comparator-adaptive switching regret against any partition $\I_1, \ldots, \I_S$ of $[T]$ and any corresponding comparators $\ucirc_1,\ldots,\ucirc_S \in \R^d$:
\begin{align*}
    \E[\RegS] \leq \otil \left(\sum_{k=1}^S\|\ucirc_k\|_2\left(\sqrt{\frac{dT}{S}}+\sqrt{\frac{dS}{T}}\left|\calI_k\right|\right)\right) \leq \otil\left(\max_{k\in [S]}\|\ucirc_k\|_2 \cdot \sqrt{dST}\right).
\end{align*}
\end{theorem}

The proof can be found in~\pref{appendix:proof-unconstrained}. 
Again, this is the first switching regret for unconstrained linear bandits, and it strictly generalizes the static regret results of~\citep{NIPS'20:Dirk}.
Although we are not directly using our new corralling recipe to achieve this result, it clearly serves as an indispensable component for this result due to the usage of \pref{alg:ell-p-ball}.

\section{Conclusion and Discussions}
\label{sec: conclusion}
In this paper, we propose a new mechanism for combining a collection of bandit algorithms with regret overhead only logarithmically depending on the number of base algorithms. As a case study, 
we provide a set of new results on switching regret for adversarial linear bandits using this recipe.
One future direction is to extend our switching regret results to linear bandits with general domains or even to general convex bandits, which appears to require additional new ideas to execute our recipe. Another interesting direction is to find more applications for our corralling mechanism beyond obtaining switching regret, as we know that logarithmic dependence on the number of base algorithms is possible.

\acks{Peng Zhao and Zhi-Hua Zhou are supported by NSFC (61921006). HL and MZ are supported by NSF Award IIS-1943607. The authors thank Chen-Yu Wei for helpful discussions on the idea of negative bias injection in the meta algorithm design.}
\bibliography{ref}

\appendix
\section{Potential Approaches for Switching Regret of Linear Bandits}
\label{appendix:potential-methods}

As mentioned in the main paper, to the best of our knowledge, we are not aware of any paper with switching regret for adversarial linear bandits. In this section, we present two potential approaches to achieve switching regret for adversarial linear bandits with $\ell_p$-ball feasible domain, however, the regret bounds are suboptimal.

\paragraph{Method 1. Periodical Restart.} The first generic method for tackling the switching regret of linear bandits is by running a classic linear bandits algorithm with a periodical restart. Specifically, suppose we employ an algorithm $\A$ as the base algorithm and restart it for every $\Delta >0$ rounds. Then, the switching regret of the overall algorithm satisfies:
\begin{align}
    \E[\RegS] \leq S \cdot \Delta + \left(\frac{T}{\Delta} - S \right) \cdot \Reg(\A; \Delta) \leq \Ot\left(S \Delta + \frac{T}{\sqrt{\Delta}}\right) = \Ot\Big( S^{\frac{1}{3}}T^{\frac{2}{3}} \Big),
\end{align}
where the first inequality holds because there are at most $S$ periods that contains a shift of comparators and we bound the regret in those periods trivially by $S  \Delta$, and for the other periods the regret is controlled by the base algorithm $\A$. The second inequality is by chosen base algorithm $\A$ such that the regret is of order $\Ot(\sqrt{\Delta})$, which can be satisfied by for example \scrible~\citep{COLT'08:CompetingDark}. The last equality is by set the period optimally as $\Delta = \lceil (T/S)^{\frac{1}{3}} \rceil$. To summarize, the restarting algorithm applies to general adversarial linear bandits and attains a suboptimal switching regret of order $\Ot(S^{\frac{1}{3}} T^{\frac{2}{3}} )$, given the knowledge of $S$.

\paragraph{Method 2. \textsc{Exp2} with Fixed-share Update.} The second method is by using the \textsc{Exp2} algorithm~\citep{NIPS'07:bandit-lower-bound} with a uniform mixing update~\citep{journals/ml/HerbsterW98,SICOMP'02:Auer-EXP3}, which can give an $\Ot(d\sqrt{ST})$ switching regret for adversarial linear bandits with a general convex and compact domain. Note that the method is based on continuous exponential weights and thus requires log-concave sampling~\citep{journal'07:log-concave}, which is theoretically efficient but usually time-consuming in practice. More importantly, the dimensional dependence is linear and hence not optimal when the feasible domain is an $\ell_p$ ball, $p\in(1,2]$.

Beyond the above two methods, one may wonder whether we can simply use FTRL/OMD with some barrier regularizer (such as \scrible~\citep{COLT'08:CompetingDark}) along with either a uniform mixing  update~\citep{journals/ml/HerbsterW98,SICOMP'02:Auer-EXP3} or a clipped domain~\citep{JMLR'01:Herbster} to achieve switching regret for linear bandits. However, the attempt fails to work as the regularization term in the regret bound will become too large to control due to the property of barrier regularizer. Indeed, this method cannot even achieve switching regret guarantees for MAB due to the same reason.
\section{Omitted Details for~\pref{sec: lp_ball}}\label{appendix:lp_ball}
In this section, we provide the omitted details for~\pref{sec: lp_ball}, including the pseudocode of the base algorithm (in~\pref{appendix:lp_ball_base}) and the proof of~\pref{thm:ell-p-switching-regret} (in~\pref{appendix:unbias-loss-estimator} -- \ref{appendix:proof-main}). To prove~\pref{thm:ell-p-switching-regret}, we first prove the unbiasedness of loss estimators in~\pref{appendix:unbias-loss-estimator}, then decompose the regret in~\pref{appendix:regret-decompose}, and subsequently upper bound each term in~\pref{appendix:bound-part1},~\pref{appendix:base-reg}, and~\pref{appendix:metareg}. We finally put everything together and present the proof in~\pref{appendix:proof-main}. 

\subsection{Pseudocode of Base Algorithm}
\label{appendix:lp_ball_base}
\pref{alg:ell-p-ball-base} shows the pseudocode of the base algorithm for linear bandits with $\ell_p$ unit-ball feasible domain, which is the same as the one proposed in~\citep{ALT'18:hybrid-barrier}.
\begin{algorithm}[!h]
   \caption{Base algorithm for linear bandits on $\ell_p$ ball}
   \label{alg:ell-p-ball-base}
   \textbf{Input:} learning rate $\baseLR$, clipping parameter $\gamma$, initial round $t_0$.
    
   \textbf{Define:} clipped feasible domain $\calX'=\{x \mid \|x\|_p\leq 1-\gamma\}$.

   \textbf{Initialize:} $a_{t_0}^{(t_0)}=\argmin_{x\in\calX'}R(x)$ and $\xi_{t_0}^{(t_0)}=0$.

   Draw $\wt{a}_{t_0}^{(t_0)}$ uniformly randomly from $\{\pm e_n\}_{n=1}^d$.
       
   \For{$t=t_0$ {\bfseries to} $T$}{
   
      Send $(\wt{a}_{t}^{(t_0)},a_{t}^{(t_0)},\xi_{t}^{(t_0)})$ to the meta algorithm.
      
      Receive a loss estimator $\ellhat_t$.
      
      Update the strategy based on OMD with regularizer $R(x)=-\log (1-\|x\|_p^p)$:
      \begin{align}
        \label{eq:base-OMD-update}
          a_{t+1}^{(t_0)} = \argmin_{a\in \calX'}\left\{ \inner{a,\ellhat_t}+\frac{1}{\baseLR}D_R(a,a_t^{(t_0)})\right\}.
      \end{align}
      
      Generate a random variable $\xi_{t+1}^{(t_0)}\sim \ber(\|a_{t+1}^{(t_0)}\|_p)$ and set
      \begin{align*}
         \wt{a}_{t+1}^{(t_0)} =
          \begin{cases}
             \nicefrac{a_{t+1}^{(t_0)}}{\|a_{t+1}^{(t_0)}\|_p} & \mbox{if $\xi_{t+1}^{(t_0)}=1$}, \\
                \delta e_n & \mbox{if $\xi_{t+1}^{(t_0)}=0$},
          \end{cases}
      \end{align*}
      where $n$ is uniformly chosen from $\{1,\ldots,d\}$ and $\delta$ is a uniform random variable over $\{-1,+1\}$.
      }
\end{algorithm}

\subsection{Unbiasedness of Loss Estimators}
\label{appendix:unbias-loss-estimator}
The following lemma shows the unbiasedness of the constructed loss estimators for both  meta and base algorithms.
\begin{lemma}\label{lem:unbiasedness}
The meta loss estimator $\bar{\ell}_t$ defined in~\pref{eq:meta-estimator} and the base loss estimator $\ellhat_t$ defined in~\pref{eq:base-estimator} satisfy that $\mathbb{E}_t[\bar{\ell}_t]=\ell_t$ and $\mathbb{E}_t[\ellhat_t]=\ell_t$ for all $t\in [T]$.
\end{lemma}
\begin{proof}
We first show the unbiasedness of the meta loss estimator $\ellbar_t$. According to the definition in~\pref{eq:meta-estimator}, we have
\begin{align}\label{eqn:proof-meta-unbiased}
    \mathbb{E}_t[\bar{\ell}_t]&=\mathbb{E}_t[\wt{M}_t^{-1}x_tx_t^\top \ell_t] \nonumber\\
    &= \left(\frac{\beta}{d}\sum_{n=1}^de_ne_n^\top+(1-\beta)\sum_{i=1}^t\sleep_{t,i}\wt{a}_t^{(i)}\wt{a}_t^{(i)^\top}\right)^{-1}\cdot \mathbb{E}_t[x_tx_t^\top]\ell_t \nonumber\\
    &= \left(\frac{\beta}{d}\sum_{n=1}^de_ne_n^\top+(1-\beta)\sum_{i=1}^t\sleep_{t,i}\wt{a}_t^{(i)}\wt{a}_t^{(i)^\top}\right)^{-1}\cdot \left(\frac{\beta}{d}\sum_{n=1}^de_ne_n^\top+(1-\beta)\sum_{i=1}^t\sleep_{t,i}\wt{a}_t^{(i)}\wt{a}_t^{(i)^\top}\right)\ell_t \nonumber\\
    &=\ell_t.
\end{align}
Next, we show the unbiasedness of the base loss estimator $\ellhat_t$. According to the definition in~\pref{eq:base-estimator}, we have
\begin{align*}
    \mathbb{E}_t[\ellhat_t] &= \mathbb{E}_t\left[\frac{1-\xi_t}{1-\beta}\cdot \frac{d}{1-\sum_{i=1}^t\sleep_{t,i}\|a_{t}^{(i)}\|_p}\cdot (\ell_t^\T x_t)\cdot x_t \cdot\indi\{\rho_t=0\}\right] \\
    &= \mathbb{E}_t\left[\sum_{j=1}^t\sleep_{t,j}\cdot \frac{d(1-\xi_t^{(j)})}{1-\sum_{i=1}^t\sleep_{t,i}\|a_{t}^{(i)}\|_p}\cdot (\ell_t^\top \wt{a}_t^{(j)})\cdot\wt{a}_t^{(j)} \right]\\
    &= \sum_{j=1}^t\sleep_{t,j}(1-\|a_t^{(j)}\|_p)\cdot \frac{d}{1-\sum_{i=1}^t\sleep_{t,i}\|a_{t}^{(i)}\|_p}\cdot \frac{1}{d}\sum_{n=1}^d e_ne_n^{\top} \ell_t  \\
    &= \ell_t.
\end{align*}
In above derivations, the first step simply substitutes the definition of loss estimator, the second step holds due to the sampling scheme of~\pref{alg:ell-p-ball} (see~\pref{line:make prediction}), and the third step is because of the sampling mechanism in base algorithm (see~\pref{alg:ell-p-ball-base}). This finishes the proof.
\end{proof}

\subsection{Regret Decomposition}
\label{appendix:regret-decompose}
We introduce shifted comparators $u_t'=(1-\gamma)u_t$ and $\ucirc_k'=(1-\gamma)\ucirc_k$ to ensure that $u_t'\in\calX'$ for $t\in [T]$ and $\ucirc_k'\in\calX'$ for $k\in[S]$, where $\calX'=\{x \mid  \|x\|_p\leq 1-\gamma\}$. Based on the unbiasedness of $\ellhat_t$ and $\ellbar_t$, the expected regret can be decomposed as
\begin{align*}
    &\E[\RegS] \\
    & = \mathbb{E}\left[\sum_{t=1}^T\inner{x_t, \ell_t}-\sum_{t=1}^T\inner{u_t, \ell_t}\right] \\
    & =\mathbb{E}\left[\sum_{t=1}^T\inner{x_t, \ell_t}\right]-\mathbb{E}\left[\sum_{t=1}^T\inner{u_t', \ellhat_t}\right] + \mathbb{E}\left[\sum_{t=1}^T\inner{u_t'-u_t,\ell_t}\right] \\
    & =(1-\beta)\mathbb{E}\left[\sum_{t=1}^T\sum_{i=1}^t\sleep_{t,i}\inner{\wt{a}_t^{(i)}, \ell_t}\right]-\mathbb{E}\left[\sum_{t=1}^T\inner{u_t', \ellhat_t}\right] + \mathbb{E}\left[\sum_{t=1}^T\inner{u_t'-u_t,\ell_t}\right]\\
    & =\mathbb{E}\left[\sum_{t=1}^T\sum_{i=1}^t\sleep_{t,i}\inner{\wt{a}_t^{(i)}, \ellbar_t}\right]-\mathbb{E}\left[\sum_{t=1}^T\inner{u_t', \ellhat_t}\right] + \mathbb{E}\left[\sum_{t=1}^T\inner{u_t'-u_t,\ell_t} - \beta\sum_{t=1}^T\sum_{i=1}^t\sleep_{t,i}\inner{\wt{a}_t^{(i)}, \ell_t}\right] \\
    & =\mathbb{E}\left[\sum_{t=1}^T\sum_{i=1}^t\sleep_{t,i}\cost_{t,i}\right]-\mathbb{E}\left[\sum_{t=1}^T\inner{u_t', \ellhat_t}\right] + \mathbb{E}\left[\sum_{t=1}^T\inner{u_t'-u_t,\ell_t} - \beta\sum_{t=1}^T\sum_{i=1}^t\sleep_{t,i}\inner{\wt{a}_t^{(i)}, \ell_t}\right],
\end{align*}
where the third equation holds because of the sampling scheme of $x_t$: with probability $\beta$, the action $x_t$ is uniformly sampled from $\{ \pm e_n \}$, $n\in[d]$; with probability $1-\beta$, the action is sampled from $\{( \tilde{a}_t^{(i)}, \xi_t^{(i)})\}_{i=1}^t$ according to $\sleep_{t}$. In the last step, we recall that the notation $\cost_t\in \mathbb{R}^t$ is defined by $\cost_{t,i}=\innersmall{\wt{a}_t^{(i)}, \ellbar_t}$ for all $i\in [t]$. 

We further decompose the above regret into several intervals. To this end, we split the horizon to a partition $\I_1,\ldots, \I_S$. Let $j_k$ be the start time stamp of $\calI_k$. Note again that we use $\ucirc_{k}\in\calX$ to denote the comparator in $\calI_k$ for $k\in [S]$, which means that $u_t = \ucirc_{k}$ for all $t \in \I_k$. Then we have
\begin{align}
    & \E[\RegS] \nonumber \\
    & \leq \mathbb{E}\left[\sum_{k=1}^S \sum_{t \in \I_k} \left(\sum_{i=1}^t\sleep_{t,i}\cost_{t,i} - \inner{\ucirc_{k}',\hat{\ell}_t}\right)\right] + \mathbb{E}\left[\sum_{t=1}^T\inner{u_t'-u_t,\ell_t} - \beta\sum_{t=1}^T\sum_{i=1}^t\sleep_{t,i}\inner{\wt{a}_t^{(i)}, \ell_t}\right] \nonumber\\
    & = \mathbb{E}\left[\sum_{k=1}^S \sum_{t \in \I_k} \inner{\sleep_{t}-e_{j_k}, \cost_{t}} + \sum_{k=1}^S \sum_{t \in \I_k} \left(\inner{e_{j_k},\cost_t} - \inner{\ucirc_k',\hat{\ell}_t}\right)\right]\nonumber\\
    &\qquad +\mathbb{E}\left[\sum_{t=1}^T\inner{u_t'-u_t,\ell_t} - \beta\sum_{t=1}^T\sum_{i=1}^t\sleep_{t,i}\inner{\wt{a}_t^{(i)}, \ell_t}\right] \nonumber\\
    & = \mathbb{E}\left[\sum_{k=1}^S \sum_{t \in \I_k}\inner{\sleep_{t}-e_{j_k}, \cost_{t}-\injbias_t} + \sum_{k=1}^S \sum_{t \in \I_k} \left(\inner{e_{j_k},\cost_t} - \inner{\ucirc_k',\hat{\ell}_t}\right)\right]\nonumber\\
    &\qquad +\E\left[\sum_{k=1}^S\sum_{t\in\calI_k}\inner{\sleep_{t}-e_{j_k}, \injbias_t}\right] +\mathbb{E}\left[\sum_{t=1}^T\inner{u_t'-u_t,\ell_t} - \beta\sum_{t=1}^T\sum_{i=1}^t\sleep_{t,i}\inner{\wt{a}_t^{(i)}, \ell_t}\right]\nonumber\\
    & = \mathbb{E}\left[\sum_{k=1}^S \sum_{t \in \I_k} \left(\sum_{i\in[t]}\sleep_{t,i}\costhat_{t,i}-\costhat_{t,j_k}\right) + \sum_{k=1}^S \sum_{t \in \I_k} \left(\inner{e_{j_k},\cost_t} - \inner{\ucirc_k',\hat{\ell}_t}\right)\right]\tag{$\costhat_{t,i} = \cost_{t,i} - \injbias_{t,i}$ for $i \in [t]$}\\
    &\qquad +\mathbb{E}\left[\sum_{k=1}^S\sum_{t\in\calI_k}\sum_{i=1}^t(\sleep_{t,i}-e_{j_k})\injbias_{t,i}\right] +\mathbb{E}\left[\sum_{t=1}^T\inner{u_t'-u_t,\ell_t} - \beta\sum_{t=1}^T\sum_{i=1}^t\sleep_{t,i}\inner{\wt{a}_t^{(i)}, \ell_t}\right]\nonumber\\
    & = \mathbb{E}\left[\sum_{k=1}^S \sum_{t \in \I_k} \sum_{i\in[t]}\inner{\combine_t-e_{j_k}, \costhat_{t}} + \sum_{k=1}^S \sum_{t \in \I_k} \left(\inner{\wt{a}_t^{(j_k)},\ellbar_t} - \inner{\ucirc_k',\hat{\ell}_t}\right)\right]\nonumber\\
    &\qquad +\mathbb{E}\left[\sum_{k=1}^S\sum_{t\in\calI_k}\sum_{i=1}^t(\sleep_{t,i}-e_{j_k})\injbias_{t,i}\right] +\mathbb{E}\left[\sum_{t=1}^T\inner{u_t'-u_t,\ell_t} - \beta\sum_{t=1}^T\sum_{i=1}^t\sleep_{t,i}\inner{\wt{a}_t^{(i)}, \ell_t}\right] \nonumber\\
    & = \mathbb{E}\Bigg[\underbrace{\sum_{k=1}^S \sum_{t \in \I_k} \inner{\combine_{t} - e_{j_k},\costhat_t}}_{\metareg} + \underbrace{\sum_{k=1}^S \sum_{t \in \I_k} \inner{{a}_t^{(j_k)}-\ucirc_k',\ellhat_t}}_{\basereg} + \underbrace{\sum_{t=1}^T \sum_{i=1}^t \sleep_{t,i}\injbias_{t,i}}_{\posterm}  \nonumber\\
    & \qquad \qquad \qquad -\underbrace{\sum_{k=1}^S \sum_{t \in \I_k} \injbias_{t,j_k}}_{\negterm} + \underbrace{\sum_{t=1}^T\inner{u_t'-u_t,\ell_t} - \beta\sum_{t=1}^T\sum_{i=1}^t\sleep_{t,i}\inner{\wt{a}_t^{(i)}, \ell_t}}_{\bias}\Bigg],\label{eq:decompose-ell-p}
\end{align}
where the second-last equality is due to the constructions of $\sleep_{t}$ and $\costhat_t$ (see~\pref{line:estimator meta} in~\pref{alg:ell-p-ball}),
\begin{align*}
    \inner{\combine_t,\costhat_t} & = \sum_{i\in[t]}\combine_{t,i}\costhat_{t,i}+\sum_{i>t}\combine_{t,i}\costhat_{t,i} =\sum_{i\in[t]}\sleep_{t,i}\left(\sum_{j\in[t]}\combine_{t,j}\right)\costhat_{t,i}+\sum_{i>t}\combine_{t,i}\sum_{j\in[t]}\sleep_{t,j}\costhat_{t,j}\\
    &=\sum_{i\in[t]}\sleep_{t,i}\left(\sum_{j\in[t]}\combine_{t,j}+\sum_{i>t}\combine_{t,i}\right)\costhat_{t,i} = \sum_{i\in[t]}\sleep_{t,i}\costhat_{t,i},
\end{align*}
and the last equality is based on the definition of $\wt{a}_{t}^{(i)}$ and $a_t^{(i)}$ and the following equation:
\begin{align*}
    & \E\left[\inner{\tilde{a}_t^{(i)},\bar{\ell}_t} - \inner{u,\hat{\ell}_t}\right] = \E\left[\inner{\tilde{a}_t^{(i)},\E_{\sleep_{t}}[\bar{\ell}_t]} - \inner{u,\hat{\ell}_t}\right] = \E\left[\inner{\tilde{a}_t^{(i)},\ell_t} - \inner{u,\hat{\ell}_t}\right]\\
    & = \E\left[\inner{\E_{\xi_t^{(i)}}[\tilde{a}_t^{(i)}],\ell_t} - \inner{u,\hat{\ell}_t}\right] = \E\left[\inner{a_t^{(i)},\E[\hat{\ell}_t]} - \inner{u,\hat{\ell}_t}\right] = \E\left[\inner{a_t^{(i)} -u,\hat{\ell}_t}\right].
\end{align*}

As a consequence, we upper bound the expected switching regret by five terms as shown in~\pref{eq:decompose-ell-p}, including: \metareg, \basereg, \posterm, \negterm, and \bias. In the following, we will bound each term respectively.

\subsection{Bounding \bias~and \posterm}
\label{appendix:bound-part1}
\paragraph{\bias.} \bias~can be simply bounded by $(\beta+\gamma) T$ as
\begin{align}\label{eqn:bias-term-bound}
    &\sum_{t=1}^T\inner{u_t'-u_t, \ell_t} - \beta\sum_{t=1}^T\sum_{i=1}^t\sleep_{t,i}\inner{\wt{a}_t^{(i)}, \ell_t}\nonumber\\
    &\leq \sum_{t=1}^T\left((1-\gamma)-1\right)\inner{u_t, \ell_t}  + \beta T\nonumber\\
    &\leq \sum_{t=1}^T\left(1-(1-\gamma)\right)+\beta T\nonumber \\
    &\leq (\beta+\gamma) T.
\end{align}
where the first and second inequalities hold  because we have $\abs{\ell_t^\T x} \leq 1$ for any $x \in \X$ and $t\in[T]$.

\paragraph{\posterm.} According to the definition of $\injbias_{t,i}$, we show that \posterm~is at most 
\begin{align}\label{eqn:positive-term-bound}
    \frac{1}{\Correct T(1-\beta)} \sum_{t=1}^T \sum_{i=1}^t \frac{\sleep_{t,i} (1 - \norm{a_t^{(i)}}_p)}{1 - \sum_{j=1}^t \sleep_{t,j} \norm{a_t^{(j)}}_p} = \frac{1}{\Correct(1-\beta)}\leq \frac{2}{\Correct}.
\end{align}

Hence, it remains to evaluate \basereg~and~\metareg, and in the following two subsections we present their upper bounds, respectively.

\subsection{Bounding \basereg}
\label{appendix:base-reg}
In order to bound \basereg, we need to introduce the following  lemma proven in~\citep{ALT'18:hybrid-barrier}, which shows that the dual local norm with respect to the regularizer $R(x)=-\log(1-\|x\|_p^p)$ is well bounded. This will later be shown to be crucial in controlling the stability of $a_t$ updated by the online mirror descent shown in~\pref{alg:ell-p-ball-base}.
\begin{lemma}[Lemma 2 in \citep{ALT'18:hybrid-barrier}]
\label{lem:stab lemma}
Let $x$, $\ell\in \mathbb{R}^d$ such that $\|x\|_p < 1$, $\|\ell\|_0=1$ and $\|\ell\|_2\leq 1$. Let $y\in \mathbb{R}^d$ such that $\nabla R(y)\in [\nabla R(x), \nabla R(x)+\ell]$, $R(x)=-\log(1-\|x\|_p^p)$. Then, we have for $p\in (1,2]$,
\begin{align*}
    \|\ell\|_{y,*}^2\leq \frac{2^{\frac{3}{p-1}}(1-\|x\|_p^p)}{p(p-1)}\sum_{n=1}^d\left(|x_n|^{2-p}+|\ell_n|^{\frac{2-p}{p-1}}\right)\ell_n^2.
\end{align*}
In above, for a vector $h \in \R^d$, $\norm{h}_0 \triangleq \#\{n \mid h_n \neq 0\}$ denotes the number of non-zero entries, $\norm{h}_{x} \triangleq \sqrt{h^\T \nabla^2 R(x) h}$ denotes the local norm induced by $R$ at $x$, and $\norm{h}_{x,*} \triangleq \sqrt{h^\T (\nabla^2 R(x))^{-1} h}$ denotes the dual local norm. 
\end{lemma}

Then we are ready to bound~\basereg~for each $k\in[S]$. Note that for each $k\in[S]$, as $j_k$ is the start time stamp of interval $\calI_k$, and base algorithm $\calB_{t}$ starts at round $t$, we know that $\sum_{t\in\calI_k}\innersmall{a_t^{(j_k)}-\ucirc_k, \ellhat_t}$ is in fact the (estimated) static regret against comparator $\ucirc_k$ for $\calB_{j_k}$.

\begin{lemma}
\label{lem:base-reg}
For an arbitrary interval $\calI$ started at round $j$, if $\gamma = 4d\baseLR_{j'}$ for all $j'\in[T]$, \pref{alg:ell-p-ball} ensures that the base regret of $\calB_j$ with learning rate $\baseLR $ for any comparator $u\in\calX'$ is at most
\begin{align}\label{eq:base-reg-bound}
    \E\left[\sum_{t \in \I} \inner{a_t^{(j)} - u,\hat{\ell_t}}\right] \leq \frac{\log(1/\gamma)}{\baseLR } + \frac{2^{\frac{4}{p-1}}d\baseLR }{(p-1)(1-\beta)}  \sum_{t \in \I} \frac{1 - \norm{a_t^{(j)}}_p}{1-\sum_{i=1}^{t} \sleep_{t,i}\norm{a_t^{(i)}}_p}.
\end{align}
\end{lemma}
\begin{proof}
Since the base algorithm $\B_j$ performs the online mirror descent over loss $\hat{\ell}_t$ with learning rate $\baseLR $, see update in~\pref{eq:base-OMD-update}, according to the standard analysis of OMD (see~\pref{lemma:OMD-bubeck}) we have
\begin{align*}
    \E\left[\sum_{t\in\calI} \inner{a_t^{(j)} - u,\hat{\ell_t}}\right] \leq \frac{R(u)-R(a_j^{(j)})}{\baseLR } + \frac{1}{\baseLR } \sum_{t\in \calI} \E\left[ D_{R^*}\left(\nabla R(a_t^{(j)}) - \baseLR  \hat{\ell}_t, \nabla R(a_t^{(j)})\right)\right].
\end{align*}
Consider the first term. As $a_j^{(j)}=\argmin_{x\in\calX}R(x)$ and $u\in\calX'=\{x \mid \|x\|_p\leq 1-\gamma\}$, we have
\begin{align}
    \label{eq:first-desire}
    R(u)-R(a_j^{(j)}) \leq -\log(1-(1-\gamma)) \leq -\log\gamma.
\end{align}
For the second term, in the following we will employ~\pref{lem:stab lemma} to show that 
\begin{align}
    \label{eq:second-desire}
    &\mathbb{E}_t\left[D_{R^*}\left(\nabla R(a_t^{(j)})-\baseLR \ellhat_t, \nabla R(a_t^{(j)})\right)\right] \leq \baseLR ^2\cdot \frac{d\cdot2^{\frac{4}{p-1}}}{(p-1)(1-\beta)}\cdot\frac{1-\|a_t^{(j)}\|_p}{1-\sum_{i=1}^t\sleep_{t,i}\|a_t^{(i)}\|_p}.
\end{align}
To this end, we need to verify the condition of~\pref{lem:stab lemma}. In fact, according to the definition of the base loss estimator in~\pref{eq:base-estimator}, $\ellhat_t$ is a non-zero vector only when~\pref{alg:ell-p-ball} samples from one of the base algorithm instances and $\xi_t=0$, meaning that $x_t=\pm e_n$ for some $n\in [d]$ according to~\pref{alg:ell-p-ball-base}.  Using the fact that $a_{t}^{(i)}\in\calX'$ and $\beta \leq \frac{1}{2}$, we have $\norm{a_{t}^{(i)}}_p \leq 1 - \gamma$ and 
\begin{align*}
    \|\baseLR \ellhat_t\|_2 \leq\frac{\baseLR d}{(1-\beta)(1-\sum_{i=1}^t\sleep_{t,i}(1-\gamma))}\leq \frac{\baseLR d}{\gamma(1-\beta)}\leq \frac{2\baseLR d}{\gamma}\leq \frac{1}{2},
\end{align*}
where the last inequality is because of the choice of $\gamma=4d\baseLR $. In addition, based on the definition of $\ellhat_t$, we have $\|\baseLR \ellhat_t\|_0=1$. Therefore, we can apply \pref{lem:stab lemma} and obtain that
\begin{align*}
    &\mathbb{E}_t\left[D_{R^*}\left(\nabla R(a_t^{(j)})-\baseLR \ellhat_t, \nabla R(a_t^{(j)})\right)\right] \\
    &= \mathbb{E}_t\left[\|\baseLR \ellhat_t\|_{y_t,*}^2\right] \\
    &\leq \baseLR ^2\cdot \frac{2^{\frac{3}{p-1}}(1-\|a_t^{(j)}\|_p^p)}{p(p-1)}\sum_{n=1}^d\mathbb{E}_t\left[\left(|a_{t,n}^{(j)}|^{2-p}+|\baseLR \ellhat_{t,n}|^{\frac{2-p}{p-1}}\right)\ellhat_{t,n}^2\right]\\
    & = \baseLR ^2\cdot \frac{2^{\frac{3}{p-1}}(1-\|a_t^{(j)}\|_p^p) }{p(p-1)} \Bigg( \underbrace{ \sum_{n=1}^d\mathbb{E}_t\left[|a_{t,n}^{(j)}|^{2-p} \cdot \ellhat_{t,n}^2\right]}_{\term{a}} + \underbrace{ \sum_{n=1}^d\mathbb{E}_t\left[|\baseLR \ellhat_{t,n}|^{\frac{2-p}{p-1}} \cdot \ellhat_{t,n}^2\right]}_{\term{b}}\Bigg),
\end{align*}
where the first equality holds for some $y_t\in [\nabla R(a_t^{(j)})-\baseLR \ellhat_t, \nabla R(a_t^{(j)})]$ by the definition of Bregman divergence and the mean value theorem, the second inequality is by~\pref{lem:stab lemma}. The last equality splits the desired quantity into two terms, and we upper bound \term{a} and \term{b} respectively.

For \term{a}, substituting the definition of loss estimator $\ellhat_t$ (see definition in~\pref{eq:base-estimator}) yields
\begin{align*}
    &\sum_{n=1}^d\mathbb{E}_t\left[|a_{t,n}^{(j)}|^{2-p}\cdot\ellhat_{t,n}^2\right] \\
    &=\frac{d^2}{(1-\beta)(1-\sum_{i=1}^t\sleep_{t,i}\|a_t^{(i)}\|_p)^2} \sum_{n=1}^d|a_{t,n}^{(j)}|^{2-p}\cdot\sum_{\tau=1}^t\sleep_{t,\tau}\mathbb{E}_t\left[(1-\xi_t^{(\tau)})^2\wt{a}_{t,n}^{(\tau)^2}\innersmall{\wt{a}_{t}^{(\tau)}, \ell_t}^2\right]\\
    &=\frac{d^2)}{(1-\beta)(1-\sum_{i=1}^t\sleep_{t,i}\|a_t^{(i)}\|_p)^2} \sum_{n=1}^d|a_{t,n}^{(j)}|^{2-p}\cdot\sum_{\tau=1}^t\sleep_{t,\tau}(1-\|a_t^{(\tau)}\|_p)\cdot\frac{1}{d}\sum_{n'=1}^d\left[\indi\{n'=n\}\ell_{t,n'}^2\right]\\
    &=\frac{d}{(1-\beta)(1-\sum_{i=1}^t\sleep_{t,i}\|a_t^{(i)}\|_p)} \sum_{n=1}^d|a_{t,n}^{(j)}|^{2-p}\ell_{t,n}^2\\
    &\leq \frac{d}{(1-\beta)(1-\sum_{i=1}^t\sleep_{t,i}\|a_t^{(i)}\|_p)},
\end{align*}
where the last inequality is because of \holder's inequality, $\|\ell_t\|_q\leq1$ and $\|a_{t}^{(j)}\|_p\leq1$.

For \term{b}, again by definition of the loss estimator, we have
\begin{align*}
    & \sum_{n=1}^d\mathbb{E}_t\left[|\baseLR \ellhat_{t,n}|^{\frac{2-p}{p-1}}\cdot\ellhat_{t,n}^2\right] \\
    &\leq \sum_{n=1}^d\mathbb{E}_t\left[\left|\frac{\baseLR (1-\xi_t)d(x_tx_t^\top\ell_t)_n}{(1-\beta)\gamma}\right|^{\frac{2-p}{p-1}}\cdot\ellhat_{t,n}^2\right]\\
    &= \sum_{n=1}^d\mathbb{E}_t\left[\left|\frac{\baseLR (1-\xi_t)d(x_tx_t^\top\ell_t)_n}{(1-\beta)\gamma}\right|^{\frac{2-p}{p-1}}\cdot \frac{(1-\xi_t)^2d^2(x_tx_t^\top\ell_t)_n^2\cdot\mathbbm{1}\{\rho_t=0\}}{(1-\beta)^2(1-\sum_{i=1}^t\sleep_{t,i}\|a_{t}^{(i)}\|_p)^2} \right] \\\
    &\leq \sum_{n=1}^d\mathbb{E}_t\left[\left|\frac{(x_tx_t^\top\ell_t)_n}{2}\right|^{\frac{2-p}{p-1}}\cdot \frac{(1-\xi_t)^2d^2(x_tx_t^\top\ell_t)_n^2\cdot\mathbbm{1}\{\rho_t=0\}}{(1-\beta)^2(1-\sum_{i=1}^t\sleep_{t,i}\|a_{t}^{(i)}\|_p)^2} \right] \tag{$\gamma=4d\baseLR $, $1-\beta\geq \frac{1}{2}$}\\
    &\leq  \frac{d^2}{(1-\beta)^2(1-\sum_{i=1}^t\sleep_{t,i}\|a_t^{(i)}\|_p)^2}\sum_{n=1}^d\mathbb{E}_t\left[(1-\xi_t)^2(x_tx_t^\top\ell_t)_n^q\cdot\mathbbm{1}\{\rho_t=0\} \right] \tag{note that $\frac{2-p}{p-1} + 2 = q$} \\
    &\leq \frac{1}{(1-\beta)^2}\cdot \frac{d^2}{(1-\sum_{i=1}^t\sleep_{t,i}\|a_t^{(i)}\|_p)^2}\cdot\sum_{n=1}^d (1-\beta)\sum_{\tau=1}^t\sleep_{t,\tau}(1-\|a_t^{(\tau)}\|_p)\cdot\frac{1}{d}\cdot\ell_{t,n}^q \\
    &\leq\frac{d}{(1-\beta)(1-\sum_{i=1}^t\sleep_{t,i}\|a_t^{(i)}\|_p)}.
\end{align*}

Combining the above upper bounds for \term{a} and \term{b}, we obtain
\begin{align*}
    \mathbb{E}_t\left[D_{R^*}\left(\nabla R(a_t^{(j)})-\baseLR \ellhat_t, \nabla R(a_t^{(j)})\right)\right] &\leq \frac{\baseLR ^2}{1-\beta}\cdot \frac{2d\cdot2^{\frac{3}{p-1}}}{p(p-1)}\cdot\frac{1-\|a_t^{(j)}\|_p^p}{1-\sum_{i=1}^t\sleep_{t,i}\|a_t^{(i)}\|_p}\\
    &\leq \frac{\baseLR ^2}{1-\beta}\cdot \frac{d\cdot2^{\frac{4}{p-1}}}{p(p-1)}\cdot\frac{1-\|a_t^{(j)}\|_p^p}{1-\sum_{i=1}^t\sleep_{t,i}\|a_t^{(i)}\|_p}\\
    &\leq \frac{\baseLR ^2}{1-\beta}\cdot \frac{d\cdot2^{\frac{4}{p-1}}}{p-1}\cdot\frac{1-\|a_t^{(j)}\|_p}{1-\sum_{i=1}^t\sleep_{t,i}\|a_t^{(i)}\|_p}.
\end{align*}
Note that the last step is true because we have $1-\|x\|_p^p \leq p(1-\|x\|_p)$ by the following inequality
\[
    1 + p \cdot \frac{\norm{x}_p^p -1}{p} \leq \left(1 + \frac{\norm{x}_p^p - 1}{p}\right)^p,
\]
which holds due to $p\in(1,2]$ and $0\leq \|x\|_p\leq 1$ as well as the Bernoulli's inequality that $1 + r \theta \leq (1+\theta)^r$ for any $r \geq 1$ and $\theta \geq -1$.

Therefore, we finish proving the desired upper bound in~\pref{eq:second-desire}. Further combining it with the upper bound in~\pref{eq:first-desire} finishes the proof of~\pref{lem:base-reg}.
\end{proof}
We will show later that the second term in the bound shown in~\pref{eq:base-reg-bound} can in fact be cancelled by the~\negterm. Finally, we bound the term~\metareg.

\subsection{Bounding \metareg}
\label{appendix:metareg}
We prove the following lemma to bound the \metareg.
\begin{lemma}
\label{lemma:switching-regret-meta}
For an arbitrary interval $\calI\subseteq[T]$ started at round $j$, setting $\frac{\metaLR}{\Correct\gamma T}\leq \frac{1}{8}$, $\beta=8d\metaLR$ and $\mu = \frac{1}{T}$, \pref{alg:ell-p-ball} guarantees that
\begin{equation}
    \label{eq:switching-regret-meta}
    \sum_{t\in \calI} \inner{\combine_{t}-e_{j}, \costhat_t} \leq \frac{2\ln T}{\metaLR} + \metaLR \sum_{t\in\calI} \sum_{i=1}^T \combine_{t,i} \costhat_{t,i}^2+ \O\left(\frac{\abs{\I}}{\metaLR T}\right).
\end{equation}
\end{lemma}

\begin{proof} 
Note that the meta algorithm essentially performs the exponential weights with a fixed-share update and sleeping expert. 
Define $v_{t+1,i} \triangleq \frac{\combine_{t,i} \exp(-\metaLR \costhat_{t,i})}{\sum_{t=1}^T \combine_{t,i} \exp(-\metaLR \costhat_{t,i})}$ for all $i\in [T]$. Then $\combine_{t+1,i} = \frac{\mu}{T} + (1-\mu) v_{t+1,i}$. Note that
\begin{align*}
    & \inner{\combine_t,\costhat_t} + \frac{1}{\metaLR}\ln\left( \sum_{i=1}^T \combine_{t,i} \exp(-\metaLR \costhat_{t,i})\right) \\
    &\leq \inner{\combine_t,\costhat_t} + \frac{1}{\metaLR}\ln\left( \sum_{i=1}^T \combine_{t,i} (1 -\metaLR \costhat_{t,i} + \metaLR^2 \costhat_{t,i}^2)\right) \\
     &=\inner{\combine_t,\costhat_t} + \frac{1}{\metaLR}\ln\left( 1 -\metaLR \inner{\combine_t,\costhat_t} + \metaLR^2 \sum_{i=1}^T \combine_{t,i} \costhat_{t,i}^2\right) \\
     &\leq \metaLR \sum_{i=1}^T \combine_{t,i} \costhat_{t,i}^2.
\end{align*}
The first inequality is because $\exp(-x)\leq 1-x+x^2$ holds for $x\geq -\frac{1}{2}$. To show that $\epsilon\max_{i\in[t]}|\costhat_{t,i}|\leq\frac{1}{2}$, we have
\begin{align*}
    \metaLR\max_{i\in [t]}\left|\costhat_{t,i}\right| = \metaLR  \max_{i\in [t]}\left|\inner{\tilde{a}_t^{(i)}, \wt{M}_t^{-1}x_t x_t^\T \ell_t} - \injbias_i\right| \leq \metaLR \max_{i\in [t]}\left| \wt{a}_t^{(i)^\top} \wt{M}_t^{-1}x_t\right|+\metaLR \max_{i\in [t]}\left|\injbias_{t,i}\right|.
\end{align*}
We can bound the first term by \holder's inequality
\begin{align*}
    \metaLR \max_{i\in [t]}\left| \wt{a}_t^{(i)^\top} \wt{M}_t^{-1}x_t\right|\leq \metaLR\max_{i\in [t]}\|\wt{a}_t^{(i)}\|_p\|\wt{M}_t^{-1}x_t\|_q\leq \metaLR\|\wt{M}_t^{-1}x_t\|_2\leq \frac{\metaLR d}{\beta}\|x_t\|_2\leq \frac{\metaLR d}{\beta}\|x_t\|_p\leq \frac{\epsilon d}{\beta},
\end{align*}
where $\frac{1}{p}+\frac{1}{q}=1$. The second inequality is because $\|\wt{a}_{t}^{(i)}\|_p\leq 1$ and $p\leq 2 \leq q$. The third inequality is because $\wt{M}_t$ has the smallest eigenvalue $\frac{\beta}{d}$. By the definition of $\injbias_{t,i}$, we can bound the second term as
\begin{align*}
    \metaLR\max_{i\in[t]}|\injbias_{t,i}|\leq \frac{\metaLR}{\Correct T(1-\beta)}\cdot \frac{1}{\gamma}\leq \frac{2\metaLR}{\Correct T\gamma}.
\end{align*}
Therefore, according to the choice of $\metaLR$, $\gamma$ and $\Correct$, we have $\metaLR\max_{i\in[t]}|\costhat_{t,i}|\leq \frac{1}{2}$.
Furthermore, by the definition of $v_{t+1,i}$, we have $\sum_{j=1}^T \combine_{t,j} \exp(-\metaLR \costhat_{t,j}) = \combine_{t,i} \exp(-\metaLR \costhat_{t,i})/v_{t+1,i}$. Therefore, we have
\begin{align*}
    \frac{1}{\metaLR}\ln \left(\sum_{j=1}^T \combine_{t,j} \exp(-\metaLR \costhat_{t,j}) \right) = -\frac{1}{\metaLR} \ln\left(\frac{v_{t+1,i}}{\combine_{t,i}}\right) - \costhat_{t,i}.
\end{align*}
Combining the two equations and taking summation over $t\in\calI$, we have for any $e_j\in \Delta_T$, $j\in [T]$, 
\begin{align*}
    \sum_{t\in\calI}\inner{\combine_t,\costhat_t} - \sum_{t\in\calI}\inner{e_j,\costhat_t} \leq \metaLR \sum_{t\in\calI} \sum_{i=1}^T \combine_{t,i} \costhat_{t,i}^2 + \frac{1}{\metaLR} \sum_{t \in \I}\ln\left(\frac{v_{t+1,j}}{\combine_{t,j}}\right).
\end{align*}
Further note that 
\begin{align}
\sum_{t \in \I}\ln\left(\frac{v_{t+1,j}}{\combine_{t,j}}\right) & = \sum_{t \in \I}\ln\left(\frac{\combine_{t+1,j}}{\combine_{t,j}}\right) + \sum_{t \in \I}\ln\left(\frac{v_{t+1,j}}{\frac{\mu}{T} + (1-\mu) v_{t+1,j}}\right)\nonumber\\
& \leq \ln\left(\frac{\combine_{q+1,j}}{\combine_{s,j}}\right) + \abs{\I} \log \left(\frac{1}{1 -\mu}\right) \tag{let $\I = [s,q]$}\\
& \leq \log(T^2) + \O\left(\frac{\abs{\I}}{T}\right) \label{eq:meta-fixed-share}
\end{align}
where the last step is due to $\combine_{t,j} \geq \frac{\mu}{T} = \frac{1}{T^2}$ for $j \in [T]$ and $t \in [T]$, and moreover, we have $\log(\frac{1}{1-\mu}) = \log(1 + \frac{\mu}{1-\mu}) = \O(1/T)$ as $\mu = \frac{1}{T} \leq \frac{1}{2}$. 

Combining the above two inequalities achieves
\begin{align*}
    \sum_{t\in\calI}\inner{\combine_t - e_j,\costhat_t} \leq \metaLR \sum_{t\in\calI} \sum_{i=1}^T \combine_{t,i} \costhat_{t,i}^2 + \frac{2\log T}{\metaLR} + \O\left(\frac{\abs{\I}}{\metaLR T}\right).
\end{align*}
which finishes the proof.
\end{proof}

Next, we prove the following lemma, which bounds the second term shown in~\pref{eq:switching-regret-meta}
\begin{lemma}
\label{lemma:variance-issue}
For any $t\in[T]$, setting $\Correct^2\gamma=\Theta\left(\sqrt{\frac{1}{dST^3}}\right)$ and $\beta\leq \frac{1}{2}$, \pref{alg:ell-p-ball} guarantees that
\begin{equation}
    \label{eq:variance-issue}
    \sum_{i=1}^T \combine_{t,i} \costhat_{t,i}^2 \leq \sum_{i \in [t]} \sleep_{t,i} \costhat_{t,i}^2 \leq 2\sum_{i \in [t]} \sleep_{t,i} \cost_{t,i}^2 + \order\left(\sqrt{\frac{dS}{T}}\right),
\end{equation}
where $\cost_{t,i}=\innersmall{\wt{a}_t^{(i)}, \ellbar_t}$ for $i \in [t]$.
\end{lemma}
\begin{proof} According to the definition of $\sleep_{t}$, we have
\begin{equation*}
\begin{split}
\sum_{i=1}^T \combine_{t,i} \costhat_{t,i}^2  & = \sum_{i \in [t]} \combine_{t,i} \costhat_{t,i}^2 + \sum_{i>t} \combine_{t,i} \left(\sum_{i \in [t]} \sleep_{t,i}\costhat_{t,i} \right)^2 \leq \sum_{i \in [t]} \combine_{t,i} \costhat_{t,i}^2 + \sum_{i >t} \combine_{t,i} \left(\sum_{i \in [t]} \sleep_{t,i} \costhat_{t,i}^2 \right)\\
    & =  \left(\sum_{i \in [t]} \combine_{t,i}\right)\left(\sum_{i \in [t]} \sleep_{t,i} \costhat_{t,i}^2\right) + \left(\sum_{i>t} \combine_{t,i}\right) \left(\sum_{i \in [t]} \sleep_{t,i} \costhat_{t,i}^2 \right) = \sum_{i \in [t]} \sleep_{t,i} \costhat_{t,i}^2,
\end{split}
\end{equation*}
where the inequality is because of Cauchy-Schwarz inequality. Besides, recall that $\cost_{t,i}=\innersmall{\wt{a}_t^{(i)}, \ellbar_t}$ and $\costhat_{t,i}^2 = \left(\cost_{t,i}-\injbias_{t,i}\right)^2\leq 2\cost_{t,i}^2 + 2\injbias_{t,i}^2$. According to the definition of $\injbias_{t,i}$, we know that 
\begin{align*}
    \sum_{i\in [t]} \sleep_{t,i} \injbias_{t,i}^2 \leq \frac{4}{(\Correct T)^2} \frac{1}{\gamma} \sum_{i\in [t]} \sleep_{t,i} \frac{1 - \|a_t^{(i)}\|_p}{1 - \sum_{j \in [t]} \sleep_{t,j} \|a_t^{(j)}\|_p} = \frac{4}{(\Correct T)^2} \frac{1}{\gamma} = \order\left(\sqrt{\frac{dS}{T}}\right),
\end{align*}
where the first inequality uses the fact that $\injbias_{t,i}\leq \frac{1}{\Correct T\gamma(1-\beta)}\leq \frac{2}{\Correct T\gamma}$ and the last step holds because we choose $\Correct^2\gamma=\Theta\left(\sqrt{\frac{1}{dST^3}}\right)$.
\end{proof}

Combining~\pref{lemma:switching-regret-meta} and~\pref{lemma:variance-issue}, we obtain the following lemma to bound the meta-regret.

\begin{lemma}\label{lemma:meta-final}
Define $C=\sqrt{p-1} \cdot 2^{-\frac{2}{p-1}}$. Set parameters $\metaLR=\min\left\{\sqrt{\frac{S}{dT}},\frac{1}{16d},\frac{C^2}{2}\right\}$, $\beta=8d\metaLR$, $\Correct = \frac{C}{\sqrt{dST}}$, $\gamma = 4C\sqrt{\frac{dS}{T}}$ and $\mu = \frac{1}{T}$. Then, \pref{alg:ell-p-ball} guarantees that
\begin{align*}
    \mathbb{E}\left[\metareg\right] \leq \otil\left(\sqrt{dST}\right).
\end{align*}
\end{lemma}

\begin{proof}
     It is evident to verify that the choice of $\metaLR$, $\Correct$, $\beta$ and $\gamma$ satisfies the condition required in~\pref{lemma:switching-regret-meta} and~\pref{lemma:variance-issue}, then based on the two lemmas, with $\beta=8d\metaLR\leq\frac{1}{2}$, for each interval $\calI_k$, we have
\begin{align}
    & \E\left[\sum_{t\in\calI_k}\inner{\combine_{t}-e_{j_k}, \costhat_t}\right]  \nonumber\\
    & \leq \frac{2\ln T}{\metaLR} + 2\metaLR \E\left[\sum_{t\in\calI_k}\sum_{i\in [t]}\sleep_{t,i}\cost_{t,i}^2\right]+\order\left(\metaLR|\calI_k|\sqrt{\frac{dS}{T}}\right) + \O\left(\frac{\abs{\I_k}}{\metaLR T}\right)\nonumber\\
    & \leq \frac{2\ln T}{\metaLR} + 2\metaLR\E\left[\sum_{t\in \calI_k}\sum_{i=1}^t\sleep_{t,i}\wt{a}_{t}^{(i)^\top}\wt{M}_t^{-1}x_tx_t^\top\wt{M}_t^{-1}\wt{a}_t^{(i)}\right]+\order\left(\metaLR|\calI_k|\sqrt{\frac{dS}{T}}\right) + \O\left(\frac{\abs{\I_k}}{\metaLR T}\right)\nonumber\\
    & \leq \frac{2\ln T}{\metaLR} + 2\metaLR\sum_{t\in \calI_k}\sum_{i=1}^t\sleep_{t,i}\wt{a}_{t}^{(i)^\top}\wt{M}_t^{-1}\wt{a}_t^{(i)}+\order\left(\metaLR|\calI_k|\sqrt{\frac{dS}{T}}\right) + \O\left(\frac{\abs{\I_k}}{\metaLR T}\right)\tag{$\mathbb{E}_t[x_tx_t^\top]=\wt{M}_t$}\\
    & \leq \frac{2\ln T}{\metaLR} + \frac{2\metaLR}{1-\beta}\sum_{t\in \calI_k}\sum_{i=1}^t\sleep_{t,i}\wt{a}_{t}^{(i)^\top}\left(\sum_{i=1}^t\sleep_{t,i}\wt{a}_t^{(i)}\wt{a}_t^{(i)^\top}\right)^{-1}\wt{a}_t^{(i)}+\order\left(\metaLR|\calI_k|\sqrt{\frac{dS}{T}}\right) + \O\left(\frac{\abs{\I_k}}{\metaLR T}\right)\nonumber\\
    & \leq \frac{2\ln T}{\metaLR} + 4\metaLR d|\calI_k|+\order\left(\metaLR|\calI_k|\sqrt{\frac{dS}{T}}\right)+ \O\left(\frac{\abs{\I_k}}{\metaLR T}\right).\label{eq:meta-inter}
\end{align}
Summing the regret over all the intervals achieves the following meta-regret upper bound:
\begin{equation}
\label{eq:meta-regret-bound}
\begin{split}
     \E\left[\metareg\right] &=\E\left[\sum_{k=1}^S\sum_{t\in\calI_k}\inner{\combine_{t}-e_{j_k}, \costhat_t}\right] \\ 
     & \leq \frac{2S\ln T}{\metaLR} + 4\metaLR dT+\order\left(\metaLR\sqrt{dST}\right) + \O\left(1/\epsilon\right) \leq \otil\left(\sqrt{dST}\right),
\end{split}
\end{equation}
where the last inequality is because we choose $\metaLR = \min\left\{\sqrt{\frac{S}{dT}},\frac{C^2}{2}, \frac{1}{16d}\right\}$.
\end{proof}

\subsection{Proof of~\pref{thm:ell-p-switching-regret}}
\label{appendix:proof-main}
Putting everything together, we are now ready to prove our main result (\pref{thm:ell-p-switching-regret}).
~\\

\begin{proof}
Based on the regret decomposition in~\pref{eq:decompose-ell-p}, upper bound of bias term in ~\pref{eqn:bias-term-bound}, upper bound of positive term~\pref{eqn:positive-term-bound}, base regret upper bound in~\pref{lem:base-reg} and meta regret upper bound in~\pref{eq:meta-regret-bound}, we have
\begin{align*}
    &\E[\RegS] = \mathbb{E}\left[\sum_{k=1}^S \sum_{t \in \I_k} \inner{x_t - \ucirc_k, \ell_t}\right] \\
    &\leq \frac{2}{\Correct} + \sum_{k=1}^S\frac{\log(1/\gamma)}{\baseLR_{j_k}} + \left(\frac{2^{\frac{4}{p-1}}d\baseLR_{j_k}}{(p-1)(1-\beta)}-\frac{1}{\Correct T(1-\beta)}\right)\sum_{t\in \calI_k}\frac{1-\|a_t^{(j_k)}\|_p}{1-\sum_{i=1}^t\sleep_{t,i}\|a_t^{(i)}\|_p}\\
    &\qquad + (\beta+\gamma)T + \otil\left(\sqrt{dST}\right).
\end{align*}
Importantly, note that the coefficient of the third term is actually zero. Indeed, due to the parameter configurations that $\gamma=4C\sqrt{\frac{dS}{T}}$, $\baseLR =C\sqrt{\frac{S}{dT}}$, $\Correct=\frac{C}{\sqrt{dST}}$, $\beta=8d\metaLR$, $\metaLR=\min\left\{\frac{1}{16d}, \frac{C^2}{2}, \sqrt{\frac{S}{dT}}\right\}$ and $C=\sqrt{p-1} \cdot 2^{-\frac{2}{p-1}}$, we can verify that
\begin{align*}
    \frac{2^{\frac{4}{p-1}}d\baseLR }{p-1}-\frac{1}{\Correct T} = \frac{2^{\frac{2}{p-1}}\sqrt{dS}}{\sqrt{(p-1)T}}-\frac{\sqrt{dS}}{C\sqrt{T}} = 0.
\end{align*}
Therefore, we obtain the following switching regret:
\begin{align*}
    \E[\RegS] \leq \frac{2}{\Correct}+8\sqrt{dST}+4C\sqrt{dST}+\otil(\sqrt{dST}) \leq \otil\left(\sqrt{dST}\right),
\end{align*}
which finishes the proof.
\end{proof}

In addition, we also provide the following theorem showing the expected interval regret bound, which will be  useful in the later analysis, for example, the unconstrained linear bandits in~\pref{sec: unconstrained}.
\begin{theorem}\label{thm:interval}
Define $C=\sqrt{p-1} \cdot 2^{-\frac{2}{p-1}}$. Set parameters $\metaLR=\min\left\{\sqrt{\frac{S}{dT}},\frac{1}{16d},\frac{C^2}{2}\right\}$, $\beta=8d\metaLR$, $\Correct = \frac{C}{\sqrt{dST}}$, $\gamma = 4C\sqrt{\frac{dS}{T}}$, $\mu = \frac{1}{T}$ and $\baseLR =C\sqrt{\frac{S}{dT}}$. Then,~\pref{alg:ell-p-ball} guarantees that for any interval $\calI$ and comparator $u\in\calX$,
\begin{align}\label{eqn:interval-regret-bound}
    \E\left[ \sum_{t\in\calI} \ell_t^\T x_t - \sum_{t\in\calI} \ell_t^\T u\right]\leq \otil\left(\sqrt{\frac{dT}{S}}+|\calI|\sqrt{\frac{dS}{T}}\right).
\end{align}
\end{theorem}
\begin{proof}
    Based on the regret decomposition~\pref{eq:decompose-ell-p}, ~\pref{eqn:bias-term-bound},~\pref{eqn:positive-term-bound},~\pref{lem:base-reg} and~\pref{eq:meta-inter} within rounds $t\in\calI$ starting at round $j$, we have
\begin{align*}
    &\mathbb{E}\left[\sum_{t\in\calI}\ell_t^\top x_t-\sum_{t\in\calI}\ell_t^\top u\right] \\
    &\leq \frac{2|\calI|}{\Correct T} + \frac{\log(1/\gamma)}{\baseLR_{j}} + \left(\frac{2^{\frac{4}{p-1}}d\baseLR_{j}}{(p-1)(1-\beta)}-\frac{1}{\Correct T(1-\beta)}\right)\sum_{t\in \calI}\frac{1-\|a_t^{(j)}\|_p}{1-\sum_{i=1}^t\sleep_{t,i}\|a_t^{(i)}\|_p}\\
    &\qquad + (\beta+\gamma)|\calI| + \otil\left(\metaLR|\calI|\sqrt{\frac{dS}{T}}\right) + \O\left(\frac{\abs{\I}}{\epsilon T}\right).
\end{align*}
Again, note that according to the choice of $\gamma$, $\baseLR $, $\Correct$, $\beta$ and $\metaLR$, we have
\begin{align*}
    \frac{2^{\frac{4}{p-1}}d\baseLR }{p-1}-\frac{1}{\Correct T} = \frac{2^{\frac{2}{p-1}}\sqrt{dS}}{\sqrt{(p-1)T}}-\frac{\sqrt{dS}}{C\sqrt{T}} = 0.
\end{align*}
Therefore, we have
\begin{align*}
    &\mathbb{E}\left[\sum_{t\in\calI}\ell_t^\top x_t-\sum_{t\in\calI}\ell_t^\top u\right] \\
    &\leq \frac{2}{C}|\calI|\sqrt{\frac{dS}{T}} + \frac{\log \left(\frac{1}{4C} \cdot \sqrt{\frac{T}{dS}}\right) }{C} \cdot \sqrt{\frac{dT}{S}}+8d|\calI|\sqrt{\frac{S}{dT}}+4C|\calI|\sqrt{\frac{dS}{T}}+\otil\left(\metaLR|\calI|\sqrt{\frac{dS}{T}} + \frac{\abs{\I}}{\epsilon T}\right) \\
    &\leq \otil\left(\sqrt{\frac{dT}{S}}+|\calI|\sqrt{\frac{dS}{T}}\right),
\end{align*}
which finishes the proof.
\end{proof}
\section{Extension to Smooth and Strongly Convex Set}
\label{appendix: strongly cvx set}

In this section, we extend our results for linear bandits with $\ell_p$-ball feasible domain in~\pref{sec: lp_ball} to the setting when the feasible domain is a \emph{smooth  and strongly convex set}.~\citet{arxiv'21:uniform-convex} studied the static regret for linear bandits in this setting, and we focus on the $S$-switching regret.

\subsection{Main Results}
\label{appendix: strongly-main}
Formally, we investigate adversarial linear bandits with a smooth and strongly convex feasible domain. In the following, we present the definitions of smooth set~\citep[Definition~1]{arxiv'21:uniform-convex} and strongly convex set~\citep[Definition~3]{arxiv'21:uniform-convex}.

\begin{definition}[smooth set]
\label{def: set smooth}
A compact convex set $\calX$ is smooth if and only if $|N_{\calX}(x)\cap\partial\calX^\circ|=1$ for any $x\in\partial \calX$, where $N_{\calX}(x)\triangleq\{u\in\mathbb{R}^d\mid\langle x-y, u\rangle\geq 0, \forall y\in\calX\}$, $\partial \calX$ is the boundary of $\calX$ and $\calX^\circ=\{u\in\mathbb{R}^d\mid\langle u,x\rangle\leq 1,\forall x\in\calX\}$ is the polar of $\calX$. 
\end{definition}

\begin{definition}[strongly convex set]
\label{def: set strong cvx}
Let $\calX$ be a centrally symmetric set with non-empty interior. Let $\alpha>0$ be the curvature coefficient. The set $\calX$ is $\alpha$-strongly convex with respect to $\|\cdot \|_\calX$ if and only if for any $x,y,z\in \calX$ and $\gamma\in [0,1]$, we have
\begin{align*}
    \left(\gamma x+(1-\gamma)y+\frac{\alpha}{2}\gamma(1-\gamma)\|x-y\|_\calX^2 \cdot z\right)\in \calX,
\end{align*}
where $\|x\|_\calX \triangleq \inf\{\lambda>0\;|\;x\in \lambda\calX\}$ is the gauge function to $\calX$.
\end{definition}

Conventionally, we assume that $|\ell_t^\T x|\leq 1$ holds for all $x\in \calX$ and $t\in [T]$. We also assume that $\ell_p(1)\subseteq\calX\subseteq\ell_q(1)$ with $p\in(1,2]$ and $\frac{1}{p}+\frac{1}{q}=1$, where $\ell_s(r) \triangleq \{x \in \R^d \mid \norm{x}_s \leq r\}$ denotes the $\ell_s$-norm ball ($s \geq 1$) with radius $r > 0$. We here stress the connection and difference between the strongly convex set setting and the $\ell_p$-ball setting considered in~\pref{sec: lp_ball}. Note that $\calX$ is a subset of $\ell_q$ ball and includes $\ell_p$ ball. Besides, $\ell_p$ ball is also smooth when $p\in(1,2]$. Therefore, it includes $\ell_p$-ball feasible set for $p\in(1,2]$ but can be more general. Nevertheless, the switching regret bound we will prove is $\otil(d^{\nicefrac{1}{p}}\sqrt{ST})$, which recovers the $\Ot(\sqrt{dST})$ switching regret  of $\ell_p$-ball feasible domain in~\pref{thm:ell-p-switching-regret} only when $p=2$ but leads to a slightly worse dependence on $d$ when $p\in(1,2)$. Note that as $p>1$, this bound is still better than $\otil(d\sqrt{ST})$. 

Our proposed algorithm for smooth and strongly convex set is basically the same as the one proposed for the $\ell_p$ ball setting, except that we now need to modify the base algorithm based on the algorithm introduced in \citep{arxiv'21:uniform-convex} and also need to revise the construction of injected bias $\injbias_{t,i}$ and the loss estimator $\ellhat_t$ in the meta level. Specifically, in the base algorithm we use online mirror descent with the following regularizer,
\[
    R(x)=-\ln(1-\|x\|_{\calX}) - \|x\|_\calX,
\]
whose detailed update procedures are presented in~\pref{alg:strongly-base}. For the meta algorithm, the update procedures are in~\pref{alg:strongly-meta}, notably, the injected bias $\injbias_t$ is constructed according to~\pref{eq:bias-construct-sc-set} and the base loss estimator $\ellhat_t$ is constructed according to~\pref{eqn:base-alpha}. 

We have the following theorem regarding the switching regret of our proposed algorithm for linear bandits on smooth and strongly convex feasible domain.
\begin{theorem}
\label{thm:main-alpha}
Consider a compact convex set $\calX$ that is centrally symmetric with non-empty interior. Suppose that $\calX$ is smooth and $\alpha$-strongly convex with respect to $\|\cdot\|_{\calX}$ and $\ell_p(1)\subseteq\calX\subseteq\ell_{q}(1)$, $p\in(1,2]$, $\frac{1}{p}+\frac{1}{q}=1$. Define $C=\sqrt{\frac{\alpha}{10\alpha+8}}$. Set parameters $\gamma=4Cd^{\frac{1}{q}}\sqrt{\frac{S}{T}}$,  $\Correct=\frac{Cd^{-\frac{1}{q}}}{\sqrt{ST}}$, $\beta=8d^{\frac{2}{p}}\metaLR$, $\metaLR=\min\left\{\frac{1}{16d^{\frac{2}{p}}}, \frac{C^2}{2}, d^{-\frac{1}{p}}\sqrt{\frac{S}{T}}\right\}$, $\mu = \frac{1}{T}$ and $\baseLR =Cd^{-\frac{1}{p}}\sqrt{\frac{S}{T}}$. Then,~\pref{alg:strongly-meta} guarantees 
\begin{align*}
    \E[\RegS] = \E\left[ \sum_{t=1}^T \ell_t^\T x_t - \sum_{t=1}^T \ell_t^\T u_{t}\right]\leq \otil\left(d^{\nicefrac{1}{p}}\sqrt{ST}\right),
\end{align*}
where $u_1,\ldots,u_T\in \calX$ is the comparator sequence such that $\sum_{t=2}^T\indi\{u_{t-1}\ne u_t\}\leq S-1$.
\end{theorem}
\begin{algorithm}[!t]
   \caption{Algorithm for adversarial linear bandits over smooth and strongly convex set with switching regret}
   \label{alg:strongly-meta}
    \textbf{Input:} 
    clipping parameter $\gamma$, base learning rate $\eta$, meta learning rate $\metaLR$, mixing rate $\mu$, exploration parameter $\beta$, bias coefficient $\lambda$, initial uniform distribution ${p}_1\in \Delta_T$.

   \For{$t=1$ {\bfseries to} $T$}{
      Start a new base algorithm $\B_t$, which is an instance of~\pref{alg:strongly-base} with learning rate $\eta$, clipping parameter $\gamma$, and initial round $t$.

      Receive local decision $(\wt{a}_t^{(i)},a_t^{(i)},\xi_t^{(i)})$ from base algorithm $\B_i$ for each $i \leq t$.
      
      Compute the renormalized distribution $\sleep_t \in \Delta_t$ such that $\sleep_{t,i}\propto \combine_{t,i}$ for $i\in[t]$.
            
      Sample a Bernoulli random variable $\rho_t$ with mean $\beta$. 
      If $\rho_t=1$, uniformly sample $x_t$ from $\{\pm e_n\}_{n=1}^d$;
      otherwise, sample $i_t \in [t]$ according to $\sleep_{t}$, and set $x_t = \tilde{a}_t^{(i_t)}$ and $\xi_t = \xi_t^{(i_t)}$. 
      
      Make the final decision $x_t$ and receive feedback $\ell_t^\top x_t$.

      Construct the base loss estimator $\wh{\ell}_t \in \R^d$ as follows and send it to all base algorithms $\{\calB_i\}_{i=1}^t$: \label{line:estimator base}
      \begin{align}\label{eqn:base-alpha}
          \wh{\ell}_t=\frac{\indi\{\rho_t=0\}\indi\{\xi_t=0\}}{1-\beta}\cdot \frac{d(\ell_t^\T x_t)}{1-\sum_{i=1}^t \wh{p}_{t,i}\|a_{t}^{(i)}\|_\calX} \cdot x_t.
      \end{align}
      
      Construct another loss estimator $\bar{\ell}_t \in \R^d$ as
      \begin{align}\label{eqn:meta-alpha}
          \bar{\ell}_t=\wt{M}_t^{-1}x_tx_t^\top\ell_t,
      \end{align}
      where $\wt{M}_t=\frac{\beta}{d}\sum_{i=1}^de_ie_i^\top+(1-\beta)\sum_{i=1}^t\sleep_{t,i}\wt{a}_t^{(i)}\wt{a}_t^{(i)^\top}$.

      Construct the meta loss estimator $\costhat_t\in \mathbb{R}^T$ as: 
      \begin{align}\label{eq:bias-construct-sc-set}
          \costhat_{t,i} = \begin{cases}
                   \innersmall{\wt{a}_t^{(i)}, \bar{\ell}_t}-b_{t,i}, &\text{$i \leq t$,}\\
                      \sum_{j=1}^t \wh{p}_{t,j}\costhat_{t,j}, &\text{$i > t$,}
          \end{cases}
          \quad\text{ where }\;
          b_{t,i}=\frac{1}{\Correct T(1-\beta)}\frac{1-\|a_t^{(i)}\|_\calX}{1-\sum_{j=1}^t \sleep_{t,j}\|a_t^{(j)}\|_\calX}.
      \end{align}

      Meta algorithm updates the weight ${p}_{t+1}\in\Delta_T$ according to
        \begin{align}\label{eqn:meta-strategy-update-alpha}
        \combine_{t+1,i} = (1-\mu)\frac{\combine_{t,i} \exp(-\metaLR \costhat_{t,i})}{\sum_{j=1}^T \combine_{t,j} \exp(-\metaLR \costhat_{t,j})} + \frac{\mu}{T}, \quad\forall i \in [T].
        \end{align}
      
      }
\end{algorithm}

\begin{algorithm}[!h]
   \caption{Base algorithm for linear bandits on strongly convex set}
   \label{alg:strongly-base}
    \textbf{Input:} learning rate $\eta$, clipping parameter $\gamma$, initial round $t_0$.
        
    \textbf{Define:} clipped feasible domain $\calX'=\{x \mid \|x\|_\calX\leq 1-\gamma, x\in\calX\}$.

    \textbf{Initialize:} $a_{t_0}^{(t_0)}=\argmin_{x\in\calX'}R(x)$ and $\xi_{t_0}^{(t_0)}=0$.

    Draw $\wt{a}_{t_0}^{(t_0)}$ uniformly randomly from $\{\pm e_n\}_{n=1}^d$.

   \For{$t=t_0$ {\bfseries to} $T$}{
   
      Send $(\wt{a}_t^{(t_0)},a_t^{(t_0)},\xi_t^{(t_0)})$ to the meta algorithm.
      
      Receive a loss estimator $\ellhat_t$.
      
      Update the strategy based on OMD with regularizer $R(x)=-\log (1-\|x\|_\calX)-\|x\|_\calX$:
      \begin{align}
          a_{t+1}^{(t_0)} = \argmin_{a\in \calX'}\left\{\inner{a, \ellhat_t}+\frac{1}{\baseLR}D_{R}(a,a_t^{(t_0)})\right\}.
      \end{align}

      Generate a random variable $\xi_{t+1}^{(t_0)}\sim \mathbf{Ber}(\|a_{t+1}^{(t_0)}\|_\calX)$ and set
      \begin{align*}
          \wt{a}_{t+1}^{(t_0)} =
          \begin{cases}
          \nicefrac{a_{t+1}^{(t_0)}}{\|a_{t+1}^{(t_0)}\|_\calX} & \mbox{if $\xi_{t+1}^{(t_0)}=1$}, \\
                \delta e_n & \mbox{if $\xi_{t+1}^{(t_0)}=0$},
          \end{cases}
      \end{align*}
      where $n$ is uniformly chosen from $\{1,\ldots,d\}$ and $\delta$ is a uniform random variable over $\{-1,+1\}$.
      }
\end{algorithm}
In the following, we first introduce some definitions and lemmas useful for the analysis in strongly convex set in~\pref{appendix:pre} and then prove~\pref{thm:main-alpha} in~\pref{appendix:unbias-loss-estimator-alpha}--\ref{appendix:proof-main-alpha}. To prove~\pref{thm:main-alpha}, similar to the analysis structure in~\pref{appendix:lp_ball}, we first prove the unbiasedness of loss estimators in~\pref{appendix:unbias-loss-estimator-alpha}, and then in~\pref{appendix:regret-decomposition-alpha}, we decompose the regret into several terms, and subsequently upper bound each term in~\pref{appendix:bound-part1-alpha},~\pref{appendix:base-reg-alpha}, and~\pref{appendix:metareg-alpha}. We finally put everything together and present the proof in~\pref{appendix:proof-main-alpha}.

\subsection{Preliminary}\label{appendix:pre}
This subsection collects some useful definitions and lemmas for the analysis. We refer the reader to~\citep{arxiv'21:uniform-convex} for detailed introductions. Define $\|\cdot\|_\calX$ is the gauge function to $\calX$ as 
\begin{align}\label{eqn:gauge}
    \|x\|_\calX=\inf\{\lambda>0\;|\;x\in \lambda\calX\}.    
\end{align}
The polar of $\calX$ is defined as $\calX^\circ=\{\ell\in \mathbb{R}^d\;|\;\inner{x,\ell}\leq 1, \forall x\in \calX\}$. If $\calX$ is symmetric, then based on the assumption $|\innersmall{x,\ell_t}|\leq 1$, we have $\ell_t\in \calX^\circ$. Based on the definition of gauge function, we have $\|x\|_\calX\leq 1$ for all $x\in \calX$. In addition, we have the \holder's inequality $\inner{x,\ell}\leq \|x\|_\calX\cdot \|\ell\|_{\calX^\circ}$. In this problem, we also assume that $\ell_p(1)\subseteq \calX\subseteq\ell_q(1)$, $p\in(1,2]$, $\frac{1}{p}+\frac{1}{q}=1$ which implies that $\{\pm e_n\}_{n\in [d]}\subseteq \calX$.  
By the definition of $\calX^\circ$, we also have $\ell_p(1)\subseteq\calX^\circ\subseteq\ell_q(1)$.
The following lemmas show some useful identities for the regularizer $R(x)$.
\begin{lemma}[Lemma 5 of~\citet{arxiv'21:uniform-convex}]
\label{lem: gauge diff}
A gauge function $\|\cdot\|_\calX$ is differentiable at $x\in \mathbb{R}^d\backslash{\{\mathbf{0}\}}$ if and only if its support set $S(\calX^\circ,x)=\{h\in \calX^\circ \mid \inner{h,x}=\sup_{h'\in\calX^\circ}\inner{h',x}\}$ contains a single point $h$. If this is the case, we have $\nabla\|\cdot\|_\calX(x)=d$. Besides, the following assertions are true: (1) $\| (\nabla\|\cdot\|_\calX(x)) \|_{\calX^{\circ}}=1$; (2) $\nabla\|\cdot\|_{\calX}(\lambda x)=\nabla\|\cdot\|_{\calX}(x)$, for any $\lambda>0$; (3) if $\calX^\circ$ is strictly convex, then $\|\cdot\|_{\calX}$ is differentiable in $\mathbb{R}^d\backslash{\{\mathbf{0}\}}$.
\end{lemma}

\begin{lemma}[Corollary 8 of~\citet{arxiv'21:uniform-convex}]\label{lem: bound-bregman}
Let $\calX$ be a centrally symmetric set with non empty interior. Assume that $\calX$ is $\alpha$-strongly convex with respect to $\|\cdot\|_{\calX}$. Then for any $(u,v)\in \mathbb{R}^n$,
\begin{align*}
    D_{\frac{1}{2}\|\cdot\|_{\calX^\circ}^2}(u,v)\leq \frac{4(\alpha+1)}{\alpha}\|u-v\|_{\calX^\circ}^2.
\end{align*}
\end{lemma}

\begin{lemma}[Lemma 15 of~\citet{arxiv'21:uniform-convex}]\label{lem: identity}
    Assume $\calX\subseteq\mathbb{R}^d$ is strictly convex compact and smooth set. Let $x\in \calX$ such that $\|x\|_\calX<1$ and $h\in \mathbb{R}^d\backslash{\{\mathbf{0}\}}$. We have $R(x)$ is differentiable on $\text{int}(\calX)$ and
    \begin{align*}
        \nabla R(x) = \frac{\|x\|_\calX}{1-\|x\|_\calX}\cdot \nabla\|\cdot\|_\calX(x), \\
        R^*(h) = \|h\|_{\calX^\circ}-\ln (1+\|h\|_{\calX^\circ}), \\
        \nabla R^*(h) = \frac{\|h\|_{\calX^\circ}}{1+\|h\|_{\calX^\circ}}\nabla\|\cdot\|_{\calX^\circ}(h).
    \end{align*}
\end{lemma}


\subsection{Unbiasedness of Loss Estimator}\label{appendix:unbias-loss-estimator-alpha} We first show that the loss estimator for the meta algorithm $\bar{\ell}_t$ and the one for the base algorithm $\ellhat_t$ constructed in~\pref{alg:strongly-meta} are unbiased.

\begin{lemma}\label{lem:unbiased-alpha}
The meta loss estimator $\ellbar_t$ defined in~\pref{eqn:meta-alpha} and the base loss estimator $\ellhat_t$ defined in~\pref{eqn:base-alpha} satisfy that $\mathbb{E}_t[\ellbar_t]=\ell_t$ and $\mathbb{E}_t[\ellhat_t]=\ell_t$ for all $t\in[T]$.
\end{lemma}
\begin{proof}
The the unbiasedness of $\ellbar_t$ can be proven in the exact same way as in~\pref{eqn:proof-meta-unbiased}. For $\ellhat_t$, according to the sampling scheme of $x_t$, we have
\begin{align*}
    \mathbb{E}_t[\ellhat_t]&=\mathbb{E}_t\left[\frac{1-\xi_t}{1-\beta}\cdot \frac{d}{1-\sum_{i=1}^t\sleep_{t,i}\|a_{t}^{(i)}\|_\calX}x_tx_t^\top \ell_t\cdot\mathbbm{1}\{\rho_t=0\}\right] \\
    &=\mathbb{E}_t\left[(1-\xi_t)\cdot \frac{d}{1-\sum_{i=1}^t\sleep_{t,i}\|a_{t}^{(i)}\|_\calX}x_tx_t^\top \ell_t \;\bigg\vert\;\rho_t=0\right] \\
    &=\mathbb{E}_t\left[\sum_{j=1}^t\sleep_{t,j}\cdot \frac{d(1-\xi_t^{(j)})}{1-\sum_{i=1}^t\sleep_{t,i}\|a_{t}^{(i)}\|_\calX}\wt{a}_t^{(j)}\wt{a}_t^{(j)^\top} \ell_t\right] \\
    &=\sum_{j=1}^t\sleep_{t,j}\cdot \frac{d(1-\|a_t^{(j)}\|_\calX)}{1-\sum_{i=1}^t\sleep_{t,i}\|a_{t}^{(i)}\|_\calX}\frac{1}{d}\sum_{n=1}^de_ne_n^\top \ell_t = \ell_t.
\end{align*}
This ends the proof.
\end{proof}

\subsection{Regret Decomposition}\label{appendix:regret-decomposition-alpha}
Similar to the analysis in~\pref{appendix:lp_ball}, we decompose the expected switching regret into five terms and then bound each term respectively. Again, we split the horizon to $\I_1,\ldots, \I_S$, and let $j_k$ be the start time stamp of $\calI_k$. We introduce $u_t'=(1-\gamma)u_t$ and $\ucirc_k'=(1-\gamma)\ucirc_k$ to ensure that $u_t'\in\calX'$ for $t\in [T]$ and $\ucirc_k'\in\calX'$ for $k\in[S]$, where $\calX'=\{x \mid \|x\|_\calX\leq 1-\gamma, x\in\calX\}$. Similar to the decomposition method of~\pref{eq:decompose-ell-p}, the expected regret can be decomposed as
\begin{align}
    & \E[\RegS]  = \mathbb{E}\left[\sum_{t=1}^T\inner{x_t, \ell_t}-\sum_{t=1}^T\inner{u_t, \ell_t}\right] \nonumber\\
    & = \mathbb{E}\Bigg[\underbrace{\sum_{k=1}^S \sum_{t \in \I_k} \inner{\combine_{t} - e_{j_k},\costhat_t}}_{\metareg} + \underbrace{\sum_{k=1}^S \sum_{t \in \I_k} \inner{{a}_t^{(j_k)}-\ucirc_k',\ellhat_t}}_{\basereg} + \underbrace{\sum_{t=1}^T \sum_{i=1}^t \sleep_{t,i}\injbias_{t,i}}_{\posterm} \nonumber\\
    & \qquad \quad  -\underbrace{ \sum_{k=1}^S \sum_{t \in \I_k} \injbias_{t,j_k}}_{\negterm} + \underbrace{\sum_{t=1}^T\inner{u_t'-u_t,\ell_t}-\beta\sum_{t=1}^T\sum_{i=1}^t\sleep_{t,i}\inner{\wt{a}_t^{(i)}, \ell_t}}_{\bias}\Bigg]. \label{eqn:decomposition-alpha}
\end{align}
In the following, we will bound each term respectively.

\subsection{Bounding \bias~and~\posterm}\label{appendix:bound-part1-alpha}
\paragraph{\bias.} \bias~term can still be bounded by $(\beta+\gamma) T$ as
\begin{align}\label{eqn:bias-term-bound-alpha}
    &\sum_{t=1}^T\inner{u_t'-u_t, \ell_t} - \beta\sum_{t=1}^T\sum_{i=1}^t\sleep_{t,i}\inner{\wt{a}_t^{(i)}, \ell_t} \nonumber\\
    & \leq \sum_{t=1}^T\left((1-\gamma)-1\right)\inner{u_t, \ell_t} +\beta T\nonumber\\
    & \leq \sum_{t=1}^T\left(1-(1-\gamma)\right)+\beta T = (\beta+\gamma) T.
\end{align}

\paragraph{\posterm.} According to the definition of $\injbias_{t,i}$, we have 
\begin{align}\label{eqn:posterm-term-bound-alpha}
    \frac{1}{\Correct T(1-\beta)} \sum_{t=1}^T \sum_{i=1}^t \frac{\sleep_{t,i} (1 - \norm{a_t^{(i)}}_\calX)}{1 - \sum_{j=1}^t \sleep_{t,j} \norm{a_t^{(j)}}_\calX} = \frac{1}{\Correct(1-\beta)}\leq \frac{2}{\Correct},
\end{align}
where the last inequality is because $\beta\leq \frac{1}{2}$.

In the following two subsections, we bound \basereg~and~\metareg~respectively.

\subsection{Bounding \basereg}\label{appendix:base-reg-alpha}
Before bounding the term~\basereg, we show the following two lemmas which will be useful in the analysis. The first lemma bounds the scale of the loss estimator used for the base algorithm.
\begin{lemma}\label{lem:bounded-estimator}
For any $x\in (1-\gamma)\calX$ and $\eta$, define $u=\nabla R(x)-\eta\ellhat_t$ and $v=\nabla R(x)$ with $\ellhat_t$ defined in~\pref{alg:strongly-meta}. We have
\begin{align*}
    \frac{\|u\|_{\calX^\circ}-\|v\|_{\calX^\circ}}{1+\|v\|_{\calX^\circ}}\geq -\frac{2\eta d}{\gamma}.
\end{align*}
\end{lemma}

\begin{proof}
First, note that using \pref{lem: gauge diff} and \pref{lem: identity}, the denominator can be written as
\begin{align}\label{eqn:local-norm-alpha}
    \frac{1}{1+\|v\|_{\calX^\circ}}=\frac{1}{1+\|\nabla R(x)\|_{\calX^\circ}}=(1-\|x\|_\calX)(\|(\nabla \|\cdot\|_{\calX}(x))\|_{\calX^\circ})^{-1} = 1-\|x\|_{\calX}.
\end{align}
For the numerator, note that $a_t^{(i)}\in (1-\gamma)\calX$ for all $t\in [T]$ and $i\in [t]$ and $\beta\leq \frac{1}{2}$, we have 
\begin{align*}
    \|\ellhat_t\|_{\calX^\circ}\leq  \frac{d}{(1-\beta)(1-(1-\gamma))}|x_t^\top \ell_t|\cdot\|x_t\|_{\calX^\circ}\cdot\mathbbm{1}\{x_t\in \{\pm e_n\}_{n\in[d]}\}\leq \frac{2d\|x_t\|_{\calX^\circ}}{\gamma} \mathbbm{1}\{x_t\in \{\pm e_n\}_{n\in[d]}\}.
\end{align*}
Therefore, according to triangle inequality, we have 
\begin{align*}
    \|u\|_{\calX^\circ}-\|v\|_{\calX^\circ} \geq -\eta\|\ellhat_t\|_{\calX^\circ}\geq -\frac{2d\eta}{\gamma}\|x_t\|_{\calX^\circ}\cdot\ind\{x_t\in \{\pm e_n\}_{n\in [d]}\}.
\end{align*}
Note that $\calX\subseteq\ell_q(1)$, we have $\ell_{q}(1)^\circ=\ell_p(1)\subseteq\calX^\circ$, which means that $e_n\in \calX^\circ$. This means that $\|e_n\|_{\calX^\circ}\leq 1$ and we have
\begin{align*}
    \|u\|_{\calX^\circ}-\|v\|_{\calX^\circ}\geq -\frac{2\eta d}{\gamma},
\end{align*}
which finishes the proof.
\end{proof}

The second lemma helps to bound the stability of the base algorithm, which is originally introduced in~\citep[Lemma 17]{arxiv'21:uniform-convex}. For completeness, we include the proof here.
\begin{lemma}\label{lem:stab lemma uniform convex}
Suppose $\calX$ to be a $\alpha$-strongly convex and centrally symmetric set with non-empty interior. Let $x\in \calX$ such that $\|x\|_\calX\leq 1-\gamma$ and if $\eta\|\ellhat_t\|_{\calX^\circ}\leq \frac{1}{2}$,
\begin{align*}
    D_{R^*}(\nabla R(x)-\eta \ellhat_t, \nabla R(x)) \leq (1-\|x\|_{\calX})\left(1+\frac{4(\alpha+1)}{\alpha}\right)\eta^2\|\ellhat_t\|_{\calX^\circ}^2.
\end{align*}
\end{lemma}
\begin{proof}
Define $u=\nabla R(x)-\eta\ellhat_t$, $v=\nabla R(x)$ and $z=\frac{\|u\|_{\calX^\circ}-\|v\|_{\calX^\circ}}{1+\|v\|_{\calX^\circ}}$. By the definition of Bregman divergence and using~\pref{lem: identity}, we have
\begin{align*}
    D_{R^*}(u,v) &= R^*(u)-R^*(v)-\inner{\nabla R^*(v), u-v} \\
    &= \|u\|_{\calX^\circ}-\|v\|_{\calX^\circ} - \ln\left(\frac{1+\|u\|_{\calX^\circ}}{1+\|v\|_{\calX^\circ}}\right) - \frac{\|v\|_{\calX^\circ}}{1+\|v\|_{\calX^\circ}}\inner{\nabla\|\cdot\|_{\calX^\circ}(v), u-v}\\
    &=z-\ln(1+z)+\frac{1}{1+\|\nu\|_{\calX^\circ}}\left[\|v\|_{\calX^\circ}(\|u\|_{\calX^\circ}-\|v\|_{\calX^\circ})-\|v\|_{\calX^\circ}\inner{\nabla\|\cdot\|_{\calX^\circ}, u-v}\right] \\
    &=z-\ln(1+z)-\frac{1}{2}\frac{(\|u\|_{\calX^\circ}-\|v\|_{\calX^\circ})^2}{1+\|v\|_{\calX^\circ}} + \frac{D_{\frac{1}{2}\|\cdot\|_{\calX^\circ}^2}(u,v)}{1+\|v\|_{\calX^\circ}}\\
    &\leq z-\ln(1+z)+ \frac{D_{\frac{1}{2}\|\cdot\|_{\calX^\circ}^2}(u,v)}{1+\|v\|_{\calX^\circ}}
\end{align*}
Note that $z\geq -\frac{1}{2}$ as $\frac{\|u\|_{\calX^\circ}-\|v\|_{\calX^\circ}}{1+\|v\|_{\calX^\circ}} \geq -\eta\|\ellhat_t\|_{\calX^\circ}\geq -\frac{1}{2}$, we have $z-\ln(1+z)\leq z^2$. Therefore, we have
\begin{align*}
    D_{R^*}(u,v)\leq \left(\frac{\|u\|_{\calX^\circ}-\|v\|_{\calX^\circ}}{1+\|v\|_{\calX^\circ}}\right)^2+\frac{1}{1+\|v\|_{\calX^\circ}}D_{\frac{1}{2}\|\cdot\|_{\calX^\circ}^2}(u,v).
\end{align*}
Note that according to~\pref{lem: gauge diff}, we have $\frac{1}{1+\|v\|_{\calX^\circ}}=1-\|x\|_{\calX}$. Therefore, using triangle inequality leads to
\begin{align*}
    D_{R^*}(u,v)\leq (1-\|x\|_{\calX})^2\|u-v\|_{\calX^\circ}^2 + (1-\|x\|_{\calX})D_{\frac{1}{2}\|\cdot\|_{\calX^\circ}^2}(u,v).
\end{align*}
Finally, using~\pref{lem: bound-bregman}, we have
\begin{align*}
    D_{R^*}(u,v)&\leq (1-\|x\|_{\calX})^2\|u-v\|_{\calX^\circ}^2+(1-\|x\|_{\calX})\cdot\frac{4(\alpha+1)}{\alpha}\|u-v\|_{\calX^\circ}^2 \\
    &\leq (1-\|x\|_{\calX})\left(1+\frac{4(\alpha+1)}{\alpha}\right)\eta^2\|\ellhat_t\|_{\calX^\circ}^2.
\end{align*}
\end{proof}

With the help of~\pref{lem:bounded-estimator} and~\pref{lem:stab lemma uniform convex}, we are able to bound \basereg.

\begin{lemma}\label{lem:base-reg-alpha}
For an arbitrary interval $\calI$ started at round $j$, setting $\gamma = 4d\baseLR '$ for all $j'\in[T]$, \pref{alg:strongly-meta} ensures that the base regret of $\calB_j$ with learning rate $\baseLR $ (starting from round $j$) for any comparator $u\in\calX'$ is at most
\begin{align*}
    \E\left[\sum_{t \in \I} \inner{a_t^{(j)} - u,\hat{\ell_t}}\right] \leq \frac{\log(1/\gamma)}{\baseLR } + \frac{2d^{\frac{2}{p}}\baseLR }{1-\beta}\cdot\left(1+\frac{4(\alpha+1)}{\alpha}\right) \sum_{t \in \I} \frac{1 - \norm{a_t^{(j)}}_\calX}{1-\sum_{i=1}^{t} \sleep_{t,i}\norm{a_t^{(i)}}_\calX}.
\end{align*}
\end{lemma}
\begin{proof}
Again, according to the standard analysis of OMD (see~\pref{lemma:OMD-bubeck}) we have
    \begin{align*}
    \E\left[\sum_{t\in\calI} \inner{a_t^{(j)} - u,\hat{\ell_t}}\right] \leq \frac{R(u)-R(a_j^{(j)})}{\baseLR } + \frac{1}{\baseLR } \sum_{t\in \calI} \E\left[ D_{R^*}\left(\nabla R(a_t^{(j)}) - \baseLR  \hat{\ell}_t, \nabla R(a_t^{(j)})\right)\right].
\end{align*}
The first term can still be upper bounded by $\frac{\log(1/\gamma)}{\baseLR }$ as $a_j^{(j)}=\argmin_{x\in\calX'}R(x)$ and $u\in\calX'=\{x \mid \|x\|_\calX\leq 1-\gamma\}$, we have
\begin{align*}
    R(u)-R(a_j^{(j)}) \leq -\log(1-(1-\gamma)) - 0 = -\log\gamma.
\end{align*}
For the second term, we will show that 
\begin{align*}
    &\mathbb{E}_t\left[D_{R^*}\left(\nabla R(a_t^{(j)})-\baseLR \ellhat_t, \nabla R(a_t^{(j)})\right)\right] \leq \frac{2d^{\frac{2}{p}}\baseLR ^2}{1-\beta}\cdot\left(1+\frac{4(\alpha+1)}{\alpha}\right) \sum_{t \in \I} \frac{1 - \norm{a_t^{(j)}}_\calX}{1-\sum_{i=1}^{t} \sleep_{t,i}\norm{a_t^{(i)}}_\calX}.
\end{align*}
According to~\pref{eqn:local-norm-alpha} and the choice of $\baseLR $ and $\gamma$, we have $\baseLR \|\ellhat_t\|_{\calX^\circ}\leq \frac{2d\baseLR }{\gamma}=\frac{1}{2}$. Based on~\pref{lem:stab lemma uniform convex}, we only need to show that
\begin{align*}
\mathbb{E}_t\left[\|\ellhat_t\|_{\calX^\circ}^2\right]\leq \frac{2d^{\frac{2}{p}}}{(1-\beta)(1-\sum_{i=1}^t\sleep_{t,i}\|a_t^{(i)}\|_{\calX})}.    
\end{align*}
In fact, according to the definition of $\ellhat_t$, we have
\begin{align*}
& \mathbb{E}_t\left[\|\ellhat_t\|_{\calX^{\circ}}^2\right] \\
& \leq \frac{d^2}{(1-\beta)^2(1-\sum_{i=1}^t\sleep_{t,i}\|a_t^{(i)}\|_{\calX})^2}\mathbb{E}_t\left[(1-\xi_t)^2\|x_t\|_{\calX^\circ}^2\cdot|x_t^\top\ell_t|^2\cdot\mathbbm{1}\{\rho_t=0\}\right]\\
& \leq \frac{d^2}{(1-\beta)^2(1-\sum_{i=1}^t\sleep_{t,i}\|a_t^{(i)}\|_{\calX})^2}\mathbb{E}_t\left[(1-\beta)\sum_{j=1}^t\sleep_{t,j}(1-\xi_t^{(j)})^2\|\wt{a}_t^{(j)}\|_{\calX^\circ}^2\cdot|\wt{a}_{t}^{(j)^\top}\ell_t|^2\right]\\
& \leq \frac{d^2}{(1-\beta)(1-\sum_{i=1}^t\sleep_{t,i}\|a_t^{(i)}\|_{\calX})^2}\sum_{j=1}^t\sleep_{t,j}\mathbb{E}_t\left[(1-\|a_t^{(j)}\|_\calX)\cdot\frac{1}{d}\sum_{n=1}^d\|e_n\|_{\calX^\circ}^2\cdot|\ell_{t,n}|^2\right].
\end{align*}
Note that $\calX \subseteq \ell_q(1)$, we have $\ell_p(1)\subseteq \calX^\circ$, which means that $e_n\in \calX^\circ$ and $\|e_n\|_{\calX^\circ}\leq 1$. Also using the fact that $\ell_p(1)\subseteq \calX$, we have $\ell_t\in \calX^\circ\subseteq \ell_q(1)$ and $\|\ell_t\|_2^2\leq d^{1-\frac{2}{q}}\|\ell_t\|_q^2\leq d^{1-\frac{2}{q}}$. Therefore, we have
\begin{align*}
    \mathbb{E}_j\left[\|\ellhat_t\|_{\calX^\circ}^2\right] \leq \frac{2d^{\frac{2}{p}}\sum_{j=1}^t\sleep_{t,j}(1-\|a_t^{(j)}\|_\calX)}{(1-\beta)(1-\sum_{i=1}^t\sleep_{t,i}\|a_t^{(i)}\|_{\calX})^2}=\frac{2d^{\frac{2}{p}}}{(1-\beta)(1-\sum_{i=1}^t\sleep_{t,i}\|a_t^{(i)}\|_\calX)},
\end{align*}
which finishes the proof.
\end{proof}

\subsection{Bounding \metareg}\label{appendix:metareg-alpha}

In this section, we first prove several useful lemmas and then bound the term \metareg. We prove the following lemma, which is a counterpart of~\pref{lemma:switching-regret-meta}.
\begin{lemma}
\label{lemma:switching-regret-meta-alpha}
For an arbitrary interval $\calI\subseteq[T]$ started at round $j$, setting $\frac{\metaLR}{\Correct\gamma T}\leq \frac{1}{8}$, $\beta=8d^{\frac{2}{p}}\metaLR\leq \frac{1}{2}$ and $\mu = \frac{1}{T}$, \pref{alg:strongly-meta} guarantees that
\begin{equation}
    \label{eq:switching-regret-meta-alpha}
    \sum_{t\in \calI} \inner{\combine_{t}-e_{j}, \costhat_t} \leq \frac{2\ln T}{\metaLR} + \metaLR \sum_{t\in\calI} \sum_{i=1}^T \combine_{t,i} \costhat_{t,i}^2 + \O\left(\frac{\abs{\I}}{\metaLR T}\right).
\end{equation}
\end{lemma}

\begin{proof} 
Define $v_{t+1,i} \triangleq \frac{\combine_{t,i} \exp(-\metaLR \costhat_{t,i})}{\sum_{t=1}^T \combine_{t,i} \exp(-\metaLR \costhat_{t,i})}$ for all $i\in [T]$. Then $\combine_{t+1,i} = \frac{\mu}{T} + (1-\mu) v_{t+1,i}$. Note that
\begin{align*}
    & \inner{\combine_t,\costhat_t} + \frac{1}{\metaLR}\ln\left( \sum_{i=1}^T \combine_{t,i} \exp(-\metaLR \costhat_{t,i})\right) \\
    &\leq \inner{\combine_t,\costhat_t} + \frac{1}{\metaLR}\ln\left( \sum_{i=1}^T \combine_{t,i} (1 -\metaLR \costhat_{t,i} + \metaLR^2 \costhat_{t,i}^2)\right) \\
     &=\inner{\combine_t,\costhat_t} + \frac{1}{\metaLR}\ln\left( 1 -\metaLR \inner{\combine_t,\costhat_t} + \metaLR^2 \sum_{i=1}^T \combine_{t,i} \costhat_{t,i}^2\right) \\
     &\leq \metaLR \sum_{i=1}^T \combine_{t,i} \costhat_{t,i}^2.
\end{align*}
The first inequality is because $\exp(-x)\leq 1-x+x^2$ for $x\geq -\frac{1}{2}$ and according to the choice of $\metaLR$, $\gamma$ and $\Correct$, we have
\begin{align*}
    \metaLR\max_{i\in [t]}\left|\costhat_{t,i}\right| \leq \metaLR \max_{i\in [t]}\left| \wt{a}_t^{(i)^\top} \wt{M}_t^{-1}x_t-\injbias_{t,i} \right| \leq \metaLR\max_{i\in[t]}\left|\wt{a}_t^{(i)^\top}\wt{M}_t^{-1}x_t\right|+\metaLR\max_{i\in[t]}\left|\injbias_{t,i}\right|.
\end{align*}
For the first term, by using \holder's inequality, we have
\begin{align*}
    \metaLR\max_{i\in[t]}\left|\wt{a}_t^{(i)^\top}\wt{M}_t^{-1}x_t\right|&\leq \metaLR\max_{i\in[t]}\|\wt{a}_t^{(i)}\|_{\calX}\cdot\|\wt{M}_t^{-1}x_t\|_{\calX^\circ} \\
    &\leq \metaLR\|\wt{M}_t^{-1}x_t\|_p \tag{$\wt{a}_{t}^{(i)}\in\calX$ and $\ell_p(1)\subseteq\calX^\circ$}\\
    &\leq \metaLR d^{\frac{1}{p}-\frac{1}{2}}\|\wt{M}_t^{-1}x_t\|_2\\
    &\leq \frac{\metaLR d^{\frac{1}{2}+\frac{1}{p}}}{\beta}\cdot \|x_t\|_2 \tag{$\wt{M}_t\succeq\frac{\beta}{d}I$}\\
    &\leq \frac{\metaLR d^{\frac{2}{p}}}{\beta}. \tag{$\|x\|_2\leq d^{\frac{1}{2}-\frac{1}{q}}\|x\|_q\leq d^{\frac{1}{p}-\frac{1}{2}}$}
\end{align*}
In above argument, we use the fact that for vector $x \in \R^d$ and $0 < s<r$, we have $\norm{x}_r \leq \norm{x}_s \leq d^{\frac{1}{s} - \frac{1}{r}} \norm{x}_r$. Moreover, note that $p \in (1,2]$ and $\ell_p(1) \subseteq \X \subseteq \ell_q(1)$.

For the second term, according to the definition of $\injbias_{t,i}$, $|\injbias_{t,i}|\leq \frac{1}{\Correct T(1-\beta)\gamma}\leq \frac{2}{\Correct T\gamma}$. Therefore, combining the above two bounds shows that $\metaLR\max_{i\in[t]}|\costhat_{t,i}|\leq \frac{1}{8}+\frac{1}{4}\leq \frac{1}{2}$ according to the choice of $\epsilon$, $\gamma$, and $\lambda$. Furthermore, by the definition of $v_{t+1,i}$, we have $\sum_{j=1}^T \combine_{t,j} \exp(-\metaLR \costhat_{t,j}) = \combine_{t,i} \exp(-\metaLR \costhat_{t,i})/v_{t+1,i}$. Therefore, we have
\begin{align*}
    \frac{1}{\metaLR}\ln \left(\sum_{j=1}^T \combine_{t,j} \exp(-\metaLR \costhat_{t,j}) \right) = -\frac{1}{\metaLR} \ln\left(\frac{v_{t+1,i}}{\combine_{t,i}}\right) - \costhat_{t,i}.
\end{align*}
Combining the two equations and taking summation over $t\in\calI$, we have for any $e_j\in \Delta_T$, $j\in [T]$, 
\begin{align*}
    \sum_{t\in\calI}\inner{\combine_t,\costhat_t} - \sum_{t\in\calI}\inner{e_j,\costhat_t} \leq \metaLR \sum_{t\in\calI} \sum_{i=1}^T \combine_{t,i} \costhat_{t,i}^2 + \frac{1}{\metaLR} \sum_{t \in \I}\ln\left(\frac{v_{t+1,j}}{\combine_{t,j}}\right).
\end{align*}
The second term can be dealt with according to~\pref{eq:meta-fixed-share} and we then have
\begin{align*}
    \sum_{t\in\calI}\inner{\combine_{t}-e_j,\costhat_t}\leq \frac{2\ln T}{\metaLR}+\metaLR\sum_{t\in\calI}\sum_{i=1}^T\combine_{t,i}\costhat_{t,i}^2 + \O\left(\frac{\abs{\I}}{\metaLR T}\right),
\end{align*}
which finishes the proof.
\end{proof}

Next, we prove the following lemma, which bounds the second term shown in~\pref{eq:switching-regret-meta-alpha}
\begin{lemma}
\label{lemma:variance-issue-alpha}
For any $t\in[T]$, setting $\Correct^2\gamma=\Theta\left(d^{-\frac{1}{q}}\sqrt{\frac{1}{ST^3}}\right)$, \pref{alg:strongly-meta} guarantees that
\begin{equation}
    \label{eq:variance-issue}
    \sum_{i=1}^T \combine_{t,i} \costhat_{t,i}^2 \leq \sum_{i \in [t]} \sleep_{t,i} \costhat_{t,i}^2 \leq 2\sum_{i \in [t]} \sleep_{t,i} \cost_{t,i}^2 + \order\left(d^{\frac{1}{q}}\sqrt{\frac{S}{T}}\right),
\end{equation}
where $\cost_{t,i}=\innersmall{\wt{a}_t^{(i)}, \ellbar_t}$.
\end{lemma}
\begin{proof} According to the definition of $\sleep_{t}$ and $\costhat_t$, we have
\begin{align*}
    \sum_{i=1}^T \combine_{t,i} \costhat_{t,i}^2  & = \sum_{i \in [t]} \combine_{t,i} \costhat_{t,i}^2 + \sum_{i>t} \combine_{t,i} \left(\sum_{j=1}^t \sleep_{t,j}\costhat_{t,j}\right)^2 \leq \sum_{i \in [t]} \combine_{t,i} \costhat_{t,i}^2 + \sum_{i \notin [t]} \combine_{t,i} \left(\sum_{j \in [t]} \sleep_{t,j} \costhat_{t,j}^2 \right)\\
    & =  \left(\sum_{i \in [t]} \combine_{t,i}\right)\left(\sum_{i \in [t]} \sleep_{t,i} \costhat_{t,i}^2\right) + \left(\sum_{i \notin [t]} \combine_{t,i}\right) \left(\sum_{i \in [t]} \sleep_{t,i} \costhat_{t,i}^2 \right) = \sum_{i \in [t]} \sleep_{t,i} \costhat_{t,i}^2,
\end{align*}
where the inequality is because of Cauchy-Schwarz inequality. Besides, recall that $\cost_{t,i}=\inner{\wt{a}_t^{(i)}, \ellbar_t}$ and $\costhat_{t,i}^2 = \left(\cost_{t,i}-\injbias_{t,i}\right)^2\leq 2\cost_{t,i}^2 + 2\injbias_{t,i}^2$. According to the definition of $\injbias_{t,i}$, we know that 
\begin{align*}
    \sum_{i\in [t]} \sleep_{t,i} \injbias_{t,i}^2 \leq \frac{4}{(\Correct T)^2} \frac{1}{\gamma} \sum_{i\in [t]} \sleep_{t,i} \frac{1 - \|a_t^{(i)}\|_\calX}{1 - \sum_{i \in [t]} \sleep_{t,i} \|a_t^{(i)}\|_\calX} = \frac{4}{(\Correct T)^2} \frac{1}{\gamma} = \order\left(d^{\frac{1}{q}}\sqrt{\frac{S}{T}}\right),
\end{align*}
where the last step holds because we choose $\Correct^2\gamma=\Theta\left(d^{-\frac{1}{q}}\sqrt{\frac{1}{ST^3}}\right)$.
\end{proof}

Combining~\pref{lemma:switching-regret-meta-alpha} and~\pref{lemma:variance-issue-alpha}, we  obtain the following lemma showing the upper bound for \metareg, which is exactly the same as~\pref{lemma:meta-final} except for the choice of parameters.

\begin{lemma}\label{lemma:meta-final-alpha}
Define $C=\sqrt{\frac{\alpha}{10\alpha+8}}$. Set $\metaLR=\min\left\{d^{-\frac{1}{p}}\sqrt{\frac{S}{T}},\frac{1}{16d^{\frac{2}{p}}},\frac{C^2}{2}\right\}$, $\beta=8d^{\frac{2}{p}}\metaLR$, $\Correct=\frac{Cd^{-\frac{1}{q}}}{\sqrt{ST}}$, $\gamma=4Cd^{\frac{1}{q}}\sqrt{\frac{S}{T}}$ and $\mu = \frac{1}{T}$.~\pref{alg:strongly-meta} guarantees that
\begin{align*}
    \mathbb{E}\left[\metareg\right] \leq \otil\left(d^{\frac{1}{p}}\sqrt{ST}\right).
\end{align*}
\end{lemma}
\begin{proof}
     First, it is direct to check that the choice of $\Correct$, $\gamma$ and $\metaLR$ satisfies the condition required in~\pref{lemma:switching-regret-meta-alpha} and~\pref{lemma:variance-issue-alpha}. Based on the two lemmas, for each interval $\calI_k$, let $j_k$ be the start time stamp for $\calI_k$. As $\beta=8d^{\frac{2}{p}}\metaLR\leq\frac{1}{2}$, we follow the derivation of~\pref{eq:meta-inter} and obtain that
\begin{align*}
    \E\left[\sum_{t\in\calI_k}\inner{\combine_{t}-e_{j_k}, \costhat_t}\right] \leq \frac{2\ln T}{\metaLR} + 4\metaLR d|\calI_k|+\order\left(\metaLR|\calI_k|d^{\frac{1}{q}}\sqrt{\frac{S}{T}}\right) + \O\left(\frac{\abs{\I_k}}{\metaLR T}\right).
\end{align*}
Summing the regret over all the intervals achieves the bound for $\metareg$:
\begin{align*}
     \E\left[\metareg\right] &=\E\left[\sum_{k=1}^S\sum_{t\in\calI_k}\inner{\sleep_{t}-e_{j_k}, \costhat_t}\right] \\
    & \leq \frac{2S\ln T}{\metaLR} + 4\metaLR dT+\order\left(d^{\frac{1}{q}}\sqrt{ST}\right)  + \O(1/\epsilon) \leq \Ot\left( d^{\frac{1}{p}}\sqrt{ST} \right),
\end{align*}
where the last inequality is because we choose $\metaLR = \min\left\{d^{-\frac{1}{p}}\sqrt{\frac{S}{T}},\frac{C^2}{2}, \frac{1}{16d^{\frac{2}{p}}}\right\}$.
\end{proof}

\subsection{Proof of~\pref{thm:main-alpha}}\label{appendix:proof-main-alpha}
Putting everything together, we are now ready to prove our main result (\pref{thm:main-alpha}) in the setting when the feasible domain is $\alpha$-strongly convex.
~\\

\begin{proof}
First, it is evident to check that the parameter choice satisfies the condition required in~\pref{lem:base-reg-alpha} and~\pref{lemma:meta-final-alpha}. Therefore, based on the regret decomposition in~\pref{eqn:decomposition-alpha}, upper bound of bias term in ~\pref{eqn:bias-term-bound-alpha}, upper bound of positive term~\pref{eqn:posterm-term-bound-alpha}, base regret upper bound in~\pref{lem:base-reg-alpha} and meta regret upper bound in~\pref{lemma:meta-final-alpha}, we have
\begin{align*}
    &\E[\RegS] = \mathbb{E}\left[\sum_{k=1}^S \sum_{t \in \I_k} \inner{x_t - \ucirc_k, \ell_t}\right] \\
    &\leq \frac{2}{\Correct} + \sum_{k=1}^S\frac{\log(1/\gamma)}{\baseLR_{j_k}} + \left(\frac{2d^{\frac{2}{p}}\baseLR }{(1-\beta)}\cdot\frac{5\alpha+4}{\alpha}-\frac{1}{\Correct T(1-\beta)}\right)\sum_{t\in \calI_k}\frac{1-\|a_t^{(j_k)}\|_\calX}{1-\sum_{i=1}^t\sleep_{t,i}\|a_t^{(i)}\|_\calX}\\
    &\qquad + (\beta+\gamma)T + \otil\left(d^{\frac{1}{p}}\sqrt{ST}\right).
\end{align*}
Importantly, note that the coefficient of the third term is actually zero. Indeed, due to the parameter configurations that $\gamma=4Cd^{\frac{1}{q}}\sqrt{\frac{S}{T}}$, $\baseLR =Cd^{-\frac{1}{p}}\sqrt{\frac{S}{T}}$, $\Correct=\frac{Cd^{-\frac{1}{q}}}{\sqrt{ST}}$, $\beta=8d^{\frac{2}{p}}\metaLR$, $\metaLR=\min\left\{\frac{1}{16d^{\frac{2}{p}}}, \frac{C^2}{2}, d^{-\frac{1}{p}}\sqrt{\frac{S}{T}}\right\}$ and $C=\sqrt{\frac{\alpha}{10\alpha+8}}$, we can verify that
\begin{align*}
    \frac{2d\baseLR (5\alpha+4)}{\alpha}-\frac{1}{\Correct T} = 0.
\end{align*}
Then we can achieve $\E[\RegS] \leq \otil\left(d^{\frac{1}{p}}\sqrt{ST}\right)$ and complete the proof.
\end{proof}
\section{Omitted Details for~\pref{sec: unconstrained}}
\label{appendix:unsconstrained}

In this section, we consider the switching regret of unconstrained linear bandits.

\subsection{Pseudocode of Black-Box Reduction}
\label{appendix:reduction-code}
We show the pseudocode of our black-box reduction in~\pref{alg:unconstrained-linear-bandits}.
\begin{algorithm}[!h]
   \caption{Comparator-adaptive algorithm for unconstrained linear bandits}
   \label{alg:unconstrained-linear-bandits}
    \textbf{Input:} subroutine $\A_{\V}$ (unconstrained OCO algorithm),  subroutine $\A_{\Z}$ (constrained linear bandits algorithm), $\Z = \{z \mid \norm{z}_2 \leq 1\}$.
    
   \For{$t=1$ {\bfseries to} $T$}{
        
        Receive the direction $z_t \in \Z$ from subroutine $\A_{\Z}$.    

        Receive the magnitude $v_t \in \R$ from subroutine $\A_{\V}$.

        Submit $x_t = z_t \cdot v_t $ and receive and observe the loss $\ell_t^\T x_t$.

        Send $\ell_t^\T z_t = \ell_t^\T x_t/v_t$ as the feedback for subroutine $\A_{\Z}$.

        Construct linear function $f_t(v) \triangleq v \cdot \ell_t^\T z_t$ as the feedback for subroutine $\A_{\V}$.      
      }
\end{algorithm}
\subsection{Proof of~\pref{lemma:unconstrained-decompose}}
\label{appendix:proof-unconstrained-reduction}
\begin{proof} Our switching regret decomposition for linear bandits is inspired by the existing black-box reduction designed for the full information online convex optimization~\citep{COLT'18:black-box-reduction} and static regret of linear bandits~\citep{NIPS'20:Dirk}. Indeed, the switching regret can be decomposed in the following way.
\begin{align*}
\Reg(u_1,\ldots,u_T) &= \sum_{t=1}^T \ell_t^\T x_t - \sum_{t=1}^T \ell_t^\T u_t\\
    &=  \sum_{k=1}^S\sum_{t\in\calI_k}\ell_t^\top x_t - \sum_{k=1}^S \sum_{t \in \I_k} \ell_t^\T \ucirc_k \\
    &=  \sum_{k=1}^S \sum_{t \in \I_k} \ell_t^\T (z_t\cdot v_t - \ucirc_k) \tag{$x_t=z_t\cdot v_t$}\\
    &=  \sum_{k=1}^S \left(\sum_{t \in \I_k} \inner{z_t,\ell_t}(v_t - \norm{\ucirc_k}_2) +  \norm{\ucirc_k}_2 \sum_{t \in \I_k} \inner{z_t - \frac{\ucirc_k}{\norm{\ucirc_k}_2},\ell_t}\right)\\
    &=  \sum_{k=1}^S \Reg_{\I_k}^{\mathcal{V}}(\norm{\ucirc_k}_2) + \sum_{k=1}^S \norm{\ucirc_k}_2 \cdot \Reg_{\I_k}^{\mathcal{Z}}\left( \frac{\ucirc_k}{\norm{\ucirc_k}_2} \right),
\end{align*}
which finishes the proof.
\end{proof}




\subsection{Algorithm for Unconstrained OCO with Switching Regret}
\label{appendix:OCO-algorithm}
In this section, we present the details of our proposed algorithm for unconstrained OCO with switching regret.

Under the unconstrained setup, the diameter of the feasible domain is $D = \infty$. However, as observed in~\citep[Appendix D.5]{COLT'21:impossible-tuning}, we can simply assume $\max_{k \in [S]} \norm{\ucirc_k}_2 \leq 2^T$. Otherwise, we will have $T \leq \log_2 (\max_{k \in [S]} \norm{\ucirc_k}_2)$, and by constraining the learning algorithm such that $\norm{v_t}_2 \leq 2^T$, we can obtain the following trivial upper bound for switching regret: $\Reg \leq \sum_{t=1}^T \norm{\nabla f_t(v_t)}_2 \norm{v_t - u_t}_2 \leq T (2^T + \max_{k \in [S]} \norm{\ucirc_k}_2) = \Ot(\max_{k \in [S]} \norm{\ucirc_k}_2)$, which is already adaptive to the comparators. Therefore, we can simply focus on the \emph{constrained} online learning with a maximum diameter $D= 2^T$. In addition, as mentioned earlier, we do not assume the knowledge of the number of switch $S$ in advance in this part. To this end, we propose a two-layer approach to simultaneously adapt to the unknown scales of the comparators and the unknown number of switch, which consists of a meta algorithm learning over a set of base learners. Below we specify the details.

\paragraph{Base algorithm.} The base algorithm tackles OCO problem with a given scale of feasible domain. The only requirement is as follows: given a constrained domain $\X \subseteq \R^d$ with diameter $D = \sup_{x \in \X} \norm{x}_2$, base algorithm running over $\X$ ensures an $\Ot(D \sqrt{\abs{\I}})$ static regret over any interval $\I \subseteq [T]$. Formally, we assume the base algorithm to satisfy the following requirement.
\begin{myRequirement}
\label{require:data-independent}
Consider the online convex optimization problem consisting a convex feasible domain $\X \subseteq \R^d$ and a sequence of convex loss functions $f_1,\ldots,f_T$, where $f_t: \X \mapsto \R$ and we assume $\mathbf{0} \in \X$ and $\norm{\nabla f_t(v)}_2 \leq 1$ for any $v \in \X$ and $t \in [T]$. An online algorithm $\A$ running over this problem returns the decision sequence $v_1,\ldots,v_T \in \X$. We require the algorithm $\A$ to ensure the following regret guarantee
\begin{equation}
    \label{eq:condition-base}
    \sum_{t \in \I} f_t(v_t) - \min_{u\in\X} \sum_{t \in \I} f_t(u) \leq \Ot\left( D \sqrt{\abs{\I}}\right)
\end{equation}
for any interval $\I \subseteq [T]$, where $D = \sup_{x \in \X} \norm{x}_2$ is the diameter of the feasible domain.
\end{myRequirement}
This requirement can be satisfied by recent OCO algorithms with \emph{interval regret} (or called \emph{strongly adaptive regret}) guarantee, such as Algorithm 1 of~\citet{ICML'15:Daniely-adaptive}, Algorithm 2 of~\citet{AISTATS'17:coin-betting-adaptive}, Theorem 6 of~\citet{ICML'20:Ashok}. We denote by $\mathfrak{B}$ any suitable base algorithm.

Since both the scale of comparators and the number of switch are unknown in advance, we maintain a set of base algorithm instances, defined as 
\begin{equation}
    \label{eq:base-learners-TM}
    \S = \Big\{ \B_{i,r}, \forall (i,r) \in [H] \times [R] ~ \big | ~ \B_{i,r} \leftarrow \mathfrak{B}(\X_i), \text{ with } \X_i = \{x \mid \norm{x}_2 \leq D_i = T^{-1} \cdot 2^{i-1}\} \Big\}.
\end{equation}
In above, $H = \lceil \log_2 T \rceil + T + 1$ and the index $i \in [H]$ maintain a grid to deal with uncertainty of unknown comparators' scale; $R = \lceil \log_2 T \rceil$ and the index $r \in [R]$ maintains a grid to handle uncertainty of unknown number of switch $S$. There are in total $N = H \cdot R$ base learners. For $i \in [H]$ and $r \in [R]$, the base learner $\B_{i,r}$ is an instantiation of the base algorithm whose feasible domain is $\X_i \subseteq \R^d$ with diameter $D_i$, and $v_{t,(i,r)}$ denotes her returned decision at round $t$. We stress that even if $S$ is known, the two-layer structure remains necessary due to the unknown comparators' scale.

\paragraph{Meta algorithm.} Then, a meta algorithm is used to combine all those base learners, and more importantly, the regret of meta algorithm should be adaptive to the individual loss scale of each base learner, such that the overall algorithm can achieve a comparator-adaptive switching regret. We achieve so by building upon the recent progress in the classic expert problem~\citep{COLT'21:impossible-tuning}. Our proposed algorithm is OMD with a multi-scale entropy regularizer and certain important correction terms. Specifically, let the weight vector produced by the meta algorithm be $w_t \in \Delta_{N}$, then the overall decision is $v_t = \sum_{i=1}^H \sum_{r=1}^R w_{t,(i,r)} v_{t,(i,r)}$, and the weight is updated by
\begin{equation}
    \label{eq:MsMwC-TM}
    w_{t+1} = \argmin_{w \in \Omega} \inner{w, \ell_t + a_t} + D_{\psi}(w,w_t),
\end{equation}
where $\Omega = \{w \mid w \in \Delta_{N} \text{ and } w_{t,(i,r)} \geq \frac{1}{{T^2 \cdot 2^{2i}}}, \forall i \in [H], r\in [R]\}$ is the clipped domain. Besides, the meta loss $\ell_t$, the correction term $a_t$,  and a certain regularizer $\psi$ are set as follows:
\begin{itemize}
    \item The regularizer $\psi: \Delta_{N} \mapsto \R$ is set as a \emph{weighted} negative-entropy regularizer defined as
    \begin{equation}
        \label{eq:weighted-regularizer}
        \psi(w) \triangleq \sum_{(i,r) \in [H] \times [R]} \frac{c_i}{\eta_r} w_{(i,r)} \ln w_{(i,r)} \mbox{ with } c_i = T^{-1} \cdot 2^{i-1} \mbox{ and } \eta_r = \frac{1}{32\cdot 2^r}.
    \end{equation}
    \item The feedback loss of meta algorithm $\ell_t \in \R^{N}$ is set as such to measure the quality of each base learner: $\ell_{t,(i,r)} \triangleq \innersmall{\nabla f_t(v_t), v_{t,(i,r)}}$ for any $(i,r) \in [H] \times [R]$.
    \item The correction term $a_t \in \R^{N}$ is set as: $a_{t,(i,r)} \triangleq 32 \frac{\eta_r}{c_i} \ell_{t,(i,r)}^2$ for any $(i,r) \in [H] \times [R]$, which is essential to ensure the meta regret compatible to the final comparator-adaptive bound. 
\end{itemize}

\begin{algorithm}[!t]
   \caption{Comparator-adaptive algorithm for unconstrained OCO}
   \label{alg:unconstrained-OCO}
   
    \textbf{Input:} base algorithm $\mathfrak{B}$. 

    \textbf{Define:} $H = \lceil \log_2 T \rceil + T + 1$ and $R = \lceil \log_2 T \rceil$.

    \textbf{Define:} clipped domain $\Omega = \{w \mid w \in \Delta_{N} \text{ and } w_{t,(i,r)} \geq \frac{1}{{T^2 \cdot 2^{2i}}}, \forall i \in [H], r\in [R]\}$.

    \textbf{Define:} weighted entropy regularizer $\psi(w) \triangleq \sum_{(i,r) \in [H] \times [R]} \frac{c_i}{\eta_r} w_{(i,r)} \ln w_{(i,r)}$ with $c_i = T^{-1} \cdot 2^{i-1}$ for $i \in [H]$ and $\eta_r = \frac{1}{32\cdot 2^r}$ for $r \in [R]$.
    
    \textbf{Initialization:} for $(i,r) \in [H]\times [R]$, initiate base algorithm $\B_{i,r} \leftarrow \mathfrak{B}(\X_i)$ with $\X_i = \{x \mid \norm{x}_2 \leq D_i\}$, which is an instance of $\mathfrak{B}$, and prior distribution $w_{1,(i,r)}\propto \eta_r^2/c_i^2$.
        
   \For{$t=1$ {\textbf{\textit{to}}} $T$}{
      Each base learner $\B_{(i,r)}$ returns the local decision $v_{t,(i,r)}$ for each $i \in [H]$ and $r \in [R]$.

      Make the final decision $v_t = \sum_{(i,r)\in [H] \times [R]} w_{t,(i,r)} v_{t,(i,r)}$ and receive feedback $\nabla f_t(v_t)$.

      Construct feedback loss $\ell_t \in \R^{N}$ and correction term $a_t \in \R^{N}$ for meta algorithm : $\ell_{t,j} \triangleq \inner{\nabla f_t(v_t), v_{t,j}}$, $a_{t,j} \triangleq 32 \frac{\eta_r}{c_i} \ell_{t,j}^2$, $\forall j = (i,r) \in [H] \times [R]$.

      Meta algorithm updates the weight by $w_{t+1} = \argmin_{w \in \Omega} \inner{w, \ell_t + a_t} + D_{\psi}(w,w_t)$.
      }
\end{algorithm}

The entire algorithm consists of meta algorithm specified above and base algorithm satisfying~\pref{require:data-independent}. We show the pseudocode in~\pref{alg:unconstrained-OCO}.

\subsection{Proof of~\pref{thm:unconstrained-OCO}}
\label{appendix:proof-unconstrained-OCO}
\begin{proof}
Consider the $k$-th interval $\I_k$. The regret within this interval can be decomposed as follows. 
\begin{align}
    \sum_{t \in \I_k} \Big(f_t(v_t) - f_t(\ucirc_k)\Big) & = \sum_{t \in \I_k} \Big(f_t(v_t) - f_t(v_{t,j})\Big) + \sum_{t \in \I_k} \Big(f_t(v_{t,j}) - f_t(\ucirc_k)\Big) \nonumber\\
    & \leq \sum_{t \in \I_k} \inner{\nabla f_t(v_t), v_t - v_{t,j}} + \sum_{t \in \I_k} \Big(f_t(v_{t,j}) - f_t(\ucirc_k)\Big)\nonumber\\
    & = \underbrace{\sum_{t \in \I_k} \inner{w_t - e_j, \ell_t}}_{\meta} + \underbrace{\sum_{t \in \I_k} \Big(f_t(v_{t,j}) - f_t(\ucirc_k)\Big)}_{\base},\label{eq:unconstrained-decompose-appendix}
\end{align}
where the final equality is because $\ell_{t,j}=\inner{\nabla f_t(v_t), v_{t,j}}$ and $v_t=\sum_{j'\in [H]\times[R]}w_{t,j'}v_{t,j'}$. Note that the decomposition holds for any index $j = (i,r) \in [H]\times[R]$.

We first consider the case when $\norm{\ucirc_k}_2 \geq \frac{1}{T}$ and will deal with the other case (when $\norm{\ucirc_k}_2 < \frac{1}{T}$) at the end of the proof. Under such a circumstance, we can choose $(i,r)=(i_k^*,r_k^*)$ such that
\begin{equation}\label{eqn:choice-of-ir}
    \begin{split}
    c_{i_k^*} &= T^{-1} \cdot 2^{i_k^* - 1} \leq \|\ucirc_k\|_2 \leq T^{-1} \cdot 2^{i_k^*} = c_{i_k^* + 1} , \mbox{ and } \\
    \eta_{r_k^*} &= \frac{1}{32\cdot 2^{r_k^*}}\leq \frac{1}{32 \sqrt{|\calI_k|}}\leq  \frac{1}{32\cdot 2^{r_k^*-1}} = \eta_{r_k^* - 1},     
    \end{split}    
\end{equation}
which is valid as $i\in [H] = [\lceil\log_2 T \rceil + T + 1]$ and $r\in [R]=[\lceil \log_2 T\rceil]$. We now give the upper bounds for \meta and \base~respectively.

\paragraph{\base.} Based on the assumption of base algorithm, we have base learner $\calB_{j_k^*}$ satisfying
\begin{equation}
    \label{eq:unconstrained-base-bound}
    \sum_{t \in \I_k} \Big(f_t(v_{t,j_k^*}) - f_t(\ucirc_k)\Big) \leq \Ot\left(c_{i_k^*} \sqrt{\abs{\I_k}}\right) \leq \Ot\left(\norm{\ucirc_k}_2 \sqrt{\abs{\I_k}}\right),
\end{equation}
where we use the interval regret guarantee of base algorithm (see~\pref{require:data-independent}) and also use the fact that the diameter of the feasible domain for base learner $\B_{j_k^*}$ is $2^{i_k^*}$ as $\X_{i_k^*} = \{x \mid \norm{x}_2 \leq D_{i_k^*}\}$ and $D_{i_k^*} = c_{i_k^*}$. The last inequality holds by the choice of $i_k^*$ shown in~\pref{eqn:choice-of-ir}.

\paragraph{\meta.} The meta algorithm is essentially online mirror descent with a weighted entropy regularizer. Based on Lemma 1 in~\citep{COLT'21:impossible-tuning}, if for all $i\in [H]$ and $r\in[R]$, $32 \frac{\eta_r}{c_i} \abs{\ell_{t,(i,r)}} \leq 1$, then we have for any $q\in\Omega$, 
\begin{align}\label{eq:meta-regret-decompose-proto}
    \sum_{t \in \I_k} \inner{w_t - q, \ell_t} &\leq \sum_{t \in \I_k} \Big( D_\psi(q, w_t) - D_\psi(q, w_{t+1})\Big) + 32 \sum_{t \in \I_k} \sum_{i\in[H]}\sum_{r\in[R]} \frac{\eta_r}{c_i} q_{(i,r)} \ell_{t,(i,r)}^2.
\end{align}
Note that this is a simplified version of Lemma 1 in~\citep{COLT'21:impossible-tuning} for the interval regret, which employs a fixed learning rate for each action and does not include the optimism in the algorithm. We present the simplified lemma in~\pref{lemma:impossible-tuning-lemma} in~\pref{appendix:OMD} for completeness.

To this end, we first verify the condition of $32 \frac{\eta_r}{c_i} \abs{\ell_{t,(i,r)}} \leq 1$ for all $i\in [H]$, $r\in [R]$. In fact, 
\begin{align*}
    \frac{32\eta_r}{c_i}\left|\ell_{t,(i,r)}\right| \leq \frac{1}{c_i \cdot 2^r} \|\nabla f_t(v_t)\|_2 \cdot \|v_{t,i}\|_2 \leq \frac{1}{2^r} \leq 1,
\end{align*}
where the first inequality is by the definition of $\eta_r = \frac{1}{32 \cdot 2^r}$ and the construction of meta loss $\ell_{t,(i,r)} = \innersmall{\nabla f_t(v_t), v_{t,(i,r)}}$, the second inequality is because $\|v_{t,i}\|_2\leq c_i$ and $\|\nabla f_t(v)\|_2\leq 1$ for all $v \in \R^d$, and the third inequality holds as $r\geq 1$. 

Then we define $\bar{e}_{j_k^*} \triangleq \bar{e}_{(i_k^*,r_k^*)} = \left(1 - \frac{R \cdot a_0}{T^2} \right)e_{(i_k^*,r_k^*)} + \sum_{(i,r) \in [H] \times [R]} \frac{1}{T^2\cdot 2^{2i}} e_{(i,r)}$, where $a_0 =\sum_{i=1}^H \frac{1}{2^{2i}} = \frac{1}{3}(1-\frac{1}{4^H})$ is a constant which guarantees $\bar{e}_{j_k^*}\in \Omega$. Using~\pref{eq:meta-regret-decompose-proto} with $q=\bar{e}_{j_k^*}$, we have
\begin{align*}
    \sum_{t \in \I_k} \inner{w_t - \bar{e}_{j_k^*}, \ell_t} \leq & \sum_{t \in \I_k} \left( D_\psi(\bar{e}_{j_k^*}, w_t) - D_\psi(\bar{e}_{j_k^*}, w_{t+1})\right) + 32 \sum_{t \in \I_k} \sum_{i\in[H]}\sum_{r\in[R]} \frac{\eta_r}{c_i} \bar{e}_{j_k^*, (i,r)} \ell_{t,(i,r)}^2\\
    = & \Big( D_\psi(\bar{e}_{j_k^*}, w_{s_k}) - D_\psi(\bar{e}_{j_k^*}, w_{s_{k+1}})\Big) + 32 \sum_{t \in \I_k} \sum_{i\in[H]}\sum_{r\in[R]} \frac{\eta_r}{c_i} \bar{e}_{j_k^*, (i,r)} \ell_{t,(i,r)}^2,
\end{align*}
where $s_k$ denotes the starting index of the interval $\I_k$ and $s_{k+1}$ is defined as $T+1$ if $\I_k$ is the last interval. The two terms on the right-hand side are called bias term and stability term respectively. In the following, we will give their upper bound individually. 

For the bias term, we have 
\begin{align}
    & D_\psi(\bar{e}_{j_k^*}, w_{s_k}) - D_\psi(\bar{e}_{j_k^*}, w_{s_{k+1}}) \nonumber \\
    & = \sum_{i\in [H]}\sum_{r\in[R]} \frac{c_{i}}{\eta_{r}} \left( \bar{e}_{j_k^*,(i,r)} \ln \frac{w_{s_k,(i,r)}}{w_{s_{k+1},(i,r)}} + w_{s_k,(i,r)} - w_{s_{k+1},(i,r)} \right)  \tag{by definition in~\pref{eq:weighted-regularizer}} \\
    & = \sum_{i\in [H]}\sum_{r\in[R]} \frac{c_{i}}{\eta_{r}} \left( \bar{e}_{j_k^*,(i,r)} \ln \frac{w_{s_k,(i,r)}}{w_{s_{k+1},(i,r)}}\right) + \sum_{i\in [H]}\sum_{r\in[R]} \frac{c_{i}}{\eta_{r}}\Big(w_{s_k,(i,r)} - w_{s_{k+1},(i,r)}\Big) \nonumber \\
    & \leq \frac{c_{i_k^*}}{\eta_{r_k^*}} \ln\left(T^2 \cdot 2^{2i_k^*}\right) + \sum_{(i,r)\neq (i_k^*, r_k^*)} \frac{1}{T^2 \cdot 2^{2i}} \frac{c_{i}}{\eta_{r}} \ln\big(T^2 \cdot 2^{2i}\big) + \sum_{i\in [H]}\sum_{r\in[R]} \frac{c_{i}}{\eta_{r}}\Big(w_{s_k,(i,r)} - w_{s_{k+1},(i,r)}\Big) \tag{$w_{s_k},w_{s_{k+1}}\in\Omega$} \\
    & \leq \frac{c_{i_k^*}}{\eta_{r_k^*}} \ln\left(4T^4 \cdot c_{i_k^*}^2\right) + \sum_{(i,r)\neq (i_k^*, r_k^*)} \frac{2 \ln T + (4\ln 2) \cdot T}{32 \cdot T^3 \cdot 2^{i+r+1}} + \sum_{i\in [H]}\sum_{r\in[R]} \frac{c_{i}}{\eta_{r}}\Big(w_{s_k,(i,r)} - w_{s_{k+1},(i,r)}\Big)\nonumber \\
    & = \Ot\left(\frac{c_{i_k^*}}{\eta_{r_k^*}} \ln c_{i_k^*} \right) + \Ot\left(\frac{1}{T^2}\right) + \sum_{i\in [H]}\sum_{r\in[R]} \frac{c_{i}}{\eta_{r}}\Big(w_{s_k,(i,r)} - w_{s_{k+1},(i,r)}\Big).\label{eq:meta-bias-term}
\end{align}
Moreover, for the stability term, we have
\begin{align}
    & 32\sum_{t \in \I_k} \sum_{i\in[H]}\sum_{r\in[R]} \frac{\eta_r}{c_i} \bar{e}_{j_k^*, (i,r)} \ell_{t,(i,r)}^2 \nonumber\\
    & = 32\sum_{t \in \I_k} \frac{\eta_{r_k^*}}{c_{i_k^*}} \left(1 - \frac{R \cdot a_0}{T^2} + \frac{1}{T^2 \cdot 2^{2 i_k^*}}\right) \ell_{t,(i_k^*,r_k^*)}^2 + 32 \sum_{t \in \I_k} \sum_{(i,r)\neq (i_k^*,r_k^*)} \frac{\eta_r}{c_i} \bar{e}_{j_k^*,(i,r)} \ell_{t,(i,r)}^2 \nonumber\\
    & \leq 32\sum_{t \in \I_k} \frac{\eta_{r_k^*}}{c_{i_k^*}} \ell_{t,(i_k^*,r_k^*)}^2 + 32 \sum_{t \in \I_k} \sum_{i \in [H]} \sum_{r \in [R]} \frac{\eta_r c_i}{T^2 \cdot 2^{2i}} \nonumber\\
    & \leq \O\left( \eta_{r_k^*}c_{i_k^*} \abs{\I_k} \right) + \sum_{t \in \I_k} \sum_{i \in [H]} \sum_{r \in [R]} \frac{1}{T^3 \cdot 2^{i+r+1}} \nonumber\\
    & = \O\left( \eta_{r_k^*}c_{i_k^*} \abs{\I_k} \right) + \O\left( \frac{1}{T^2}\right) \label{eq:meta-stability-term}
\end{align}
where the two inequalities hold as $\ell_{t,(i,r)}^2 = \innersmall{\nabla f_t(v_t), v_{t,(i,r)}}^2 \leq \norm{\nabla f_t(v_t)}_2^2 \norm{v_{t,(i,r)}}_2^2 \leq c_i^2$. Combining the upper bounds of bias term in~\pref{eq:meta-bias-term} and stability term in~\pref{eq:meta-stability-term}, we get
\begin{equation*}
    \sum_{t \in \I_k} \inner{w_t - \bar{e}_{j_k^*}, \ell_t} \leq \Ot\left( \eta_{r_k^*}c_{i_k^*} \abs{\I_k} + \frac{c_{i_k^*}}{\eta_{r_k^*}} \ln c_{i_k^*} \right) + \Ot\left(\frac{1}{T^2}\right) + \sum_{i\in [H]}\sum_{r\in[R]} \frac{c_{i}}{\eta_{r}}\Big(w_{s_k,(i,r)} - w_{s_{k+1},(i,r)}\Big).
\end{equation*}
Further, notice that
\begin{align}
    \label{eq:meta-regret-additional-term}
    \sum_{t\in\calI_k}\inner{\bar{e}_{j_k^*}-e_{j_k^*}, \ell_t} \leq\sum_{t\in\calI_k} \sum_{i\in[H]}\sum_{r\in[R]}\frac{1}{T^2\cdot 2^{2i}}\cdot\ell_{t,(i,r)}\leq \sum_{t\in\calI_k} \sum_{i\in[H]}\sum_{r\in[R]} \frac{1}{T^3\cdot 2^{i+1}} \leq \Ot\left(\frac{1}{T^2}\right),
\end{align}
and we thus obtain the overall meta regret upper bound in the interval $\I_k$:
\begin{align}
    & \sum_{t \in \I_k} \inner{w_t - e_{j_k^*}, \ell_t} = \sum_{t \in \I_k} \inner{w_t - \bar{e}_{j_k^*}, \ell_t} + \sum_{t \in \I_k} \inner{\bar{e}_{j_k^*}-e_{j_k^*}, \ell_t} \nonumber\\
    & \leq \Ot\left( \eta_{r_k^*}c_{i_k^*} \abs{\I_k} + \frac{c_{i_k^*}}{\eta_{r_k^*}} \ln c_{i_k^*} \right) + \Ot\left(\frac{1}{T^2}\right) + \sum_{i\in [H]}\sum_{r\in[R]} \frac{c_{i}}{\eta_{r}}\Big(w_{s_k,(i,r)} - w_{s_{k+1},(i,r)}\Big) \nonumber\\
    & = \Ot\left(\norm{\ucirc_k}_2 \sqrt{\abs{\I_k}}\right) + \Ot\left(\frac{1}{T^2}\right) + \sum_{i\in [H]}\sum_{r\in[R]} \frac{c_{i}}{\eta_{r}}\Big(w_{s_k,(i,r)} - w_{s_{k+1},(i,r)}\Big), \label{eq:meta-regret-last}
\end{align}
where the last inequality is because of the choice of $i_k^*$ and $r_k^*$ defined in~\pref{eqn:choice-of-ir}. The $\Ot(\cdot)$-notation omits logarithmic dependence on $T$ and comparator norm $\norm{\ucirc_k}_2$. 

\paragraph{\textsc{Overall Regret}.} The overall regret is obtained by combining the base regret and meta regret and further summing over all the intervals $\I_1,\ldots,\I_S$. Indeed, we have the following total meta-regret by taking summation over intervals on~\pref{eq:meta-regret-last},
\begin{align}
    & \sum_{k=1}^S \sum_{t \in \I_k} \inner{w_t - e_{j_k^*}, \ell_t} \nonumber\\
    & \leq \Ot\left( \sum_{k=1}^S \norm{\ucirc_k}_2 \sqrt{\abs{\I_k}}\right) + \Ot\left(\frac{S}{T^2}\right) + \sum_{k=1}^S \sum_{i\in [H]}\sum_{r\in[R]} \frac{c_{i}}{\eta_{r}}\Big(w_{s_k,(i,r)} - w_{s_{k+1},(i,r)}\Big) \nonumber\\
    & \leq \Ot\left( \sum_{k=1}^S \norm{\ucirc_k}_2 \sqrt{\abs{\I_k}}\right) + \sum_{i\in [H]}\sum_{r\in[R]} \frac{c_{i}}{\eta_{r}} w_{1,(i,r)}\nonumber\\
    & = \Ot\left( \sum_{k=1}^S \norm{\ucirc_k}_2 \sqrt{\abs{\I_k}}\right), \label{eq:unconstrained-meta-bound}
\end{align}
where the final equality is because we choose $w_{1,(i,r)}\propto \frac{\eta_r^2}{c_i^2}$ for all $(i,r)\in[H]\times[R]$. Indeed, such a setting of prior distribution ensures that
\begin{align*}
    &\sum_{i\in [H]}\sum_{r\in [R]}\frac{c_i}{\eta_r}\cdot w_{1,(i,r)} = \frac{\sum_{i\in [H]}\sum_{r\in [R]}\frac{\eta_r}{c_i}}{\sum_{i\in [H]}\sum_{r\in [R]}\frac{\eta_r^2}{c_i^2}} = \frac{16}{T} \cdot \frac{\sum_{i\in [H]}\sum_{r\in [R]}\frac{1}{2^{i+r}}}{\sum_{i\in [H]}\sum_{r\in [R]}\frac{1}{2^{2i+2r}}} \\
    &= \frac{144}{T} \cdot \frac{\left(1-\big( \frac{1}{2} \big)^R\right)\left(1-\big( \frac{1}{2} \big)^H\right)}{\left(1-\big( \frac{1}{4} \big)^R\right)\left(1-\big( \frac{1}{4} \big)^H\right)} = \frac{144}{T} \cdot \frac{1}{\left(1+\big( \frac{1}{2} \big)^R\right)\left(1+\big( \frac{1}{2} \big)^H\right)} \leq \O\left( \frac{1}{T}\right),
\end{align*}
and also guarantees that $w_{1}\in\Omega$ since for any $(i,r)\in[H]\times[R]$,
\begin{align*}
    w_{1,(i,r)} & = \frac{\frac{\eta_r^2}{c_i^2}}{\sum_{i'\in[H]}\sum_{r'\in[R]}\frac{\eta_{r'}^2}{c_{i'}^2}} = \frac{\frac{1}{2^{2i+2r}}}{\sum_{i'\in[H]}\sum_{r'\in[R]}\frac{1}{2^{2i'+2r'}}}\\
    & \geq \frac{1}{T^2\cdot 2^{2i}}\cdot \frac{1}{\frac{1}{9}\left(1-\big(\frac{1}{4}\big)^R\right)\left(1-\big(\frac{1}{4}\big)^H\right)}\geq \frac{1}{T^2\cdot 2^{2i}},
\end{align*}
where the first inequality holds in that we have $2^{r} \leq T$ for $r \in [R]$.

Substituting the meta regret upper bound~\pref{eq:unconstrained-meta-bound} and the base regret upper bound~\pref{eq:unconstrained-base-bound} into the regret decomposition~\pref{eq:unconstrained-decompose-appendix} obtains that
\begin{align}
\label{eq:result-proof}
    \sum_{k=1}^S \sum_{t \in \I_k} \Big(f_t(v_t) - f_t(\ucirc_k)\Big) \leq \Ot\left( \sum_{k=1}^S \norm{\ucirc_k}_2 \sqrt{\abs{\I_k}}\right) \leq \otil\left( \max_{k \in [S]} \norm{\ucirc_k}_2\cdot \sqrt{ST}\right),
\end{align}
which finishes the proof for the case when $\norm{\ucirc_k}_2 \geq \frac{1}{T}$ holds for every $k \in [S]$. 

We now consider the case when the condition is violated. Suppose for some $k \in [S]$, it holds that $\norm{\ucirc_k}_2 < \frac{1}{T}$. Then, we pick any $\ucirc_k' \in \R^d$ such that $\norm{\ucirc_k'}_2 = \frac{1}{T}$, and obtain that
\begin{align*}
    \sum_{t\in\I_k} \Big(f_t(v_t) - f_t(\ucirc_k)\Big) &= \sum_{t\in\I_k} \Big(f_t(v_t) - f_t(\ucirc_k')\Big) + \sum_{t\in\I_k} \Big(f_t(\ucirc_k') - f_t(\ucirc_k)\Big)\\
     &=  \sum_{t\in\I_k} \Big(f_t(v_t) - f_t(\ucirc_k')\Big) + \sum_{t\in\I_k} \norm{\nabla f_t(\ucirc_k')}_2 \norm{\ucirc_k' - \ucirc_k}_2\\
     &\leq \sum_{t\in\I_k} \Big(f_t(v_t) - f_t(\ucirc_k')\Big) + \O\left(\frac{\abs{\I_k}}{T}\right).
\end{align*}
Clearly, the last additional term will not be the issue even after summation over $S$ intervals. Moreover, notice that now the comparator $\ucirc_k'$ satisfies the condition of $\norm{\ucirc_k'}_2 \geq \frac{1}{T}$, we can still use the earlier results including the base regret bound in~\pref{eq:unconstrained-base-bound} and meta regret bound in~\pref{eq:meta-regret-last}. Thus, we can guarantee the same regret bound as~\pref{eq:result-proof} under this scenario.

Hence, we finish the proof for the overall theorem. We finally remark that our algorithm for unconstrained OCO actually does not require the knowledge of $S$ ahead of time.
\end{proof}

\subsection{Data-dependent Switching Regret of Unconstrained Online Convex Optimization}
\label{appendix:data-dependent-OCO}
In this subsection, we further consider achieving data-dependent switching regret bound for unconstrained online convex optimization. 

In \pref{appendix:OCO-algorithm}, we require the base algorithm to achieve an $\Ot(D\sqrt{\abs{\I}})$ interval regret for any interval $\I \subseteq [T]$, where $D$ is the diameter of the feasible domain. See~\pref{require:data-independent} for more details. To achieve a data-dependent switching regret for unconstrained OCO, we require a stronger regret for the base algorithm. 

\begin{myRequirement}
\label{require:data-dependent}
Consider the online convex optimization problem consisting a convex feasible domain $\X \subseteq \R^d$ and a sequence of convex loss functions $f_1,\ldots,f_T$, where $f_t: \X \mapsto \R$ and we assume $\mathbf{0} \in \X$ and $\norm{\nabla f_t(v)}_2 \leq 1$ for any $v \in \X$ and $t \in [T]$. An online algorithm $\A$ running over this problem returns the decision sequence $v_1,\ldots,v_T \in \X$. We require the algorithm $\A$ to ensure the following regret guarantee
\begin{equation}
    \label{eq:condition-base-data-dependent}
    \sum_{t \in \I} f_t(v_t) - \min_{u\in\X} \sum_{t \in \I} f_t(u) \leq \Ot\left( D \sqrt{\sum_{t\in\I} \norm{\nabla f_t(v_t)}_2^2}\right)
\end{equation}
for any interval $\I \subseteq [T]$, where $D = \sup_{x \in \X} \norm{x}_2$ is the diameter of the feasible domain.
\end{myRequirement}
This requirement can be satisfied by recent OCO algorithm with \emph{data-dependent interval regret} guarantee, such as Algorithm 2 of~\citet{ICML19:Zhang-Adaptive-Smooth} and Theorem 6 of~\citet{ICML'20:Ashok}. 

Using the new base algorithm and the same meta algorithm as~\pref{appendix:OCO-algorithm}, the overall algorithm can ensure a data-dependent comparator-adaptive switching regret.
\begin{theorem}
\label{thm:unconstrained-OCO-data-dependent}
\pref{alg:unconstrained-OCO} with a base algorithm satisfying~\pref{require:data-dependent} guarantees that for any $S$, any partition $\I_1, \ldots, \I_S$ of $[T]$, and any comparator sequence $\ucirc_1,\ldots,\ucirc_S \in \R^d$, we have
\begin{equation}
    \label{eq:unconstrained-OCO-data-dependent}
    \begin{split}
    \sum_{k=1}^S \left(\sum_{t \in\I_k} f_t(v_t) -  \sum_{t \in\I_k} f_t(\ucirc_k) \right) & \leq \otil\left( \sum_{k=1}^S \norm{\ucirc_k}_2\sqrt{\sum_{t \in \I_k} \norm{\nabla f_t(v_t)}_2^2}  \right) \\
    & \leq \otil\left( \max_{k \in [S]} \norm{\ucirc_k}_2\cdot \sqrt{S \sum_{t=1}^T \norm{\nabla f_t(v_t)}_2^2}\right).    
    \end{split}
\end{equation}
Notably, the algorithm does not require the prior knowledge of the number of switch $S$ as the input.
\end{theorem}

\begin{proof}
The argument follows the proof of~\pref{appendix:proof-unconstrained-OCO}. Similar to~\pref{eq:unconstrained-decompose-appendix}, the regret within the interval can be decomposed into meta-regret and base-regret:
\begin{align}
    \sum_{t \in \I_k} \Big(f_t(v_t) - f_t(\ucirc_k)\Big) \leq \underbrace{\sum_{t \in \I_k} \inner{w_t - e_j, \ell_t}}_{\meta} + \underbrace{\sum_{t \in \I_k} \Big(f_t(v_{t,j}) - f_t(\ucirc_k)\Big)}_{\base},
\end{align}
which holds for any index $j = (i,r) \in [H]\times[R]$. 

We first the case when $\norm{\ucirc_k}_2 \geq \frac{1}{T}$ and will deal with the other case (when $\norm{\ucirc_k}_2 < \frac{1}{T}$) at the end of the proof. Under such a circumstance, we can choose $(i,r)=(i_k^*,r_k^*)$ such that
\begin{equation}
\label{eqn:choice-of-ir-data-dependent}
    \begin{split}
    c_{i_k^*} &= T^{-1} \cdot 2^{i_k^* - 1} \leq \|\ucirc_k\|_2 \leq T^{-1} \cdot 2^{i_k^*} = c_{i_k^* + 1} , \mbox{ and } \\
    \eta_{r_k^*} &= \frac{1}{32\cdot 2^{r_k^*}}\leq \frac{1}{32 \sqrt{\sum_{t \in \I_k} \norm{\nabla f_t(v_t)}_2^2}}\leq  \frac{1}{32\cdot 2^{r_k^*-1}} = \eta_{r_k^* - 1},     
    \end{split}    
\end{equation}
which is valid as $i\in [H] = [\lceil \log_2 T \rceil + T + 1]$ and $r\in [R]=[\lceil \log_2 T\rceil]$. We now give the upper bounds for \meta and \base~respectively.

\paragraph{\base.} Based on the assumption of base algorithm, we have base learner $\calB_{j_k^*}$ satisfies
\begin{equation}
    \label{eq:unconstrained-base-bound-data}
    \sum_{t \in \I_k} \Big(f_t(v_{t,j_k^*}) - f_t(\ucirc_k)\Big) \leq \Ot\left(2^{i_k^*} \sqrt{\sum_{t \in \I_k} \norm{\nabla f_t(v_t)}_2^2}\right) \leq \Ot\left(\norm{\ucirc_k}_2 \sqrt{\sum_{t \in \I_k} \norm{\nabla f_t(v_t)}_2^2}\right),
\end{equation}
where we use the interval regret guarantee of base algorithm (see~\pref{require:data-dependent}) and also use the fact that the diameter of the feasible domain for base learner $\B_{j_k^*}$ is $2^{i_k^*}$ as $\X_{i_k^*} = \{x \mid \norm{x}_2 \leq D_{i_k^*}\}$ and $D_{i_k^*} = c_{i_k^*}$. The last inequality holds by the choice of $i_k^*$ shown in~\pref{eqn:choice-of-ir-data-dependent}.

\paragraph{\meta.} Note that the meta algorithm remains the same, so we will only improve the analysis to show that the meta algorithm can also enjoy a data-dependent guarantee. The bias term will not be affected, which is the same as the data-independent one presented in~\pref{eq:meta-bias-term}, and the main modification will be conducted on the stability term. Indeed, continuing the analysis of the stability term exhibited in~\pref{eq:meta-stability-term}, we have
\begin{align}
    & 32\sum_{t \in \I_k} \sum_{i\in[H]}\sum_{r\in[R]} \frac{\eta_r}{c_i} \bar{e}_{j_k^*, (i,r)} \ell_{t,(i,r)}^2 \nonumber\\
    & \leq 32\sum_{t \in \I_k} \frac{\eta_{r_k^*}}{c_{i_k^*}} \ell_{t,(i_k^*,r_k^*)}^2 + \O\left( \frac{1}{T^2}\right) \nonumber\\
    & \leq \O\left( \eta_{r_k^*}c_{i_k^*} \sum_{t \in \I_k} \norm{\nabla f_t(v_t)}_2^2 \right) + \O\left( \frac{1}{T^2}\right) \label{eq:meta-stability-term-data-dependent}
\end{align}
where the last inequality holds as $\ell_{t,(i,r)}^2 = \innersmall{\nabla f_t(v_t), v_{t,(i,r)}}^2 \leq \norm{\nabla f_t(v_t)}_2^2 \norm{v_{t,(i,r)}}_2^2 \leq c_i^2 \norm{\nabla f_t(v_t)}_2^2$. Then, combining the upper bounds of bias term~\pref{eq:meta-bias-term}, above stability term~\pref{eq:meta-stability-term-data-dependent}, and additional term~\pref{eq:meta-regret-additional-term} leads to the following result:
\begin{align}
    & \sum_{t \in \I_k} \inner{w_t - e_{j_k^*}, \ell_t} = \sum_{t \in \I_k} \inner{w_t - \bar{e}_{j_k^*}, \ell_t} + \sum_{t \in \I_k} \inner{\bar{e}_{j_k^*}-e_{j_k^*}, \ell_t} \nonumber\\
    & \leq \Ot\left( \eta_{r_k^*}c_{i_k^*} \sum_{t \in \I_k} \norm{\nabla f_t(v_t)}_2^2 + \frac{c_{i_k^*}}{\eta_{r_k^*}} \ln c_{i_k^*} \right) + \Ot\left(\frac{1}{T^2}\right) + \sum_{i\in [H]}\sum_{r\in[R]} \frac{c_{i}}{\eta_{r}}\Big(w_{s_k,(i,r)} - w_{s_{k+1},(i,r)}\Big) \nonumber\\
    & = \Ot\left(\norm{\ucirc_k}_2 \sqrt{\sum_{t \in \I_k} \norm{\nabla f_t(v_t)}_2^2}\right) + \Ot\left(\frac{1}{T^2}\right) + \sum_{i\in [H]}\sum_{r\in[R]} \frac{c_{i}}{\eta_{r}}\Big(w_{s_k,(i,r)} - w_{s_{k+1},(i,r)}\Big), \label{eqn:meta-regret-mid-data-dependent}
\end{align}
where the last inequality is because of the choice of $i_k^*$ and $r_k^*$ defined in~\pref{eqn:choice-of-ir-data-dependent}. Summing over all the intervals $\I_1,\ldots,\I_S$ achieves a data-dependent upper bound for the meta-regret:
\begin{align}
    & \sum_{k=1}^S \sum_{t \in \I_k} \inner{w_t - e_{j_k^*}, \ell_t} \nonumber\\
    & \leq \Ot\left( \sum_{k=1}^S \norm{\ucirc_k}_2 \sqrt{\sum_{t \in \I_k} \norm{\nabla f_t(v_t)}_2^2}\right) + \sum_{k=1}^S \sum_{i\in [H]}\sum_{r\in[R]} \frac{c_{i}}{\eta_{r}}\Big(w_{s_k,(i,r)} - w_{s_{k+1},(i,r)}\Big) \nonumber\\
    & \leq \Ot\left( \sum_{k=1}^S \norm{\ucirc_k}_2 \sqrt{\sum_{t \in \I_k} \norm{\nabla f_t(v_t)}_2^2}\right) + \sum_{i\in [H]}\sum_{r\in[R]} \frac{c_{i}}{\eta_{r}} w_{1,(i,r)} \nonumber \\
    & = \Ot\left( \sum_{k=1}^S \norm{\ucirc_k}_2 \sqrt{\sum_{t \in \I_k} \norm{\nabla f_t(v_t)}_2^2}\right)\nonumber\\
    & \leq \otil\left( \max_{k \in [S]} \norm{\ucirc_k}_2\cdot \sqrt{S \sum_{t=1}^T \norm{\nabla f_t(v_t)}_2^2}\right).
\end{align}
The last equality holds by the same argument for~\pref{eq:unconstrained-meta-bound} and the final inequality is by Cauchy-Schwarz inequality. Combining the meta-regret and base-regret upper bounds finishes the proof for the case when $\norm{\ucirc_k}_2 \geq \frac{1}{T}$ holds for every $k \in [S]$. 

In addition, when the above condition of the comparators' norm is violated, we can deal with the scenario by the same argument at the end of~\pref{appendix:proof-unconstrained-OCO} and attain the same regret guarantee. Hence, we finish the proof of the overall theorem.
\end{proof}

\begin{remark}
Note that~\pref{thm:unconstrained-OCO-data-dependent} is for the unconstrained OCO setting, while from the proof we can see that actually the result holds even if the algorithm is required to make decisions from a bounded domain. Indeed, in the unconstrained setting, we only need to focus on a bounded domain with maximum diameter $2^T$ as observed in~\citep[Appendix D.5]{COLT'21:impossible-tuning}). As a result, when working under constrained OCO with a diameter $D_{\max} > 0$, we can still use our algorithm by simply maintaining the set of base algorithm instances as
\[
    \S' = \Big\{ \B_{i,r}, \forall (i,r) \in [H'] \times [R] ~ \big | ~ \B_{i,r} \leftarrow \mathfrak{B}(\X_i), \text{ with } \X_i = \{x \mid \norm{x}_2 \leq D_i = T^{-1} \cdot 2^{i-1}\} \Big\}.
\]
where $H' = \lceil \log_2 T \rceil + \lceil \log_2 D_{\max}\rceil + 1$ and $R = \lceil \log_2 T \rceil$ now. Thus, our result strictly improves the $\Ot\Big(D_{\max} \sqrt{S \sum_{t=1}^T \norm{\nabla f_t(v_t)}_2^2}\Big)$ result of~\citep{ICML'20:Ashok,NIPS'20:sword} for the constrained OCO setting.
\end{remark}

\subsection{Proof of~\pref{thm:unconstrained-linear-bandits}}
\label{appendix:proof-unconstrained}
\begin{proof}
From~\pref{lemma:unconstrained-decompose}, we have 
\begin{equation}
    \label{eq:unconstrained-decompose-new}
    \Reg(u_1,\ldots,u_T) = \sum_{k=1}^S \Reg_{\I_k}^{\mathcal{V}}(\norm{\ucirc_k}_2) + \sum_{k=1}^S \norm{\ucirc_k}_2 \cdot \Reg_{\I_k}^{\mathcal{Z}}\left( \frac{\ucirc_k}{\norm{\ucirc_k}_2} \right).
\end{equation}
In the following, we bound the two terms respectively. 

The first term on the right-hand side of~\pref{eq:unconstrained-decompose-new} is the switching regret of the OCO algorithm $\A_\V$, we have
\begin{align*}
    \sum_{k=1}^S \Reg_{\I_k}^{\mathcal{V}}(\norm{\ucirc_k}_2) = \sum_{k=1}^S \sum_{ t \in \I_k} \Big( f_t(v_t) -  f_t(\norm{\ucirc_k}_2)\Big) \leq \Ot\left( \sum_{k=1}^S \norm{\ucirc_k}_2 \sqrt{\abs{\I_k}}  \right),
\end{align*}
where the first equality is due to the definition of online function $f_t(v) = v \cdot \inner{\ell_t, z_t}$ and the second inequality holds by the regret guarantee of $\A_\V$ proven in~\pref{thm:unconstrained-OCO}.

The second term on the right-hand side of~\pref{eq:unconstrained-decompose-new} requires the switching regret analysis of the online algorithm for constrained linear bandits $\A_{\Z}$. Indeed, since the comparator satisfies that $\norm{\frac{\ucirc_k}{\norm{\ucirc_k}_2}}_2 = 1$, the subroutine $\A_{\Z}$ can be chosen as the proposed algorithm for linear bandits with $\ell_p$-ball feasible domain (with $p=2$), see~\pref{alg:ell-p-ball}. We thus get the following regret bound according to \pref{thm:interval}:
\begin{align*}
    \E\left[\Reg_{\I_k}^{\mathcal{Z}}\left( \frac{\ucirc_k}{\norm{\ucirc_k}_2} \right)\right] \leq \Ot\left( \sqrt{\frac{dT}{S}} + \sqrt{\frac{Sd}{T}} \abs{\I_k} \right).
\end{align*}

Substituting the above two upper bounds in~\pref{eq:unconstrained-decompose-new} gives that
\begin{align*}
    \E\left[\Reg(u_1,\ldots,u_T)\right] & = \sum_{k=1}^S  \E\left[\Reg_{\I_k}^{\mathcal{V}}(\norm{\ucirc_k}_2) \right] + \sum_{k=1}^S \E\left[\norm{\ucirc_k}_2 \cdot \Reg_{\I_k}^{\mathcal{Z}}\left( \frac{\ucirc_k}{\norm{\ucirc_k}_2} \right)\right]\\
    & \leq \Ot\left( \sum_{k=1}^S \norm{\ucirc_k}_2 \sqrt{\abs{\I_k}}  \right) + \Ot\left(\sum_{k=1}^S \norm{\ucirc_k}_2 \Big( \sqrt{\frac{dT}{S}} + \sqrt{\frac{Sd}{T}} \abs{\I_k} \Big) \right)\\
    & \leq \Ot\left(\sum_{k=1}^S \norm{\ucirc_k}_2 \Big( \sqrt{\frac{dT}{S}} + \sqrt{\frac{Sd}{T}} \abs{\I_k} \Big) \right) \\
    & \leq  \Ot\left(\max_{k\in[S]}\norm{\ucirc_k}_2 \cdot \sqrt{dST} \right)
\end{align*}
where the second inequality is because $\sqrt{\abs{\I_k}} \leq \sqrt{\frac{dT}{S}} + \sqrt{\frac{Sd}{T}} \abs{\I_k}$. Hence, we finish the proof.
\end{proof}

\section{Lemmas Related to Online Mirror Descent}
\label{appendix:OMD}
This section collects several useful lemmas related to online mirror descent (OMD).

We first introduce a general regret guarantee for OMD due to~\citet{bubeck2012regret}.
\begin{lemma}[{Theorem 5.5 of~\citet{bubeck2012regret}}]
\label{lemma:OMD-bubeck}
Let $\mathcal{D} \subset \R^d$ be an open convex set and let $\overline{\mathcal{D}}$ be the closure of $\mathcal{D}$. Let $\X$ be a compact and convex set and let $F$ be a Legendre function defined on $\overline{\mathcal{D}} \supset \X$ such that $\nabla F(x)-\metaLR \nabla \ell(x) \in \mathcal{D}^{*}$ holds for any $(x, \ell) \in(\X \cap \mathcal{D}) \times \mathcal{L}$, where $\mathcal{D}^* = \nabla F(\mathcal{D})$ is the dual space of $\mathcal{D}$ under $F$. Consider the following online mirror descent:
\begin{equation}
\begin{split}
x'_{t+1} &= \nabla F^{*}\left(\nabla F(x_t)-\metaLR \nabla \ell_t (x_t\right)),\\
x_{t+1} &= \argmin_{x \in \X} D_{F}(x, x'_{t+1}),
\end{split}
\end{equation}
where $F^*$ is the Legendre--Fenchel transform of $F$ defined by $F^*(u) = \sup_{x \in \X} (x^\T u - F(x))$. Then, we have 
\begin{align}
\label{eqn: OMD-bubeck}
\sum_{t=1}^{T} \ell_{t}(x_{t})-\sum_{t=1}^{T} \ell_{t}(x) \leq & \frac{F(x)-F(x_{1})}{\metaLR} +\frac{1}{\metaLR} \sum_{t=1}^{T} D_{F^{*}}\Big(\nabla F(x_{t})-\metaLR \nabla \ell_{t}(x_{t}), \nabla F(x_{t})\Big) .
\end{align}
\end{lemma}

We next introduce an important lemma related to the online mirror descent with weighted entropy regularizer, which is a  version of~\citep[Lemma 1]{COLT'21:impossible-tuning} in the fixed learning rate and non-optimistic setting. Note that this is actually an interval version of~\citep[Lemma 1]{COLT'21:impossible-tuning}, replacing the summation range from $[T]$ to an interval $\I \subseteq [T]$, which is also used in~\citep[Appendix C.3]{COLT'21:impossible-tuning}.

\begin{lemma}[{Lemma 1 of~\citet{COLT'21:impossible-tuning}}]
\label{lemma:impossible-tuning-lemma}
Consider the following online mirror descent update over a compact convex decision subset $\Omega \subseteq \Delta_d$,
\[
    w_{t+1} = \argmin_{w \in \Omega} \Big\{ \inner{w, \ell_t + a_t} + D_\psi(w,w_t) \Big\}
\]
where $\psi(w) = \sum_{n=1}^d \frac{1}{\baseLR_n} w_n \ln w_n$ is the weighted entropy regularizer. Suppose that for all $t \in [T]$, $32 \baseLR_n \abs{\ell_{t,n}} \leq 1$ holds for all $n \in [d]$ such that $w_{t,n} >0$. 
Then the above update ensures for any $u \in \Omega$,
\begin{equation*}
    \sum_{t \in \I} \inner{\ell_t, w_t - u} \leq \sum_{t \in \I_k} \Big( D_\psi(u, w_t) - D_\psi(u, w_{t+1})\Big) + 32 \sum_{t \in \I} \sum_{n=1}^d \baseLR_n u_n \ell_{t,n}^2 - 16 \sum_{t \in \I} \sum_{n=1}^d \baseLR_n w_{t,n} \ell_{t,n}^2.
\end{equation*}
\end{lemma}

\end{document}